\documentclass{article} 
\usepackage[final]{colm2026_conference}

\usepackage[utf8]{inputenc} 
\usepackage[T1]{fontenc}    
\usepackage{hyperref}       
\usepackage{url}            
\usepackage{booktabs}       
\usepackage{amsfonts}       
\usepackage{nicefrac}       
\usepackage{microtype}      
\usepackage{xcolor}         
\usepackage{lineno}

\definecolor{darkblue}{rgb}{0, 0, 0.5}
\hypersetup{colorlinks=true, citecolor=darkblue, linkcolor=darkblue, urlcolor=darkblue}
\usepackage{graphicx}       
\usepackage{amsmath}
\usepackage{amsthm}
\usepackage{amssymb}
\usepackage[inline]{enumitem}
\usepackage{mathtools}      
\usepackage{bbm}            
\usepackage[breakable]{tcolorbox} 

\theoremstyle{plain}
\newtheorem{theorem}{Theorem}[section]
\newtheorem{lemma}[theorem]{Lemma}

\newtheorem{corollary}[theorem]{Corollary}

\theoremstyle{definition}

\newtheorem{assumption}[theorem]{Assumption}

\theoremstyle{remark}

\newtheorem*{theorem*}{Theorem}
\newtheorem*{lemma*}{Lemma}
\newtheorem*{proposition*}{Proposition}
\newtheorem*{corollary*}{Corollary}
\newtheorem*{definition*}{Definition}
\newtheorem*{remark*}{Remark}

\newcommand{\Ebb}{\mathbb{E}}
\newcommand{\Rbb}{\mathbb{R}}
\newcommand{\Pbb}{\mathbb{P}}
\newcommand{\Ncal}{\mathcal{N}}

\title{Augmented Hypothesis Testing with Persona-Based LLM Simulations}

\author{%
  Ziyad Benomar \\
  Amazon \\
  Luxembourg, Luxembourg \\
  \texttt{zbenomar@amazon.lu} \\
  \And
  Aymen Al Marjani \\
  Amazon \\
  Luxembourg, Luxembourg \\
  \texttt{almarjan@amazon.lu} \\
  \And
  Paul Missault \\
  Amazon \\
  Luxembourg, Luxembourg \\
  \texttt{pmissaul@amazon.lu} \\
  \And
  Saab Mansour \\
  Amazon \\
  Barcelona, Spain \\
  \texttt{saabm@amazon.es}
}

\begin{document}

\ifcolmsubmission
\linenumbers
\fi

\maketitle

\begin{abstract}
A/B testing requires large sample sizes, long timelines, and significant costs. When auxiliary predictions of experimental outcomes are available from machine learning models, uncertain prediction quality precludes replacing human experiments entirely, yet these predictions may still contain useful signal. We propose a principled framework for learning-augmented hypothesis testing that leverages predictions of unknown quality to reduce sample sizes while maintaining statistical validity. Predictions naturally vary in granularity, from coarse aggregate signals to fine-grained individual-level estimates, and our framework addresses both ends of this spectrum: (1) for population-level directional predictions, where only a binary signal on the treatment effect sign is available, we use an asymmetric test and prove consistency and robustness bounds within the learning-augmented algorithms paradigm; (2) for individual-level predictions, we introduce Generalized PPI++ (GPPI), extending Prediction-Powered Inference to handle nonlinear prediction errors through higher-dimensional transformations. Both methods benefit from accurate predictions while remaining robust to inaccurate or adversarial ones. We validate our framework using persona-based LLM simulations, where AI agents equipped with user personas predict individual behavior, as a natural prediction source spanning both granularity levels. Experiments on four real-world datasets demonstrate that our methods, combined with persona-based predictions, substantially reduce experimental costs while preserving rigorous statistical validity.
\end{abstract}


\section{Introduction}


Hypothesis testing through A/B experiments remains the gold standard for behavioral data decision-making, with organizations conducting thousands of tests annually to optimize products and services~\citep{kohavi2009controlled, xu2015infrastructure}. However, traditional A/B tests impose substantial costs: they require exposing different versions to distinct user groups, collecting behavioral data over extended periods, and 
hundreds of thousands or millions of interactions. These requirements translate directly into lost opportunity costs, delayed decisions, and significant infrastructure expenses~\citep{deng2017data}.

Recent advances in LLMs~\citep{brown2020language, openai2023gpt4, touvron2023llama} have enabled a promising alternative: \emph{persona-based simulation}, where AI agents equipped with detailed user personas predict experimental outcomes~\citep{park2023generative, horton2023large, argyle2023out,  mansour2025paars}. A persona encodes a user's historical behavior, preferences, and characteristics in textual form. By providing this persona to an LLM and asking it to predict how the user would respond to new stimuli, researchers can simulate behavior across a population of personas, potentially predicting A/B test outcomes before conducting costly human experiments~\citep{aher2023using, manning2024automated}. This approach has motivated growing research in agentic simulations for behavior modeling, with early results suggesting that LLM-based agents can capture meaningful patterns in human decision-making~\citep{grossmann2023ai, dillion2023can}.

However, completely replacing human A/B tests with persona-based predictions faces fundamental challenges.
The black-box nature of LLMs, sensitivity to prompt engineering~\citep{zhao2021calibrate, lu2022fantastically}, distribution shifts, and the inherent complexity of human behavior~\citep{skinner1965science} make formal quality guarantees difficult. Moreover, persona-based simulations introduce a critical structural consideration: while a collection of personas may approximate the overall population distribution, establishing reliable one-to-one mappings between specific personas and specific users is often infeasible.
This asymmetry means that persona simulations naturally produce different types of predictions depending on the simulation design—from coarse population-level directional signals (e.g., observing the sign of the treatment effect on the persona population) to fine-grained individual-level predictions when persona-user alignment is available.

Rather than treating LLM predictions as perfect substitutes for human experiments, more principled approaches propose leveraging them to \emph{reduce} the required human sample size while maintaining statistical validity \citep{angelopoulos2023prediction,ji2026overview, calderon2025alternative}. These approaches acknowledge prediction uncertainty while extracting value when predictions contain useful signal. The fundamental question becomes: given predictions of unknown quality from persona-based LLM simulations (or other sources), how can we design hypothesis testing procedures that provably benefit from accurate predictions while remaining robust to inaccurate ones?

\paragraph{Learning-augmented algorithms.} Our work contributes to the literature on learning-augmented algorithms~\citep{lykouris2018competitive, mitzenmacher2022algorithms}, which addresses how to leverage machine learning predictions lacking formal guarantees. The key criteria for designing such algorithms are (i) improved performance when predictions are accurate (\emph{consistency}), (ii) guarantees comparable to prediction-free methods when predictions fail (\emph{robustness}), (iii) and graceful degradation as prediction error increases (\emph{smoothness})~\citep{purohit2018improving}. This framework has been initially studied in the context of online algorithms such as caching~\citep{lykouris2018competitive, rohatgi2020near}, scheduling~\citep{lattanzi2020online, mitzenmacher2020scheduling, benomar2024non, benomar2026non}, ski-rental \citep{purohit2018improving,bamas2020primal,benomar2023advice} and online selection~\citep{dutting2021secretaries, antoniadis2020secretary}, it extended to designing data structures \citep{lin2022learning, mccauley2023online, zeynalirobust, benomar2024learning}, and improving machine learning algorithms and statistical methods \citep{bhaskara2021logarithmic,golowich2022can, raman2024online}. The typical objective is establishing a good tradeoff between consistency and robustness \citep{benomar2025tradeoffs}.

\paragraph{Prediction-Powered Inference (PPI).} PPI++~\citep{angelopoulos2023prediction, angelopoulos2023ppi} extends the learning-augmented setting to statistical inference, addressing hypothesis testing in semi-supervised settings with $n$ labeled samples, $N$ unlabeled samples, and a predictor of unknown quality. PPI estimates systematic predictor bias using labeled data, then applies bias correction to predictions on unlabeled data. This provably improves inference quality for sufficiently large $n$, reducing confidence intervals and increasing power, with improvement magnitude depending on prediction-label correlation. When predictions are perfect, PPI achieves the performance of $N$ additional labeled samples; when uninformative, it gracefully degrades to using only $n$ labeled samples. PPI++ can be more generally extended to semi-supervised settings beyond inference, for example in online optimization \citep{shoham2025prediction}, multi-armed bandits \citep{ji2025multi}, ranking \citep{chatzi2024prediction}, and inference with multiple predictors under budget constraints \citep{cowen2026multippi}

\paragraph{Different prediction types.}  
Predicting AB test outcomes can be framed as regression or classification tasks. For instance, \citet{li2015toward} develop a regression model to predict expected reward values and classify experiments as WIN, TIE, or LOSS when comparing treatment to control variants, which gives a \textit{population-level} prediction. Some works also explored \textit{individual-level} behavior prediction from historical data \citep{sales2018using} using classical ML methods. These different types of predictions can be used in the learning-augmented framework to improve AB testing. A more recent approach to obtain either population- or individual-level predictions is through persona-based LLM simulation \citep{hu2024quantifying, chandrasegaran2025synthetic, li2025llm}. When persona-user alignment is unavailable (for e.g. synthetic personas), these simulations provide population-level predictions \citep{mansour2025paars, castelo2026simgym, rieder2026simab}.

When persona-user mapping is possible (e.g., through demographic or behavioral matching), simulations may yield \emph{individual-level} predictions, though with uncertain quality. This distinction is fundamental: the simulation architecture determines what type of prediction is available, which in turn dictates what statistical methods are appropriate.

\subsection{Contributions and Organization}

\paragraph{Theoretical contributions.}
\emph{(1) Population-level directional predictions:} We consider hypothesis testing with binary prediction on the sign of the effect, i.e. the prediction only signals which treatment version is better. In this setting, we analyze hypothesis testing within the learning-augmented algorithms framework, characterizing the fundamental trade-off between leveraging predictions when accurate and maintaining validity when incorrect. For that, we used an asymmetric Z-test and prove consistency and robustness bounds.

\emph{(2) Individual-level predictions:}
When granular predictions of individual engagement with both treatments are available, Prediction-Powered Inference (PPI/PPI++) \citep{angelopoulos2023ppi} can leverage these predictions to enhance hypothesis testing. We propose Generalized PPI (GPPI), extending to handle nonlinear prediction errors through higher dimension transformations. We also prove that, given a set of candidate transformations, a greedy selection approach can achieve the same performance as the best candidate, then we show how the performance scales with the dimension of such transformations.

\paragraph{Persona-based simulation validation.}
A key contribution is applying PPI++ and our theoretical framework to persona-based LLM simulations. We define personas using users' action histories and employ LLMs with these personas to predict individual ratings on new items. This can be translated into preference simulation for AB testing, where the objective is determining which item has statistically higher average ratings. Given a pair of items, we compute rating differences across personas. This validates our framework with realistic AI predictions exhibiting complex error patterns. We conduct experiments on four real-world datasets: \emph{MovieLens-25M}~\citep{harper2015movielens}, \emph{Jester}~\citep{goldberg2001eigentaste}, \emph{Last.fm}~\citep{Bertinmahieux2011}, and \emph{Book-Crossing}~\citep{ziegler2005improving}. We use three Mistral models to make predictions: Mistral-7B-Instruct-v0.2 (7B parameters), Mistral-Small-24B-Instruct-2501 (24B parameters), and Mistral-Large-Instruct-2407 (123B parameters), to which we refer respectively as \emph{mistral-7b}, \emph{mistral-24b}, and \emph{mistral-123b} in the following. These models were chosen because they are open-source, which ensures reproducibility of the experiments, and because they have different parameter sizes, which demonstrates that our methods work across different model scales

This comprehensive evaluation demonstrates that persona-based simulations, when combined with our statistical framework (either on individual or population level), can reduce experimental costs while maintaining rigorous validity guarantees.

\section{Hypothesis testing with directional predictions}
\label{sec:directional}

We investigate how population-level directional predictions can improve hypothesis testing in A/B testing scenarios. In many practical settings, the key decision is simply determining whether a treatment effect is positive or negative, for example whether a new feature increases or decreases user engagement. We consider the setting where an analyst has access to a binary prediction indicating the sign of the treatment effect. This represents the minimal type of prediction assistance, i.e. a single bit of directional information.

\paragraph{Problem setup.} 
Consider i.i.d. observations $X_1, \ldots, X_n \sim \mathcal{N}(\mu,\sigma^2)$ with known $\sigma$, and let $\bar{X} = \frac{1}{n}\sum_{i=1}^n X_i$. We test $H_0: \mu = \mu_0$ versus $H_1: \mu \neq \mu_0$ using the Z-test statistic $Z = (\bar{X} - \mu_0)/(\sigma/\sqrt{n})$. Under $H_0$, we have $Z \sim \mathcal{N}(0, 1)$. Under $H_1$ with $\mu = \mu_0 + \delta$, the statistic follows $Z \sim \mathcal{N}(\theta, 1)$ where $\theta = \delta\sqrt{n}/\sigma$ is the non-centrality parameter. 

The classical two-sided test at significance level $\alpha$ rejects $H_0$ when $|Z| > z_{1-\alpha/2}$, where $z_{1-\alpha/2} = \Phi^{-1}(1-\alpha/2)$. This symmetric critical region allocates equal probability $\alpha/2$ to each tail, reflecting no prior directional information. While we assume exact normality for analytical convenience, the Central Limit Theorem ensures asymptotic validity even with non-normal data. Test quality is measured by power $\text{Power}(\theta)$, the probability of correctly rejecting $H_0$ when $\theta \neq 0$.

\paragraph{Asymmetric tests with predictions.} 
Suppose we have a directional prediction $\hat{s} \in \{-1, +1\}$ for $\text{sign}(\delta)$, potentially from domain expertise, preliminary analysis, or of by aggregating persona-based simulations. We perform an asymmetric test parameterized by $\lambda \in [0, 1)$ controlling the level of asymmetry. Without loss of generality, assume $\hat{s} = +1$. We define asymmetric tail probabilities:
$\alpha_L = \frac{1-\lambda}{2}\alpha$, and 
$\alpha_R = \frac{1+\lambda}{2}\alpha$,
where $\alpha_L + \alpha_R = \alpha$ maintains the desired significance level. The critical values are $c_L(\lambda) = z_{(1-\lambda)\alpha/2}$ and $c_R(\lambda) = z_{1-(1+\lambda)\alpha/2}$, where $z_p = \Phi^{-1}(p)$. We reject $H_0$ if $Z < c_L(\lambda)$ or $Z > c_R(\lambda)$. 
By allocating more Type I error to the predicted tail (increasing $\alpha_R$ when $\hat{s} = +1$), we lower the threshold $c_R(\lambda)$, facilitating rejection when the true effect aligns with the prediction. Conversely, we raise $|c_L(\lambda)|$, making rejection harder when the effect contradicts the prediction. The parameter $\lambda$ controls this bias strength, with $\lambda = 0$ recovering the symmetric classical test and $\lambda \to 1$ concentrating nearly all error probability in the predicted direction, which corresponds to the one-sided Z-test for that direction. This embodies a fundamental consistency-robustness trade-off.

\paragraph{Learning-augmented framework and performance guarantees.} 
Our procedure falls into the learning-augmented algorithms paradigm: leveraging auxiliary predictions to improve performance when accurate while maintaining acceptable performance when predictions fail. We formalize this through approximation ratios comparing our test's power to two natural benchmarks. When predictions are correct ($\text{sign}(\theta) = \hat{s}$), we measure consistency relative to the one-sided test ($\lambda = 1$):
\[
\text{Consistency ratio} = \inf_{\alpha, \theta: \text{sign}(\theta) = \hat{s}}\frac{\text{Power}(\theta, \alpha, \lambda)}{\text{Power}(\theta, \alpha, 1)}\;.
\]
When predictions are incorrect ($\text{sign}(\theta) \neq \hat{s}$), we measure robustness relative to the prediction-agnostic two-sided baseline ($\lambda = 0$):
\[
\text{Robustness ratio} = \inf_{\alpha, \theta: \text{sign}(\theta) \neq \hat{s}}\frac{\text{Power}(\theta, \alpha, \lambda)}{\text{Power}(\theta, \alpha, 0)}\;.
\]
These ratios are worst-case guarantees. A consistency $c$ means achieving at least $c$-fraction of the oracle's power if the prediction is correct, while a robustness $r$ means retaining at least $r$-fraction of baseline power regardless of the quality of the prediction.
Our main theorem establishes that the asymmetric test has consistency and robustness guarantees that scale linearly with $\lambda$.

\begin{theorem}[Consistency-Robustness Trade-off]
\label{thm:main-directional-pred}
Assume without loss of generality that the predicted direction is $\hat{s} = +1$.
Let $\lambda \in [0,1]$. For the asymmetric test with parameter $\lambda$:
\begin{enumerate}[label=(\roman*), leftmargin=*]
\item \textbf{Consistency.} For all $\alpha \in (0,1)$ and $\theta > 0$ (correct prediction): $\frac{\mathrm{Power}(\theta, \alpha, \lambda)}{\mathrm{Power}(\theta, \alpha, 1)} \geq \frac{1+\lambda}{2}$

\item \textbf{Robustness.} For all $\alpha \in (0,1)$ and $\theta < 0$ (incorrect prediction): $\frac{\mathrm{Power}(\theta, \alpha, \lambda)}{\mathrm{Power}(\theta, \alpha, 0)} \geq 1-\lambda$
\end{enumerate}
\end{theorem}

Theorem~\ref{thm:main-directional-pred} reveals a linear trade-off structure. Setting $\lambda = 0$ yields perfect robustness but only 50\% consistency, as the symmetric test achieves half the power of the oracle one-sided test in each direction. Conversely, setting $\lambda \to 1$ yields perfect consistency but zero robustness. Intermediate values $\lambda \in (0,1)$ interpolate smoothly these extreme behaviors.

\section{Hypothesis testing with individual-level predictions}\label{sec:idividual}

Hypothesis testing for effect size estimation is directly related to mean estimation and confidence interval construction. In A/B testing, we typically estimate means for both treatments and construct confidence intervals to determine the effect size sign. This is why in this section we focus on mean estimation in the setting where individual-level predictions are available: rather than a single population-level prediction as in the previous section, we assume access to predictions for each individual. When agentic simulations or machine learning models generate individual-level outcome estimates, we can leverage substantially more information to improve statistical inference. This aligns with the Prediction-Powered Inference (PPI) framework \citep{angelopoulos2023prediction, angelopoulos2023ppi}, which utilizes both labeled data (ground truth outcomes) and abundant unlabeled data (predictions only).

\paragraph{Problem setup.} 
We observe labeled data $(X_i, Y_i)$ for $i \in [n]$, unlabeled data $\tilde{X}_{j}$ for $j \in [N]$ (typically $N \gg n$), and have access to predictor $f: \mathcal{X} \to \mathbb{R}$ available for all observations. Our goal is to estimate the population mean $\theta^* = \mathbb{E}[Y]$ and construct a valid confidence interval leveraging the predictions. The classical estimator $\hat{\theta}_{\text{classical}} = \frac{1}{n}\sum_{i=1}^{n} Y_i$ ignores unlabeled data entirely. 
PPI++ \citep{angelopoulos2023ppi} improves upon this by constructing
\begin{equation}\label{eq:ppi-estimator}
\widehat{\theta}_{\text{PPI}} = \frac{\lambda}{N}\sum_{j=1}^{N} f(\tilde{X}_j) + \frac{1}{n}\sum_{i=1}^{n} (Y_i - \lambda f(X_i)) \;,
\end{equation}
where $\lambda \in \mathbb{R}$ minimizes asymptotic variance. This estimator combines scaled predictions on unlabeled data with a correction term learned from labeled data that captures systematic deviation between $Y_i$ and $\lambda f(X_i)$. The structure applies a multiplicative then additive correction to predictions, which is effective when $Y \approx \lambda^* f(X) + b$ for constants $\lambda^*, b$, up to noise. However, when the relationship is more complex, for example $Y = G(f(X))$ for a nonlinear function $G$, the linear correction is insufficient.

\subsection{Generalized PPI++ (GPPI)} 
To overcome the limitation above, we extend PPI++ by transforming predictions through a transformation $\phi: \mathbb{R} \to \mathbb{R}^d$, for instance $\phi(z) = (z, z^2, \ldots, z^d)^T$. Define sample averages
\[
\bar{Y}_n = \frac{1}{n}\sum_{i=1}^n Y_i, \quad 
\bar{\phi}_n^L = \frac{1}{n}\sum_{i=1}^n \phi(f(X_i)), \quad 
\bar{\phi}_N^U = \frac{1}{N}\sum_{j=1}^{N} \phi(f(\tilde{X}_j)).
\]
The GPPI estimator is then
\begin{equation} \label{eq:gppi-estimator}
\hat{\theta}_{\text{GPPI}} = \bar{Y}_n + \lambda^T(\bar{\phi}_N^U - \bar{\phi}_n^L)
\end{equation}
where $\lambda \in \mathbb{R}^d$ is a weight vector. Assuming $n/N \to r$, the optimal weight minimizing asymptotic variance is 
$\lambda^* = \frac{1}{1+r} \Sigma_\phi^{-1} \sigma_{Y,\phi}$, 
where $\Sigma_\phi = \mathrm{Cov}(\phi(f(X)))$ and $\sigma_{Y,\phi} = \mathrm{Cov}(Y, \phi(f(X)))$. We estimate these using sample covariances:
\begin{align}
\hat{\Sigma}_\phi &= \frac{1}{N-1}\sum_{j=1}^{N} (\phi(f(\tilde{X}_j)) - \bar{\phi}_N^U)(\phi(f(\tilde{X}_j)) - \bar{\phi}_N^U)^T \label{eq:sigma-phi-hat} \\
\hat{\sigma}_{Y,\phi} &= \frac{1}{n-1}\sum_{i=1}^n (Y_i - \bar{Y}_n)(\phi(f(X_i)) - \bar{\phi}_n^L) \label{eq:sigma-y-phi-hat}
\end{align}
yielding the plug-in estimator $\hat{\lambda} = \frac{1}{1+n/N} \hat{\Sigma}_\phi^{-1} \hat{\sigma}_{Y,\phi}$.

Our main result proves how to construct an asymptotically valid confidence interval using the predictions and their augmentation via transformation.

\begin{theorem}
\label{thm:gppi-main}
Assume that $\mathbb{E}[Y^2] < \infty$, $\mathbb{E}[\|\phi(f(X))\|^2] < \infty$, $r_n := n/N \to r > 0$, and $\Sigma_\phi = \mathrm{Cov}(\phi(f(X)))$ is positive-definite. Then:
\begin{enumerate}[label=(\roman*), leftmargin=*]
\item Let $V^* = \sigma_Y^2 - \frac{1}{1+r} \sigma_{Y,\phi}^T \Sigma_\phi^{-1} \sigma_{Y,\phi}$, then we have $\sqrt{n}(\hat{\theta}_{\text{GPPI}} - \theta^*) \xrightarrow{d} \mathcal{N}(0, V^*)$
\item Define respectively the residual variance and the variance estimator
$$
\hat{\sigma}_R^2 = \frac{1}{n-1}\sum_{i=1}^n [(Y_i - \bar{Y}_n) - \hat{\lambda}^T(\phi(f(X_i)) - \bar{\phi}_n^L)]^2
\quad \text{ and } \quad
\hat{V} = \hat{\sigma}_R^2 + r_n \hat{\lambda}^T \hat{\Sigma}_\phi \hat{\lambda}\;,
$$
Then $\hat{V} \xrightarrow{p} V^*$.
\item
$\text{CI}_{1-\alpha} = [\hat{\theta}_{\text{GPPI}} \pm z_{\alpha/2} \sqrt{\hat{V}/n}\;]$
satisfies $\lim_{n\to\infty} \mathbb{P}(\theta^* \in \text{CI}_{1-\alpha}) = 1-\alpha$.
\end{enumerate}
\end{theorem}

The proof of Theorem~\ref{thm:gppi-main} follows the general approach of \citet{angelopoulos2023prediction}, which establishes asymptotic normality for PPI estimators in the context of general M-estimators. Our result extends mean-estimation to $d$-dimensional transformations $\phi(\cdot)$ and vector-valued weights $\lambda \in \mathbb{R}^d$.

\subsection{GPPI with greedy transformation Selection}
\label{sec:basis-selection}
While Theorem~\ref{thm:gppi-main} establishes GPPI's validity with a fixed transformation $\phi$, the choice of this transformation critically affects performance. Although function classes like polynomials can approximate the optimal transformation under appropriate assumptions with sufficient dimension, this creates a fundamental trade-off: high-dimensional bases better approximate the relationship between $f(X)$ and $Y$ but yield noisy coefficient estimates with limited labeled data. The challenge is that decision-makers lack prior knowledge of the underlying relationship, hence cannot decide beforehand which transformation to consider.

We propose a greedy selection algorithm from a candidate set of $m$ transformations, relaxing the need to specify a single transformation a priori.
Let $\phi^1, \ldots, \phi^m$ be candidate transformations where $\phi^j: \mathbb{R} \to \mathbb{R}^{d^j}$ with $d^j \leq d$ for all $j \in [m]$. Each transformation induces an asymptotic variance $V_j$ in the estimator \eqref{eq:gppi-estimator}. Since width of the confidence interval built by GPPI scales with $\sqrt{V_j}$, we seek $j^* = \arg\min_{j \in [m]} V_j$. Our greedy algorithm is
\begin{enumerate}[label=(\roman*), leftmargin=*]
\item For each $\phi^j$, compute the plug-in variance estimator using Theorem~\ref{thm:gppi-main}(ii):
\begin{equation}\label{eq:variance-estimator-j}
\hat{V}_j = \hat{\sigma}_{R,j}^2 + r_n \hat{\lambda}_j^T \hat{\Sigma}_{\phi^j} \hat{\lambda}_j
\end{equation}
where $\hat{\lambda}_j = \frac{1}{1+r_n} \hat{\Sigma}_{\phi^j}^{-1} \hat{\sigma}_{Y,\phi^j}$ and $\hat{\sigma}_{R,j}^2$ is the residual variance for transformation $j$.
\item Select $\hat{j} = \arg\min_{j \in [m]} \hat{V}_j$.
\item Construct the GPPI estimator using $\phi^{\hat{j}}$.
\end{enumerate}

The following theorem shows this procedure asymptotically achieves the optimal transformation variance.

\begin{theorem}[Greedy transformation Selection]
\label{thm:basis-selection}
Let $\phi^1, \ldots, \phi^m$ be candidate transformations with dimensions $d^j \leq d$, where $d$ is a fixed constant. Assume for all $j \in [m]$ that $\|\phi^j(f(X))\|^2 \leq M$ and $Y^2 \leq K_Y$ almost surely, $\Sigma_{\phi^j} = \mathrm{Cov}(\phi^j(f(X)))$ is positive-definite, and there exists a unique minimizer $j^* = \arg\min_{j \in [m]} V_j$. Let $\hat{j}$ be the selected transformation index. Then:
\begin{enumerate}[label=(\roman*), leftmargin=*]
\item \textbf{Consistent selection:} $\hat{j} \xrightarrow{p} j^*$ as $n \to \infty$.
\item \textbf{Asymptotic normality:} The greedy GPPI estimator satisfies 
$\sqrt{n}(\hat{\theta}_{\text{GPPI}}^{\hat{j}} - \theta^*) \xrightarrow{d} \mathcal{N}(0, V_{j^*})\;.$
\item \textbf{Variance bound:} For all $\delta \in (0,1)$, we have with probability $1-\delta$ that
\[
\hat{V}_{\hat{j}} \leq \min_{j \in [m]} V_j + C K_Y M^2 \sqrt{\frac{\log(md/\delta)}{n}}\;.
\]
\end{enumerate}
\end{theorem}

Note that Theorem~\ref{thm:basis-selection} requires $Y$ and $\phi(f(X))$ to be bounded almost surely, rather than finite second moments as in Theorem~\ref{thm:gppi-main}. This assumption enables high-probability bounds.

A natural question concerns coverage: does it deteriorate (in finite sample regime) as the number of candidate transformations increases? We prove in \ref{cor:greedy_gppi_coverage} that the nominal coverage is also asymptotically preserved by greedy GPPI. 

\paragraph{Other aggregation strategies}

Theorem~\ref{thm:basis-selection} (iii) demonstrates that the confidence interval width, proportional to $\sqrt{\hat{V}_{\hat{j}}}$, approaches the optimal asymptotic width among candidate transformations. Alternatively, we could stack all transformations into an $md$-dimensional transformation and estimate the optimal linear combination from labeled data. While this may improve asymptotic performance, it creates a crucial trade-off: since $M$ bounds $\|\phi^j(f(X))\|_2$, it scales with dimension $d$ sometimes polynomially or even exponentially. Therefore, the greedy selection achieves a performance near the best transformation with minimal convergence cost, whereas stacking transformations may yield even better asymptotic variance but suffers from parameter estimation variance with limited labeled data.

In practice, PPI++ typically works with black-box predictors $f$ (e.g., LLMs, pre-trained models) that provide reasonable signals but require minor corrections for biases or behavioral changes due to distribution shifts,  systematic over/underestimation, threshold adjustments, or capping effects. In these settings, corrections often require only simple transformations (linear, polynomial, threshold functions,...). Making a broad initial guess of possible transformations and applying greedy selection provides stable improvements without the variance penalty of high-dimensional parameter estimation from finite labeled data. If the model needs significant and complex modification, then further training might be a better approach.

\subsection{Dimension scaling}
Following the discussion above, we dive deeper into the impact of high dimensional transformations 
In practice, GPPI provides the greatest benefits when predictions exhibit systematic nonlinear errors. However, the transformation dimension must remain moderate to avoid large variance from estimating $\lambda$. The following result characterizes the permissible growth rate of dimension $d$ with sample size $n$.

\begin{theorem}
\label{thm:dimension-scaling}
Let $\phi_d: \mathbb{R} \to \mathbb{R}^d$ denote a $d$-dimensional transformation where $d=d_n$ may grow with $n$. Assume:
\begin{enumerate*}[label=(\roman*), leftmargin=*]
\item $\mathbb{E}[Y^2] < \infty$ and $\lim_{n\to\infty} n/N = r > 0$
\item $\|\phi_d(f(X))\|^2 \leq M_d$ almost surely
\item The eigenvalues of $\Sigma_d = \mathrm{Cov}(\phi_d(f(X)))$ satisfy $\lambda_{\min}(\Sigma_d) \geq c_d > 0$
\item The asymptotic variance $V^*_d$ converges to a limit $V_\infty$
\end{enumerate*}
If $M_d^2 \sqrt{\log d} / c_d^2 = o(\sqrt{n})$, then
\[
\sqrt{n}(\hat{\theta}_{\text{GPPI}} - \theta^*) \xrightarrow{d} \mathcal{N}(0, V_\infty), \quad 
\hat{V} \xrightarrow{p} V_\infty,
\]
and the confidence interval $\text{CI}_{1-\alpha}$ achieves asymptotic coverage probability $1-\alpha$.
\end{theorem}
\section{Experiments}
We validate our findings through experiments on population-level directional predictions (Section~\ref{sec:directional}) and individual-level predictions (Section~\ref{sec:idividual}) using persona-based LLM simulations on four datasets: (i) \emph{MovieLens-25M}~\citep{harper2015movielens}: 25M ratings from 162K users on 59K movies, (ii) \emph{Jester}~\citep{goldberg2001eigentaste}: 1.4M ratings from 26K users on 140 jokes, (iii) \emph{Last.fm}~\citep{Bertinmahieux2011}: 423K listening events from 8.6K users on 20K tracks, and (iv) \emph{Book-Crossing}~\citep{ziegler2005improving}: 296K ratings from 7.3K users on 146K books.
For each dataset, we select 15 test items based on recency and popularity  (detailed in Appendix~\ref{app:experimental-details})), with remaining items forming the training set for persona construction. Given a user's historical interactions, we construct textual personas summarizing their preferences and provide these to LLMs for test item predictions. We use three models to make predictions: \emph{mistral-7b}, \emph{mistral-24b}, and \emph{mistral-123b}. Comprehensive experimental details and additional results are provided in Appendix~\ref{app:experimental-details}.

We compare GPPI with greedy transform selection against PPI++. For numerical stability, we use ridge regularization in estimating the weights: $\hat{\lambda} = \frac{1}{1+n/N} (\hat{\Sigma}_\phi + \gamma I)^{-1} \hat{\sigma}_{Y,\phi}$ with $\gamma = 10^{-3}$. 
Furthermore, to avoid biasing the greedy selection algorithm towards selecting high dimensional transformations, we implement two penalized selection variants. Denoting by $d_j$ the dimension of the transformation $\phi^j$:
\begin{itemize}
    \item \textbf{Greedy AIC}: selects the family minimizing $\text{Var}(\hat{\theta}^\text{GPPI}(\phi^j)) + 2d_j/n$. This penalization proportional to $d$ is inspired by the Akaike-type correction~\citep{akaike1974new} for model selection.
    \item \textbf{Greedy BIC}: selects the family minimizing $\text{Var}(\hat{\theta}^\text{GPPI}(\phi^j)) + d_j \log(n)/n$, which is inspired by the  Bayesian information criterion penalty~\citep{schwarz1978estimating}.
\end{itemize}

To correct nonlinear prediction biases without prior knowledge of the error structure, we use $m=10$ complementary transformations $\phi^j: \mathbb{R} \to \mathbb{R}^{d_j}$ with dimensions $d \leq 6$, selected to span diverse function classes while maintaining small dimension:
\begin{itemize}[nosep, leftmargin=*]
    \item \textbf{Polynomials}: standard and Bernstein polynomials ($d \in \{3,5\}$)
    \item \textbf{Logarithmic}: $p \mapsto (\log(|p|+1), \ldots, \log(|p|+1)^3)$, handle magnitude-dependent errors
    \item \textbf{Piecewise}: cubic B-splines and tent functions ($d=4$), provide corrections for local biases
    \item \textbf{Sigmoidal}: logistic and softplus functions ($d \in \{3,5\}$), capture saturation effects and bounded prediction ranges
\end{itemize}

To all families, we add the identity transformation, ensuring PPI++ remains recoverable. 

As explained in Section \ref{sec:idividual}, GPPI is designed to guarantee better mean estimation than PPI++, and this naturally extends to better guarantees also for hypothesis testing and A/B tests. We experimentally test it on both tasks.

\subsection{Mean estimation with individual-level predictions}

We measure improvement using effective sample size (ESS) gain~\citep{krsteski2025valid}:
$\text{ESS}_{\text{gain}\%} = \left(\frac{\text{Var}(\hat{\theta}_{\text{classical}})}{\text{Var}(\hat{\theta}_{\text{method}})} - 1\right) \times 100$,
which measures how much the sample size can be reduced with the method we're testing (PPI, GPPI, greedy selection, ...) compared to the baseline estimator that uses only labeled data, while maintaining the same level of precision in the estimation. The higher this metric the better is the improvement.

We conduct a first experiment with $N = 5n$ over 5000 Monte Carlo repetitions, sampling $n$ labeled users and $N$ unlabeled users to compute estimators and 95\% confidence intervals. The figure below show ESS gains for Jester datasets, with columns representing items and rows representing different models. The subplot corresponding to each item and model shows ESS gain across different sample sizes $n$ and different estimators.

\begin{figure}[h]
    \centering
    \includegraphics[width=0.9\linewidth]{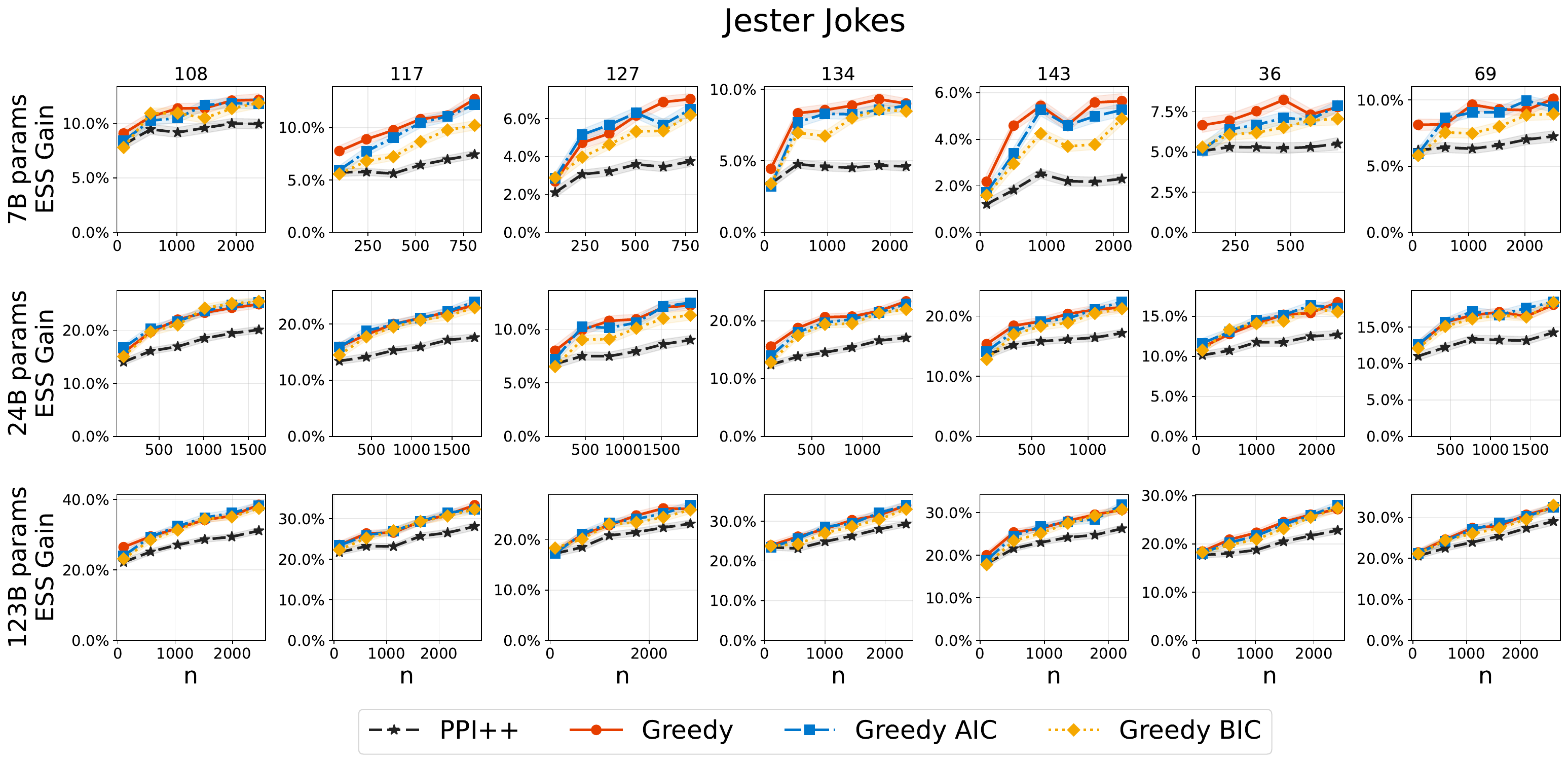}
    \caption{ESS gain comparison on Jester dataset across different jokes and Mistral models}
    \label{fig:jester-gppi}
\end{figure}
In Appendix \ref{app:heatmap-analysis}, we present results for the three remaining datasets and compare the coverage of PPI++ and GPPI across all four datasets. The experiments demonstrate that GPPI with greedy selection consistently outperforms PPI++ in ESS gain while maintaining comparable coverage. For the Book Crossing and Last.FM datasets, coverage falls below the nominal level for certain items; however, these failure modes are common to both PPI++ and GPPI. They occur for items whose individual values contain outliers or exhibit heavy-tailed distributions, which violate the finite-sample CLT approximation.

\subsection{A/B testing with population- and individual-level predictions}
\label{sec:exp-population}

We evaluate our framework's effectiveness in hypothesis testing by comparing item pairs to determine which has higher average ratings. 
We compare individual- and population-level predictions against classical estimators using ESS gain as our quality measure, with $N=5n$. Here, ESS gain is defined as the reduction of the sample size that we can achieve compared to a classical A/B test using only the labeled data, while maintaining the same power of the test. In this experiment, we set the type I error to $\alpha=5\%$.
For population-level predictions, we aggregate all available model predictions and compare mean predicted ratings. We test asymmetry values $\lambda \in \{0.25, 0.5, 0.75\}$.

Figures \ref{fig:ab_movielens} and \ref{fig:ab_jester} show that GPPI with greedy transform selection consistently outperforms PPI++, with both methods achieving significant improvements over the classical baselines which only uses labeled data, up to 30\% sample size reduction. Population-level predictions yield substantial improvements when predictions are accurate (MovieLens), but can negatively impact performance with noisy predictions (Jester).

\begin{figure}[h]
    \centering
    \includegraphics[width=0.8\linewidth]{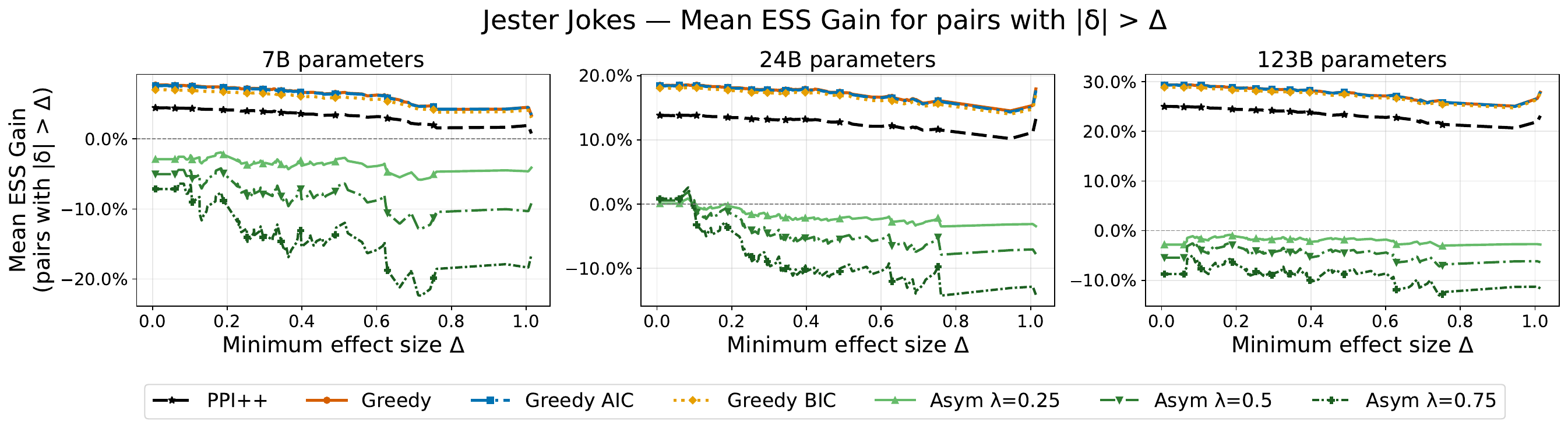}
    \caption{Comparison of the learning-augmented A/B testing approaches for Jester dataset}
    \label{fig:ab_jester}
\end{figure}

While the experiments above were not conducted in a live production A/B test environment, the methodology and conclusions directly translate to such settings, as A/B tests fundamentally rely on hypothesis testing to decide the best treatment.

\section{Conclusion}

We presented a framework for learning-augmented hypothesis testing that leverages predictions to reduce experimental costs while maintaining validity. For population-level directional predictions, we established a consistency-robustness trade-off through asymmetric tests with linear guarantees. For individual-level predictions, we introduced Generalized PPI (GPPI) with greedy transformation selection, extending Prediction-Powered Inference to handle nonlinear prediction errors.
In particular, we applied our methods to  persona-based LLM simulations on four real-world datasets,  demonstrating the usefulness of our framework and how it improves over prior methods.

\bibliographystyle{colm2026_conference}
\bibliography{bibliography}

@article{angelopoulos2023prediction,
  title={Prediction-powered inference},
  author={Angelopoulos, Anastasios N and Bates, Stephen and Fannjiang, Clara and Jordan, Michael I and Zrnic, Tijana},
  journal={Science},
  volume={382},
  number={6671},
  pages={669--674},
  year={2023},
  publisher={American Association for the Advancement of Science}
}

@article{angelopoulos2023ppi,
  title={Ppi++: Efficient prediction-powered inference},
  author={Angelopoulos, Anastasios N and Duchi, John C and Zrnic, Tijana},
  journal={arXiv preprint arXiv:2311.01453},
  year={2023}
}

@book{van2000asymptotic,
  title={Asymptotic statistics},
  author={Van der Vaart, Aad W},
  volume={3},
  year={2000},
  publisher={Cambridge university press}
}

@article{tropp2015introduction,
  title={An introduction to matrix concentration inequalities},
  author={Tropp, Joel A},
  journal={Foundations and Trends in Machine Learning},
  volume={8},
  number={1-2},
  pages={1--230},
  year={2015},
  publisher={Now Publishers, Inc.}
}

@article{kohavi2009controlled, 
  title={Controlled experiments on the web: survey and practical guide},
  author={Kohavi, Ron and Longbotham, Roger and Sommerfield, Dan and Henne, Randal M},
  journal={Data mining and knowledge discovery},
  volume={18},
  number={1},
  pages={140--181},
  year={2009},
  publisher={Springer}
}

@inproceedings{xu2015infrastructure,
  title={From infrastructure to culture: A/B testing challenges in large scale social networks},
  author={Xu, Ya and Chen, Nanyu and Fernandez, Addrian and Sinno, Omar and Bhasin, Anmol},
  booktitle={Proceedings of the 21th ACM SIGKDD international conference on knowledge discovery and data mining},
  pages={2227--2236},
  year={2015}
}

@inproceedings{deng2017data,
  title={Data-driven metric development for online controlled experiments: Seven lessons learned},
  author={Deng, Alex and Shi, Xiaolin},
  booktitle={Proceedings of the 22nd ACM SIGKDD International Conference on Knowledge Discovery and Data Mining},
  pages={77--86},
  year={2017}
}

@article{brown2020language,
  title={Language models are few-shot learners},
  author={Brown, Tom and Mann, Benjamin and Ryder, Nick and Subbiah, Melanie and Kaplan, Jared D and Dhariwal, Prafulla and Neelakantan, Arvind and Shyam, Pranav and Sastry, Girish and Askell, Amanda and others},
  journal={Advances in neural information processing systems},
  volume={33},
  pages={1877--1901},
  year={2020}
}

@article{openai2023gpt4,
  title={Gpt-4 technical report},
  author={Achiam, Josh and Adler, Steven and Agarwal, Sandhini and Ahmad, Lama and Akkaya, Ilge and Aleman, Florencia Leoni and Almeida, Diogo and Altenschmidt, Janko and Altman, Sam and Anadkat, Shyamal and others},
  journal={arXiv preprint arXiv:2303.08774},
  year={2023}
}

@article{touvron2023llama,
  title={Llama 2: Open foundation and fine-tuned chat models},
  author={Touvron, Hugo and Martin, Louis and Stone, Kevin and Albert, Peter and Almahairi, Amjad and Babaei, Yasmine and Bashlykov, Nikolay and Batra, Soumya and Bhargava, Prajjwal and Bhosale, Shruti and others},
  journal={arXiv preprint arXiv:2307.09288},
  year={2023}
}

@inproceedings{park2023generative,
  title={Generative agents: Interactive simulacra of human behavior},
  author={Park, Joon Sung and O'Brien, Joseph and Cai, Carrie Jun and Morris, Meredith Ringel and Liang, Percy and Bernstein, Michael S},
  booktitle={Proceedings of the 36th annual acm symposium on user interface software and technology},
  pages={1--22},
  year={2023}
}

@techreport{horton2023large,
  title={Large language models as simulated economic agents: What can we learn from homo silicus?},
  author={Horton, John J and Filippas, Apostolos and Manning, Benjamin S},
  year={2023},
  institution={National Bureau of Economic Research}
}

@article{argyle2023out,
  title={Out of one, many: Using language models to simulate human samples},
  author={Argyle, Lisa P and Busby, Ethan C and Fulda, Nancy and Gubler, Joshua R and Rytting, Christopher and Wingate, David},
  journal={Political Analysis},
  volume={31},
  number={3},
  pages={337--351},
  year={2023},
  publisher={Cambridge University Press}
}

@article{aher2023using,
  title={Using large language models to simulate multiple humans and replicate human subject studies},
  author={Aher, Gati V and Arriaga, Rosa I and Kalai, Adam Tauman},
  journal={International Conference on Machine Learning},
  pages={337--371},
  year={2023}
}

@techreport{manning2024automated,
  title={Automated social science: Language models as scientist and subjects},
  author={Manning, Benjamin S and Zhu, Kehang and Horton, John J},
  year={2024},
  institution={National Bureau of Economic Research}
}

@article{grossmann2023ai,
  title={AI and the transformation of social science research},
  author={Grossmann, Igor and Feinberg, Matthew and Parker, Dawn C and Christakis, Nicholas A and Tetlock, Philip E and Cunningham, William A},
  journal={Science},
  volume={380},
  number={6650},
  pages={1108--1109},
  year={2023},
  publisher={American Association for the Advancement of Science}
}

@article{dillion2023can,
  title={Can AI language models replace human participants?},
  author={Dillion, Danica and Tandon, Niket and Gu, Yuling and Gray, Kurt},
  journal={Trends in Cognitive Sciences},
  volume={27},
  number={7},
  pages={597--600},
  year={2023},
  publisher={Elsevier}
}

@inproceedings{zhao2021calibrate,
  title={Calibrate before use: Improving few-shot performance of language models},
  author={Zhao, Zihao and Wallace, Eric and Feng, Shi and Klein, Dan and Singh, Sameer},
  booktitle={International conference on machine learning},
  pages={12697--12706},
  year={2021},
  organization={Pmlr}
}

@inproceedings{lu2022fantastically,
  title={Fantastically ordered prompts and where to find them: Overcoming few-shot prompt order sensitivity},
  author={Lu, Yao and Bartolo, Max and Moore, Alastair and Riedel, Sebastian and Stenetorp, Pontus},
  booktitle={Proceedings of the 60th Annual Meeting of the Association for Computational Linguistics (Volume 1: Long Papers)},
  pages={8086--8098},
  year={2022}
}

@book{skinner1965science,
  title={Science and human behavior},
  author={Skinner, Burrhus Frederic},
  number={92904},
  year={1965},
  publisher={Simon and Schuster}
}

@article{lykouris2018competitive,
  title={Competitive caching with machine learned advice},
  author={Lykouris, Thodoris and Vassilvitskii, Sergei},
  journal={International Conference on Machine Learning},
  pages={3296--3305},
  year={2018}
}

@article{mitzenmacher2022algorithms,
  title={Algorithms with predictions},
  author={Mitzenmacher, Michael and Vassilvitskii, Sergei},
  journal={Communications of the ACM},
  volume={65},
  number={7},
  pages={33--35},
  year={2022}
}

@inproceedings{purohit2018improving,
  title={Improving online algorithms via ML predictions},
  author={Purohit, Manish and Svitkina, Zoya and Kumar, Ravi},
  booktitle={Advances in Neural Information Processing Systems},
  volume={31},
  year={2018}
}

@article{bamas2020primal,
  title={The primal-dual method for learning augmented algorithms},
  author={Bamas, Etienne and Maggiori, Andreas and Svensson, Ola},
  journal={Advances in Neural Information Processing Systems},
  volume={33},
  pages={20083--20094},
  year={2020}
}

@article{benomar2023advice,
  title={Advice querying under budget constraint for online algorithms},
  author={Benomar, Ziyad and Perchet, Vianney},
  journal={Advances in Neural Information Processing Systems},
  volume={36},
  pages={75026--75047},
  year={2023}
}

@inproceedings{rohatgi2020near,
  title={Near-optimal bounds for online caching with machine learned advice},
  author={Rohatgi, Dhruv},
  booktitle={Proceedings of the Fourteenth Annual ACM-SIAM Symposium on Discrete Algorithms},
  pages={1834--1845},
  year={2020}
}

@inproceedings{lattanzi2020online,
  title={Online scheduling via learned weights},
  author={Lattanzi, Silvio and Lavastida, Thomas and Moseley, Benjamin and Vassilvitskii, Sergei},
  booktitle={Proceedings of the Fourteenth Annual ACM-SIAM Symposium on Discrete Algorithms},
  pages={1859--1877},
  year={2020}
}

@inproceedings{mitzenmacher2020scheduling,
  title={Scheduling with predictions and the price of misprediction},
  author={Mitzenmacher, Michael},
  booktitle={Proceedings of the 11th Innovations in Theoretical Computer Science Conference},
  pages={14:1--14:18},
  year={2020}
}

@inproceedings{dutting2021secretaries,
  title={Secretaries with advice},
  author={D{\"u}tting, Paul and Lattanzi, Silvio and Paes Leme, Renato and Vassilvitskii, Sergei},
  booktitle={Proceedings of the 22nd ACM Conference on Economics and Computation},
  pages={409--429},
  year={2021}
}

@inproceedings{antoniadis2020secretary,
  title={Secretary and online matching problems with machine learned advice},
  author={Antoniadis, Antonios and Gouleakis, Themis and Kleer, Pieter and Kolev, Pavel},
  booktitle={Advances in Neural Information Processing Systems},
  volume={33},
  pages={7933--7944},
  year={2020}
}

@inproceedings{mansour2025paars,
  title={Paars: Persona aligned agentic retail shoppers},
  author={Mansour, Saab and Perelli, Leonardo and Mainetti, Lorenzo and Davidson, George and D’Amato, Stefano},
  booktitle={Proceedings of the 1st Workshop for Research on Agent Language Models (REALM 2025)},
  pages={143--159},
  year={2025}
}

@inproceedings{lin2022learning,
  title={Learning augmented binary search trees},
  author={Lin, Honghao and Luo, Tian and Woodruff, David},
  booktitle={International Conference on Machine Learning},
  pages={13431--13440},
  year={2022},
  organization={PMLR}
}

@article{mccauley2023online,
  title={Online list labeling with predictions},
  author={McCauley, Samuel and Moseley, Ben and Niaparast, Aidin and Singh, Shikha},
  journal={Advances in Neural Information Processing Systems},
  volume={36},
  pages={60278--60290},
  year={2023}
}

@inproceedings{zeynalirobust,
  title={Robust Learning-Augmented Dictionaries},
  author={Zeynali, Ali and Kamali, Shahin and Hajiesmaili, Mohammad},
  booktitle={Forty-first International Conference on Machine Learning}
}

@article{benomar2024learning,
  title={Learning-augmented priority queues},
  author={Benomar, Ziyad and Coester, Christian},
  journal={Advances in Neural Information Processing Systems},
  volume={37},
  pages={124163--124197},
  year={2024}
}

@article{bhaskara2021logarithmic,
  title={Logarithmic regret from sublinear hints},
  author={Bhaskara, Aditya and Cutkosky, Ashok and Kumar, Ravi and Purohit, Manish},
  journal={Advances in Neural Information Processing Systems},
  volume={34},
  pages={28222--28232},
  year={2021}
}

@inproceedings{golowich2022can,
  title={Can Q-learning be improved with advice?},
  author={Golowich, Noah and Moitra, Ankur},
  booktitle={Conference on Learning Theory},
  pages={4548--4619},
  year={2022},
  organization={PMLR}
}

@article{raman2024online,
  title={Online classification with predictions},
  author={Raman, Vinod and Tewari, Ambuj},
  journal={Advances in Neural Information Processing Systems},
  volume={37},
  pages={55884--55914},
  year={2024}
}

@article{shoham2025prediction,
  title={Prediction-Powered Semi-Supervised Learning with Online Power Tuning},
  author={Shoham, Noa and Dorfman, Ron and Shaer, Shalev and Levy, Kfir Y and Romano, Yaniv},
  journal={arXiv preprint arXiv:2510.22586},
  year={2025}
}

@inproceedings{
cowen2026multippi,
title={Multiple-Prediction-Powered Inference},
author={Charlie Cowen-Breen and Alekh Agarwal and Stephen Bates and William W. Cohen and Jacob Eisenstein and Amir Globerson and Adam Fisch},
booktitle={The Fourteenth International Conference on Learning Representations},
year={2026},
url={https://openreview.net/forum?id=gJZ5rf2bS4}
}

@article{ji2025multi,
  title={Multi-armed bandits with machine learning-generated surrogate rewards},
  author={Ji, Wenlong and Pan, Yihan and Zhu, Ruihao and Lei, Lihua},
  journal={arXiv preprint arXiv:2506.16658},
  year={2025}
}

@article{krsteski2025valid,
  title={Valid survey simulations with limited human data: The roles of prompting, fine-tuning, and rectification},
  author={Krsteski, Stefan and Russo, Giuseppe and Chang, Serina and West, Robert and Gligori{\'c}, Kristina},
  journal={arXiv preprint arXiv:2510.11408},
  year={2025}
}

@article{goldberg2001eigentaste,
  title={Eigentaste: A constant time collaborative filtering algorithm},
  author={Goldberg, Ken and Roeder, Theresa and Gupta, Dhruv and Perkins, Chris},
  journal={Information Retrieval},
  volume={4},
  number={2},
  pages={133--151},
  year={2001},
  publisher={Springer}
}

@article{harper2015movielens,
  title={The movielens datasets: History and context},
  author={Harper, F Maxwell and Konstan, Joseph A},
  journal={ACM Transactions on Interactive Intelligent Systems (TiiS)},
  volume={5},
  number={4},
  pages={1--19},
  year={2015},
  publisher={ACM New York, NY, USA}
}

@article{rieder2026simab,
  title={SimAB: Simulating A/B Tests with Persona-Conditioned AI Agents for Rapid Design Evaluation},
  author={Rieder, Tim and Schneider, Marian and Truss, Mario and Tsaplin, Vitaly and Rublea, Alina and Dere, Sinem and Sanz, Francisco Chicharro and Reiss, Tobias and Dogan, Mustafa Doga},
  journal={arXiv preprint arXiv:2603.01024},
  year={2026}
}

@article{chatzi2024prediction,
  title={Prediction-powered ranking of large language models},
  author={Chatzi, Ivi and Straitouri, Eleni and Thejaswi, Suhas and Rodriguez, Manuel G},
  journal={Advances in Neural Information Processing Systems},
  volume={37},
  pages={113096--113133},
  year={2024}
}

@inproceedings{hu2024quantifying,
  title={Quantifying the persona effect in LLM simulations},
  author={Hu, Tiancheng and Collier, Nigel},
  booktitle={Proceedings of the 62nd Annual Meeting of the Association for Computational Linguistics (Volume 1: Long Papers)},
  pages={10289--10307},
  year={2024}
}

@article{ji2026overview,
  title={An overview of large language models for statisticians},
  author={Ji, Wenlong and Yuan, Weizhe and Getzen, Emily and Cho, Kyunghyun and Jordan, Michael I and Mei, Song and Weston, Jason and Su, Weijie J and Xu, Jing and Zhang, Linjun},
  journal={The American Statistician},
  number={just-accepted},
  pages={1--106},
  year={2026},
  publisher={Taylor \& Francis}
}

@inproceedings{calderon2025alternative,
  title={The alternative annotator test for llm-as-a-judge: How to statistically justify replacing human annotators with llms},
  author={Calderon, Nitay and Reichart, Roi and Dror, Rotem},
  booktitle={Proceedings of the 63rd Annual Meeting of the Association for Computational Linguistics (Volume 1: Long Papers)},
  pages={16051--16081},
  year={2025}
}

@article{castelo2026simgym,
  title={SimGym: Traffic-Grounded Browser Agents for Offline A/B Testing in E-Commerce},
  author={Castelo, Alberto and Foumani, Zahra Zanjani and Fan, Ailin and Koay, Keat Yang and Malik, Vibhor and Zhu, Yuanzheng and Li, Han and Feghhi, Meysam and Uliana, Ronie and Xie, Shuang and others},
  journal={arXiv preprint arXiv:2602.01443},
  year={2026}
}

@article{li2025llm,
  title={Llm generated persona is a promise with a catch},
  author={Li, Ang and Chen, Haozhe and Namkoong, Hongseok and Peng, Tianyi},
  journal={arXiv preprint arXiv:2503.16527},
  year={2025}
}

@article{chandrasegaran2025synthetic,
  title={Synthetic users: insights from designers’ interactions with persona-based chatbots},
  author={Chandrasegaran, Senthil and Lloyd, Peter and others},
  journal={AI EDAM},
  volume={39},
  pages={e2},
  year={2025},
  publisher={Cambridge University Press}
}

@inproceedings{ziegler2005improving,
  title={Improving recommendation lists through topic diversification},
  author={Ziegler, Cai-Nicolas and McNee, Sean M and Konstan, Joseph A and Lausen, Georg},
  booktitle={Proceedings of the 14th international conference on World Wide Web},
  pages={22--32},
  year={2005}
}

@INPROCEEDINGS{Bertinmahieux2011,
  author = {Thierry Bertin-Mahieux and Daniel P.W. Ellis and Brian Whitman and Paul Lamere},
  title = {The Million Song Dataset},
  booktitle = {{Proceedings of the 12th International Conference on Music Information
	Retrieval ({ISMIR} 2011)}},
  year = {2011}
}

@article{akaike1974new,
  title={A new look at the statistical model identification},
  author={Akaike, Hirotugu},
  journal={IEEE Transactions on Automatic Control},
  volume={19},
  number={6},
  pages={716--723},
  year={1974}
}

@article{schwarz1978estimating,
  title={Estimating the dimension of a model},
  author={Schwarz, Gideon},
  journal={The Annals of Statistics},
  volume={6},
  number={2},
  pages={461--464},
  year={1978}
}

@inproceedings{li2015toward,
  title={Toward predicting the outcome of an A/B experiment for search relevance},
  author={Li, Lihong and Kim, Jin Young and Zitouni, Imed},
  booktitle={Proceedings of the Eighth ACM International Conference on Web Search and Data Mining},
  pages={37--46},
  year={2015}
}

@inproceedings{sales2018using,
  title={Using big data to sharpen design-based inference in A/B tests},
  author={Sales, Adam and Botelho, AF and Patikorn, Thanaporn and Heffernan, Neil T},
  booktitle={Proceedings of the eleventh international conference on educational data mining},
  year={2018}
}

@article{benomar2024non,
  title={Non-clairvoyant scheduling with partial predictions},
  author={Benomar, Ziyad and Perchet, Vianney},
  journal={arXiv preprint arXiv:2405.01013},
  year={2024}
}

@article{benomar2026non,
  title={Non-Clairvoyant Scheduling with Progress Bars},
  author={Benomar, Ziyad and Cosson, Romain and Lindermayr, Alexander and Schl{\"o}ter, Jens},
  journal={Advances in Neural Information Processing Systems},
  volume={38},
  pages={93520--93561},
  year={2026}
}

@article{benomar2025tradeoffs,
  title={On tradeoffs in learning-augmented algorithms},
  author={Benomar, Ziyad and Perchet, Vianney},
  journal={arXiv preprint arXiv:2501.12770},
  year={2025}
}

\newpage
\appendix
\section{Population level sign prediction}
In this section we provide the proof of Theorem \ref{thm:main-directional-pred}. The first Lemma below gives the expression of the power of the asymmetric Z-test. Then we prove the consistency and robustness bounds in separate lemmas.

\begin{lemma}\label{lem:power-lambda}
For the asymmetric Z-test with critical values at quantiles $(1-\lambda)\alpha/2$ and $1-(1+\lambda)\alpha/2$, the power function is:
$$\text{Power}(\theta, \alpha, \lambda) = \Phi(z_{(1-\lambda)\alpha/2} - \theta) + \Phi(\theta - z_{1-(1+\lambda)\alpha/2})$$
\end{lemma}

\begin{proof}
We test $H_0: \theta = 0$ vs $H_1: \theta \neq 0$ using test statistic $Z \sim \Ncal(\theta, 1)$ under the alternative with effect size $\theta$.
The rejection region is:
$$Z < z_{(1-\lambda)\alpha/2} \quad \text{or} \quad Z > z_{1-(1+\lambda)\alpha/2}\;.$$
 Therefore, since $Z-\theta \sim \Ncal(0,1)$ the power is:
\begin{align*}
\text{Power}(\theta, \alpha, \lambda) 
&= P(Z < z_{(1-\lambda)\alpha/2} \mid \theta) + P(Z > z_{1-(1+\lambda)\alpha/2} \mid \theta)\\
&= P\left(Z-\theta < z_{(1-\lambda)\alpha/2} - \theta\right) + P\left(Z-\theta > z_{1-(1+\lambda)\alpha/2} - \theta\right)\\
&= \Phi(z_{(1-\lambda)\alpha/2} - \theta) + \left(1 - \Phi(z_{1-(1+\lambda)\alpha/2} - \theta)\right)\\
&= \Phi(z_{(1-\lambda)\alpha/2} - \theta) + \Phi(\theta - z_{1-(1+\lambda)\alpha/2})\;,
\end{align*}
where we used for the last inequality the symmetry $\Phi(-u) = 1-\Phi(u)$.
\end{proof}

\subsection{Consistency}

\begin{lemma}\label{lem:f-decreasing}
Denote by $\phi: x \mapsto e^{-x^2/2}/\sqrt{2\pi}$ the PDF of the standard normal distribution $\Ncal(0,1)$, and $\Phi: x \mapsto \int_{-\infty}^x \phi(t) dt$ its CDF. Then
$f = \phi/\Phi$
is decreasing.
\end{lemma}

\begin{proof}
To demonstrate the result, we will show that the derivative of $f$ is negative. It holds for all $x \in \Rbb$ that
\[
f'(x) = \frac{\phi'(x) \Phi(x) - \phi(x)^2}{\Phi(x)^2}\:.
\]
Recalling that  $\phi(x) = e^{-x^2/2}/\sqrt{2\pi}$, we have that $\phi'(x) = -x \phi(x)$, and since $\phi$ and $\Phi$ are both positives, we have for all $x \in \Rbb$ that
\begin{align*}
f'(x) < 0 
&\iff \phi'(x) \Phi(x) < \phi(x)^2\\
&\iff -x \phi(x) \Phi(x) < \phi(x)^2\\
&\iff -x \Phi(x) < \phi(x)\;.
\end{align*}
Therefore, it suffices to prove that $-x \Phi(x) < \phi(x)$ for all $x$.
Let $x \in \Rbb$, we have
\begin{align*}
    -x \Phi(x) 
    &= -x \int_{-\infty}^{x} \frac{e^{-t^2/2}}{\sqrt{2\pi}}  dt = \int_{-\infty}^{x} (-x) \frac{e^{-t^2/2}}{\sqrt{2\pi}} dt\\
    &< \int_{-\infty}^{x} (-t) \frac{e^{-t^2/2}}{\sqrt{2\pi}} dt
    = \frac{1}{\sqrt{2\pi}} \left[ e^{-t^2/2} \right]_{-\infty}^x
    = \phi(x)\;.
\end{align*}
which concludes the proof.
\end{proof}

\begin{lemma}[Consistency]
For any fixed $\lambda \in (0,1)$, we have for all $\theta > 0$ and $\alpha \in (0,1)$ that
$$\frac{\text{Power}(\theta, \alpha, \lambda)}{\text{Power}(\theta, \alpha, 1)} \geq \frac{1+\lambda}{2}$$
\end{lemma}

\begin{proof}
Let $\lambda \in [0,1]$, $\theta \geq 0$, and $\alpha \in (0,1)$. We have by Lemma \ref{lem:power-lambda} that
\begin{align*}
\text{Power}(\theta, \alpha, \lambda) &= \Phi(z_{(1-\lambda)\alpha/2} - \theta) + \Phi(\theta - z_{1-(1+\lambda)\alpha/2})\\
\text{Power}(\theta, \alpha, 1) &= \Phi(\theta - z_{1-\alpha})
\end{align*}
We need to prove that
\[
\Phi(z_{(1-\lambda)\alpha/2} - \theta) + \Phi(\theta - z_{1-(1+\lambda)\alpha/2}) \geq \frac{1+\lambda}{2} \Phi(\theta - z_{1-\alpha})\;.
\]
Since $\Phi > 0$, it suffices to prove that
\begin{equation}\label{eq:consisteny-simplified-one-term}
\Phi(\theta - z_{1-(1+\lambda)\alpha/2}) \geq \frac{1+\lambda}{2} \Phi(\theta - z_{1-\alpha})\;.
\end{equation}

Let us denote by $t = z_{1-\alpha}$ and $s = z_{1-(1+\lambda)\alpha/2}$. We have that $(1+\lambda)\alpha/2 \leq \alpha$ and $u \mapsto z_{1-u}$ is decreasing, hence $s\geq t$.
Now observe that for all $x \in \Rbb$ we have
\begin{align*}
\frac{d}{dx}\left[ \frac{\Phi(x-s)}{\Phi(x-t)}\right] 
&= \frac{\phi(x-s)\Phi(x-t) - \Phi(x-s)\phi(x-t)}{\Phi(x-t)^2}\\
&= \frac{\Phi(x-s)}{\Phi(x-t)} \left( \frac{\phi(x-s)}{\Phi(x-s)} - \frac{\phi(x-t)}{\Phi(x-t)}\right)\\
&\geq 0\;,
\end{align*}
where we used for the last inequality that $\phi/\Phi$ is decreasing (by Lemma \ref{lem:f-decreasing}) and that $x-s \leq x-t$. We deduce that $x \mapsto \frac{\Phi(x-s)}{\Phi(x-t)}$ is non-decreasing, hence for $\theta \geq 0$ we have
\begin{align*}
\frac{\Phi(\theta-s)}{\Phi(\theta-t)}
&\geq \frac{\Phi(-s)}{\Phi(-t)}
= \frac{1-\Phi(s)}{1-\Phi(t)}\\
&= \frac{1-\Phi(z_{1-(1+\lambda)\alpha/2})}{1-\Phi(z_{1-\alpha})}
= \frac{(1+\lambda)\alpha/2)}{\alpha}
= \frac{1+\lambda}{2}
\end{align*}
where we used the symmetry property of the normal CDF $\Phi(-u) = 1 - \Phi(u)$ and the definition of the quantiles $z_u = \Phi^{-1}(u)$.

Therefore, we deduce that Equation \eqref{eq:consisteny-simplified-one-term} is true, which concludes the proof.

\end{proof}

\subsubsection{Robustness}
\begin{lemma}[Robustness]
For any fixed $\lambda \in (0,1)$, we have for all $\theta > 0$ and $\alpha \in (0,1)$ that
\[
\frac{\text{Power}(\theta, \alpha, \lambda)}{\text{Power}(\theta, \alpha, 0)} \geq 1 - \lambda
\]
\end{lemma}

\begin{proof}
The power functions are
\begin{align*}
\text{Power}(\theta, \alpha, 0) &= \Phi(z_{\alpha/2} - \theta) + \Phi(\theta - z_{1-\alpha/2})\\
\text{Power}(\theta, \alpha, \lambda) &= \Phi(z_{(1-\lambda)\alpha/2} - \theta) + \Phi(\theta - z_{1-(1+\lambda)\alpha/2})
\end{align*}
We need to prove that
\begin{equation}\label{eq:robustness-goal}
\forall \theta \leq 0: \quad\text{Power}(\theta, \alpha, \lambda) - (1-\lambda)\text{Power}(\theta, \alpha, 0) \geq 0\;,
\end{equation}
Defining for all $\theta \leq 0$ the functions
\begin{align*}
g(\theta) &= \Phi(z_{(1-\lambda)\alpha/2} - \theta) - (1-\lambda)\Phi(z_{\alpha/2} - \theta)\\
h(\theta) &= \Phi(\theta - z_{1-(1+\lambda)\alpha/2}) - (1-\lambda)\Phi(\theta - z_{1-\alpha/2})\;,
\end{align*}
The objective \eqref{eq:robustness-goal} is equivalent to 
\[
\forall \theta \leq 0: \quad g(\theta) + h(\theta) \geq 0\;.
\]
In the following, we will prove separately that both $g$ and $h$ and non-negative for $\theta \leq 0$.

\paragraph{Analysis of $g(\theta)$}

Let $a = z_{\alpha/2}$, $b = z_{(1-\lambda)\alpha/2}$. We can write $g(\theta) = \Phi(b-\theta) - (1-\lambda)\Phi(a-mu)$, and we have for all $\theta \leq 0$ that
\[
g'(\theta) = -\phi(b-\theta) + (1-\lambda)\phi(a-\theta)
\]
Since $\alpha/2 \geq (1-\lambda)\alpha/2$, then $a \geq b$. Moreover, $\phi: x \mapsto e^{-x^2/2}/\sqrt{2\pi}$ is decreasing on $[0,\infty)$ and we have for $\theta \leq 0$ that $a-\theta \geq b-\theta \geq 0$, hence $\phi(a-\theta) \geq \phi(b-\theta)$, which yields
\begin{align*}
g'(\theta) 
&= -\phi(b-\theta) + (1-\lambda)\phi(a-\theta) \\
&\leq -\phi(b-\theta) + (1-\lambda)\phi(b-\theta) \\
&= - \lambda \phi(b-\theta) \leq 0\;.
\end{align*}
We deduce that $g$ is non-increasing on $(-\infty,0]$, and consequently we have for all $\theta \leq 0$ that
\begin{align*}
g(\theta) 
& \geq g(0) \\
&= \Phi(b) - (1-\lambda)\Phi(a) \\
&= \Phi(z_{(1-\lambda)\alpha/2}) - (1-\lambda)\Phi(z_{\alpha/2})\\
&= (1-\lambda)\alpha/2 - (1-\lambda)\alpha/2 \\
&= 0\;,
\end{align*}
which proves that $g(\theta) \geq 0$ for all $\theta \leq 0$.

\paragraph{Analysis of $h(\theta)$:}

Let $c = z_{1-\alpha/2}$ and $d = z_{1-(1+\lambda)\alpha/2}$, then $h(\theta) = \Phi(\theta - d) - (1-\lambda)\Phi(\theta - c)$, and we have for all $\theta \leq 0$ that
\[
h'(\theta) = \phi(\theta - d) - (1-\lambda)\phi(\theta - c)\;.
\]
Using similar arguments as for the analysis of $g$, observe that since $(1+\lambda)\alpha/2 > \alpha/2$, then $d < c$. 
Moreover, $\phi: x \mapsto e^{-x^2/2}/\sqrt{2\pi}$ is increasing on $(-\infty,0]$ and we have for $\theta \leq 0$ that $\theta - c \leq \theta - d \leq 0$, hence $\phi(\theta-c) \leq \phi(\theta - d)$. The derivative of $h$ can be lower bounded as
\begin{align*}
h'(\theta) 
&= \phi(\theta-d) - (1-\lambda)\phi(\theta-c) \\
&\geq \phi(\theta-c) - (1-\lambda)\phi(\theta - c) \\
&= \lambda \phi(\theta - c) \geq 0\;.
\end{align*}
Therefore, $h$ is non-decreasing on $(-\infty,0]$ and it holds for all $\theta \leq 0$ that
\begin{align*}
h(\theta) 
&\geq \lim_{u \to -\infty} h(u) \\
&= \lim_{u \to -\infty} \left[ \Phi(u-d) - (1-\lambda)\Phi(u-c)  \right] \\
&= 0
\end{align*}
where the last inequality is true because $\lim\limits_{u \to -\infty} \Phi(u-d) = \lim\limits_{u \to -\infty} \Phi(u-c) = 0$. 

All in all, we proved that $g$ and $h$ are both non-negative on $(-\infty,0]$, and consequently $\quad\text{Power}(\theta, \alpha, \lambda) - (1-\lambda)\text{Power}(\theta, \alpha, 0) \geq 0$ for all $\theta \leq 0$, which concludes the result.

\end{proof}
\section{Individual-level predictions - GPPI}
\subsection{Proof of Theorem~\ref{thm:gppi-main}}

We establish three key lemmas, then prove the main result.

\begin{lemma}[Fixed Weight Asymptotics]
\label{lem:gppi-fixed}
For any fixed $\lambda \in \mathbb{R}^d$, define $\hat{\theta}(\lambda) = \bar{Y}_n + \lambda^T(\bar{\phi}_N^U - \bar{\phi}_n^L)$. Then:
\[
\sqrt{n}(\hat{\theta}(\lambda) - \theta^*) \xrightarrow{d} \mathcal{N}(0, V(\lambda))
\]
where $V(\lambda) = \sigma_R^2(\lambda) + r \lambda^T \Sigma_\phi \lambda$ and $\sigma_R^2(\lambda) = \mathrm{Var}(Y - \lambda^T \phi(f(X)))$.
\end{lemma}

\begin{proof}
By linearity of expectation, $\mathbb{E}[\hat{\theta}(\lambda)] = \theta^*$ (unbiasedness), and defining $\mu_\phi = \mathbb{E}[\phi(f(X))]$ we have that
\begin{align*}
\theta(\lambda) - \theta^*
&= \big[(\bar{Y}_n - \lambda^T \bar{\phi}_n^L) - (\theta^* - \lambda^T \mu_\phi)\big] + \lambda^T (\bar{\phi}_N^U - \mu_\phi)\\
&= \frac{1}{n} \sum_{i=1}^n R_i(\lambda) + \frac{\lambda^T}{N} \sum_{i=1}^N (\phi(f(\tilde{X}_j)) - \mu_\phi)\;,
\end{align*}
where
\begin{align*}
R_i(\lambda) &= (Y_i - \theta^*) - \lambda^T(\phi(f(X_i)) - \mu_\phi), \quad i=1,\ldots,n\;.
\end{align*}
where $\mu_\phi = \mathbb{E}[\phi(f(X))]$. Muliplying by $\sqrt{n}$ gives
\begin{equation}\label{eq:CLT_decomp}
\sqrt{n}(\hat{\theta}(\lambda) - \theta^*) = \frac{1}{\sqrt{n}}\sum_{i=1}^n R_i(\lambda) + \sqrt{\frac{n}{N}} \cdot \frac{\lambda^T}{\sqrt{N}}\sum_{j=1}^{N} \phi_j^U\;.
\end{equation}

\paragraph{Term 1:} $R_i(\lambda)$ are i.i.d. with $\mathbb{E}[R_i] = 0$ and 
\begin{align*}
\mathrm{Var}(R_i) &= \mathrm{Var}(Y - \lambda^T\phi(f(X))) \\
&= \sigma_Y^2 + \lambda^T\Sigma_\phi\lambda - 2\lambda^T\sigma_{Y,\phi} =: \sigma_R^2(\lambda) < \infty
\end{align*}
since $\mathbb{E}[Y^2] < \infty$ and $\mathbb{E}[\|\phi(f(X))\|^2] < \infty$. By the Central Limit Theorem 
\[
\frac{1}{\sqrt{n}}\sum_{i=1}^n R_i(\lambda) \xrightarrow{d} \mathcal{N}(0, \sigma_R^2(\lambda))\;.
\]

\paragraph{Term 2:} $\phi_j^U$ are i.i.d. with $\mathbb{E}[\phi_j^U] = 0$ and $\mathrm{Cov}(\phi_j^U) = \Sigma_\phi$. The Multivariate Central Limit Theorem yields
\[
\frac{1}{\sqrt{N}}\sum_{j=1}^{N} \phi_j^U \xrightarrow{d} \mathcal{N}_d(0, \Sigma_\phi)\;,
\]
and we deduce that
\[
\frac{\lambda^T}{\sqrt{N}}\sum_{j=1}^{N} \phi_j^U \xrightarrow{d} \mathcal{N}(0, \lambda^T\Sigma_\phi\lambda)\;.
\]
Since $r_n = n/N \to r$, we have $\sqrt{n/N} \to \sqrt{r}$. By Slutsky's theorem:
\[
\sqrt{\frac{n}{N}} \cdot \frac{\lambda^T}{\sqrt{N}}\sum_{j=1}^{N} \phi_j^U \xrightarrow{d} \mathcal{N}(0, r \lambda^T\Sigma_\phi\lambda)
\]

Since labeled and unlabeled data are independent, the two sums in \eqref{eq:CLT_decomp} are independent. Therefore
\[
\sqrt{n}(\hat{\theta}(\lambda) - \theta^*) \xrightarrow{d} \mathcal{N}(0, \sigma_R^2(\lambda) + r\lambda^T\Sigma_\phi\lambda) = \mathcal{N}(0, V(\lambda))\;.
\]
\end{proof}

\begin{lemma}[Optimal Weight]
\label{lem:gppi-optimal}
$V(\lambda)$ is minimized at $\lambda^* = \frac{1}{1+r}\Sigma_\phi^{-1}\sigma_{Y,\phi}$ with minimum value:
\[
V^* = \sigma_Y^2 - \frac{1}{1+r}\sigma_{Y,\phi}^T\Sigma_\phi^{-1}\sigma_{Y,\phi}
\]
\end{lemma}

\begin{proof}
Expanding $\sigma_R^2(\lambda)$:
\[
\sigma_R^2(\lambda) = \sigma_Y^2 + \lambda^T\Sigma_\phi\lambda - 2\lambda^T\sigma_{Y,\phi}
\]
Therefore
\[
V(\lambda) = \sigma_Y^2 - 2\lambda^T\sigma_{Y,\phi} + (1+r)\lambda^T\Sigma_\phi\lambda\;,
\]
and
\[
\nabla_\lambda V = -2\sigma_{Y,\phi} + 2(1+r)\Sigma_\phi\lambda\;.
\]
The Hessian is $\nabla^2_\lambda V = 2(1+r)\Sigma_\phi \succ 0$ since $\Sigma_\phi \succ 0$, hence $V(\lambda)$ is convex, and it is minimal at the point where the gradient is zero, i.e. at
\[
\lambda^* = \frac{1}{1+r}\Sigma_\phi^{-1}\sigma_{Y,\phi}\;.
\]

Substituting $\lambda^*$ gives
\begin{align*}
V^* &= \sigma_Y^2 - 2(\lambda^*)^T\sigma_{Y,\phi} + (1+r)(\lambda^*)^T\Sigma_\phi\lambda^* \\
&= \sigma_Y^2 - \frac{2}{1+r}\sigma_{Y,\phi}^T\Sigma_\phi^{-1}\sigma_{Y,\phi} + \frac{1}{1+r}\sigma_{Y,\phi}^T\Sigma_\phi^{-1}\sigma_{Y,\phi} \\
&= \sigma_Y^2 - \frac{1}{1+r}\sigma_{Y,\phi}^T\Sigma_\phi^{-1}\sigma_{Y,\phi}\;.
\end{align*}
\end{proof}

\begin{lemma}[Consistency of Plug-in Estimators]
\label{lem:gppi-consistency}
Under the conditions of Theorem~\ref{thm:gppi-main}:
\[
\hat{\Sigma}_\phi \xrightarrow{p} \Sigma_\phi, \quad \hat{\sigma}_{Y,\phi} \xrightarrow{p} \sigma_{Y,\phi}, \quad \hat{\lambda} \xrightarrow{p} \lambda^*\;.
\]
\end{lemma}

\begin{proof}
By the Strong Law of Large Numbers, since $\mathbb{E}[\|\phi(f(X))\|^2] < \infty$ and $\mathbb{E}[Y^2] < \infty$, we have
\[
\hat{\Sigma}_\phi \xrightarrow{p} \Sigma_\phi, \quad \hat{\sigma}_{Y,\phi} \xrightarrow{p} \sigma_{Y,\phi}\;.
\]
Since $\Sigma_\phi \succ 0$, matrix inversion is continuous at $\Sigma_\phi$. By the continuous mapping theorem
\[
\hat{\Sigma}_\phi^{-1} \xrightarrow{p} \Sigma_\phi^{-1}\;.
\]
Since $r_n \to r$ and $1/(1+r_n) \to 1/(1+r)$, we deduce by Slutsky's theorem that
\[
\hat{\lambda} = \frac{1}{1+r_n}\hat{\Sigma}_\phi^{-1}\hat{\sigma}_{Y,\phi} \xrightarrow{p} \frac{1}{1+r}\Sigma_\phi^{-1}\sigma_{Y,\phi} = \lambda^*\;.
\]
\end{proof}

\subsubsection{Proof of Main Theorem}

\begin{proof}[Proof of Theorem~\ref{thm:gppi-main}]
We prove below the the three statements of the theorem.

\paragraph{(i): Asymptotic normality.}
Decompose
\[
\sqrt{n}(\hat{\theta}_{\text{GPPI}} - \theta^*) = \sqrt{n}(\hat{\theta}(\lambda^*) - \theta^*) + \sqrt{n}(\hat{\theta}(\hat{\lambda}) - \hat{\theta}(\lambda^*))\;.
\]
We have by Lemma~\ref{lem:gppi-fixed} that $\sqrt{n}(\hat{\theta}(\lambda^*) - \theta^*) \xrightarrow{d} \mathcal{N}(0, V^*)$.
For the second term, we can write
\[
\sqrt{n}(\hat{\theta}(\hat{\lambda}) - \hat{\theta}(\lambda^*)) = \sqrt{n}(\hat{\lambda} - \lambda^*)^T(\bar{\phi}_N^U - \bar{\phi}_n^L)\;,
\]
and we have by the CLT that $\bar{\phi}_N^U - \mu_\phi = O_p(N^{-1/2})$ and $\bar{\phi}_n^L - \mu_\phi = O_p(n^{-1/2})$. Since $n/N \to r \in (0,\infty)$, we have $N = \Theta(n)$, hence
\[
\bar{\phi}_N^U - \bar{\phi}_n^L = O_p(n^{-1/2})\;.
\]
Furthermore, Lemma~\ref{lem:gppi-consistency} gives $\hat{\lambda} - \lambda^* = o_p(1)$. Therefore:
\[
\sqrt{n}(\hat{\theta}(\hat{\lambda}) - \hat{\theta}(\lambda^*)) = \sqrt{n} \cdot o_p(1) \cdot O_p(n^{-1/2}) = o_p(1)\;.
\]
Finally, we deduce by Slutsky's theorem that
\[
\sqrt{n}(\hat{\theta}_{\text{GPPI}} - \theta^*) \xrightarrow{d} \mathcal{N}(0, V^*)
\]

\paragraph{(ii): Variance estimation.}
Define 
\[
\tilde{R}_i = (Y_i - \theta^*) - (\lambda^*)^T(\phi(f(X_i)) - \mu_\phi)\;.
\]
we have that $\mathrm{Var}(\tilde{R}_i) = \sigma_R^2(\lambda^*)$. The sample residuals satisfy
\begin{align*}
\hat{R}_i 
&= (Y_i - \bar{Y}_n) - \hat{\lambda}^T(\phi(f(X_i)) - \bar{\phi}_n^L) \\
&= \tilde{R}_i + (\lambda^* - \hat{\lambda})^T(\phi(f(X_i)) - \mu_\phi) + O_p(n^{-1/2})\;.
\end{align*}
By standard results on sample variance (see for example \citep{van2000asymptotic}) and Lemma~\ref{lem:gppi-consistency}:
\[
\hat{\sigma}_R^2 = \frac{1}{n-1}\sum_{i=1}^n \hat{R}_i^2 \xrightarrow{p} \sigma_R^2(\lambda^*)\;.
\]
Using again Lemma~\ref{lem:gppi-consistency} and the continuity of $(\lambda, \Sigma) \mapsto (\lambda^T \Sigma \lambda$, we obtain
\[
\hat{\lambda}^T\hat{\Sigma}_\phi\hat{\lambda} \xrightarrow{p} (\lambda^*)^T\Sigma_\phi\lambda^*
\]
Since $r_n = n/N \to r$:
\begin{align*}
\hat{V} 
&= \hat{\sigma}_R^2 + r_n\hat{\lambda}^T\hat{\Sigma}_\phi\hat{\lambda} \\
&\xrightarrow{p} \sigma_R^2(\lambda^*) + r(\lambda^*)^T\Sigma_\phi\lambda^* = V(\lambda^*) = V^*    
\end{align*}

\paragraph{(iii): Confidence interval validity.}
From claims (i)--(ii) and Slutsky's theorem, we deduce that
\[
\frac{\sqrt{n}(\hat{\theta}_{\text{GPPI}} - \theta^*)}{\sqrt{\hat{V}}} \xrightarrow{d} \mathcal{N}(0,1)\;.
\]
Therefore
\[
\lim_{n,N\to\infty} \mathbb{P}\left(\left|\frac{\sqrt{n}(\hat{\theta}_{\text{GPPI}} - \theta^*)}{\sqrt{\hat{V}}}\right| \leq z_{\alpha/2}\right) = 1-\alpha
\]
which is equivalent to $\lim_{n,N\to\infty} \mathbb{P}(\theta^* \in \text{CI}_{1-\alpha}) = 1-\alpha$.
\end{proof}

\subsection{Efficiency gains}\label{app:efficiency-gain}

Here we quantify how much the confidence interval width can be reduced using GPPI compared to PPI++.

\begin{lemma}
\label{lem:variance-decomposition}
Under the conditions of Theorem~\ref{thm:gppi-main}, the following hold:
\begin{enumerate}[label=(\roman*), leftmargin=*]
\item The asymptotic variance admits the decomposition
\begin{equation}
V^* = \sigma_Y^2\left(1 - \frac{R^2_\phi}{1+r}\right),
\end{equation}
where $R^2_\phi = \sigma_{Y,\phi}^T \Sigma_\phi^{-1} \sigma_{Y,\phi} / \sigma_Y^2 \in [0,1]$ is the population coefficient of determination from regressing $Y$ on $\phi(f(X))$.

\item For PPI++, with $\phi(z) = z$, we have $R^2_\phi = \rho_{Y,f}^2$, where $\rho_{Y,f} = \mathrm{Cor}(Y, f(X))$.

\item The asymptotic confidence interval half-widths satisfy
\begin{equation}
W_{\text{GPPI}} = W_{\text{classical}}\sqrt{1 - \frac{R^2_\phi}{1+r}} \leq W_{\text{PPI}} = W_{\text{classical}}\sqrt{1 - \frac{\rho_{Y,f}^2}{1+r}} \leq W_{\text{classical}},
\end{equation}
where $W_{\text{classical}} = z_{\alpha/2}\sqrt{\sigma_Y^2/n}$.

\item The width ratio is given by
\begin{equation}
\frac{W_{\text{GPPI}}}{W_{\text{PPI}}} = \sqrt{\frac{1 - R^2_\phi/(1+r)}{1 - \rho_{Y,f}^2/(1+r)}}.
\end{equation}
Strict inequality $W_{\text{GPPI}} < W_{\text{PPI}}$ holds if and only if $R^2_\phi > \rho_{Y,f}^2$, which occurs when nonlinear transformations capture additional structure beyond linear correlation.
\end{enumerate}
\end{lemma}

\begin{proof} We prove below the three caims of the lemma.
\paragraph{Claim (i):} From Theorem~\ref{thm:gppi-main}(i), the asymptotic variance is
\begin{equation}
V^* = \sigma_Y^2 - \frac{1}{1+r} \sigma_{Y,\phi}^T \Sigma_\phi^{-1} \sigma_{Y,\phi} = \sigma_Y^2 \left(1 - \frac{1}{1+r} \cdot \frac{\sigma_{Y,\phi}^T \Sigma_\phi^{-1} \sigma_{Y,\phi}}{\sigma_Y^2}\right).
\end{equation}
Define $R^2_\phi = \sigma_{Y,\phi}^T \Sigma_\phi^{-1} \sigma_{Y,\phi} / \sigma_Y^2$. This is the population $R^2$ from the best linear predictor of $Y$ given $\phi(f(X))$. By the Cauchy-Schwarz inequality, $R^2_\phi \in [0,1]$.

\paragraph{Claim (ii):} When $\phi(z) = z$, we have $\Sigma_\phi = \mathrm{Var}(f(X)) = \sigma_f^2$ and $\sigma_{Y,\phi} = \mathrm{Cov}(Y, f(X))$. Thus
\begin{equation}
R^2_\phi = \frac{\mathrm{Cov}(Y, f(X))^2}{\sigma_Y^2 \cdot \sigma_f^2} = \mathrm{Cor}(Y, f(X))^2 = \rho_{Y,f}^2.
\end{equation}

\paragraph{Claim (iii):} The asymptotic $(1-\alpha)$ confidence interval has half-width $W = z_{\alpha/2}\sqrt{V^*/n}$. For the classical estimator, $V_{\text{classical}} = \sigma_Y^2$, giving $W_{\text{classical}} = z_{\alpha/2}\sqrt{\sigma_Y^2/n}$. Substituting the variance decomposition from Claim (i):
\begin{equation}
W_{\text{GPPI}} = z_{\alpha/2}\sqrt{\sigma_Y^2(1 - R^2_\phi/(1+r))/n} = W_{\text{classical}}\sqrt{1 - \frac{R^2_\phi}{1+r}}.
\end{equation}
Similarly, $W_{\text{PPI}} = W_{\text{classical}}\sqrt{1 - \rho_{Y,f}^2/(1+r)}$ by Claim (ii).

Since $\phi(z) = z$ is a special case of the general transformation (taking $d=1$ and including only the linear term), and $R^2_\phi$ measures variance explained by a $d$-dimensional transformation while $\rho_{Y,f}^2$ measures variance explained by a 1-dimensional predictor, we have $R^2_\phi \geq \rho_{Y,f}^2$. Therefore $W_{\text{GPPI}} \leq W_{\text{PPI}} \leq W_{\text{classical}}$.

\textbf{Claim (iv):} Taking the ratio of the expressions in Claim (iii) yields the stated formula.
\end{proof}
\subsection{Matrix concentration inequalitites}

The two Lemmas we prove below are immediate corollaries of the Bernstein inequalities for symmetrical then rectangular matrices (Theorems 1.6.2 and 7.3.1 in \citep{tropp2015introduction}).

\begin{lemma}\label{lem:bernstein}
Let $Z_1, \ldots, Z_n$ be independent random vectors in $\Rbb^d$, satisfying for all $i \in [n]$ that $\|Z_i\|^2 \leq M$ and having the same covariance matrix $\Sigma$. Let $\hat{\Sigma} = \frac{1}{n} \sum_{i=1}^n Z_i Z_i^T$, then for all $t\geq 0$
\[
\Pbb \left\{ \|\hat{\Sigma} - \Sigma\|_{\text{op}} \geq t\right\} \leq 2d \exp\left( \frac{-nt^2}{8M(\|\Sigma\|_{\text{op}} + 2t/3)} \right)\;.
\]
In particular, if $d=\omega(1)$, $M = \Omega(1)$ and $M^2 \log d = o(n)$ then
\[
\|\hat{\Sigma} - \Sigma\|_{\text{op}} = O_p\left( \sqrt{\frac{M^2 \log d}{n}} \right)\;.
\]
\end{lemma}

\begin{proof}
By the matrix Bernstein inequality \citep{tropp2015introduction}, for independent centered symmetric matrices $Q_1, \ldots, Q_n$ with $\|Q_i\|_{\text{op}} \leq L$ almost surely:
\[
\mathbb{P}\left\{\left\|\sum_{i=1}^n Q_i\right\|_{\text{op}} \geq t\right\} \leq 2d \exp\left(-\frac{t^2/2}{\sigma^2 + Lt/3}\right),
\]
where $\sigma^2 = \left\|\sum_{i=1}^n \Ebb[Q_i^2]\right\|_{\text{op}}$.

Define for all $i \in [n]$ the matrix $Q_i =  \frac{1}{n}(\tilde{Z}_i\tilde{Z}_i^T - \Sigma)$. The matrices $Q_i$ are centered, and symmetric. To use Bernstein inequality, we will prove first the bound $L$ on the operator norm almost surely, then the bound on $\sigma^2$.  

Using that $\|\tilde{Z}_i\|^2 \leq M$ a.s., by definition of the operator norm and Cauchy Schwarz inequality we obtain
\[
\|\tilde{Z}_i\tilde{Z}_i^T\|_{\text{op}}
= \sup_{\|u\|=1} u^T \tilde{Z}_i\tilde{Z}_i^Tu
= \sup_{\|u\|=1} (\tilde{Z}_i^T u)^2
\leq \sup_{\|u\|=1} (\|\tilde{Z}_i\| \|u\|)^2
\leq M \quad \text{a.s.}\;,
\]
and by Jensen's inequality $\|\Sigma\|_{\text{op}} = \|\Ebb[\tilde{Z}_i\tilde{Z}_i^T]\|_{\text{op}} \leq \Ebb \|\tilde{Z}_i\tilde{Z}_i^T\|_{\text{op}} \leq M$. Therefore $\|Q_i\|_{\text{op}} \leq \frac{1}{n}\|\tilde{Z}_i \tilde{Z}_i^T\|_{\text{op}} + \frac{1}{n}\|\Sigma\|_{\text{op}} \leq 2M/n$ a.s.

On the other hand, $\left\|\sum_{i=1}^n \Ebb[Q_i^2]\right\|_{\text{op}} \leq \sum_{i=1}^n \left\|\Ebb[Q_i^2]\right\|_{\text{op}} = \frac{1}{n^2} \sum_{i=1}^n \left\|\Ebb[(\tilde{Z}_i\tilde{Z}_i^T - \Sigma)^2]\right\|_{\text{op}}$, and for all $i \in [n]$ we have
\begin{align*}
\Ebb[(\tilde{Z}_i\tilde{Z}_i^T - \Sigma)^2]
&= \Ebb[(\tilde{Z}_i\tilde{Z}_i^T)^2] - \Sigma^2 \\
&\preceq \Ebb[(\tilde{Z}_i\tilde{Z}_i^T)^2]
= \Ebb[\tilde{Z}_i (\tilde{Z}_i^T \tilde{Z}_i) \tilde{Z}_i^T]
= \Ebb[\|\tilde{Z}_i\|^2  \tilde{Z}_i \tilde{Z}_i^T] \\
&\preceq M \Ebb[\tilde{Z}_i \tilde{Z}_i^T]
= M \Sigma\;.
\end{align*}
Therefore $\sigma^2 = \left\|\sum_{i=1}^n \Ebb[Q_i^2]\right\|_{\text{op}} \leq M  \|\Sigma\|_{\text{op}}/n$

Therefore, using Bernstein's inequality with $L = 2M/n$ and $\sigma^2 \leq M  \|\Sigma\|_{\text{op}}/n$ gives
\[
\Pbb \left\{ \|\hat{\Sigma} - \Sigma\|_{\text{op}} \geq t\right\} \leq 2d \exp\left( \frac{-nt^2}{2M(\|\Sigma\|_{\text{op}} + 2t/3)} \right)\;.
\]
which proves the concentration inequality for centered random variables.

Now in the general case, if we do not assume that $(Z_i)_i$ are centered, then we define for all $i$ the centered random variable $\tilde{Z}_i = Z_i - \Ebb[Z_i]$. With the triangle then Jensen's inequality we have that $\|\tilde{Z}_i\| \leq 2 \sqrt{M}$, and the previous result holds by replacing $M$ with $4 M$:
\[
\Pbb \left\{ \|\hat{\Sigma} - \Sigma\|_{\text{op}} \geq t\right\} \leq 2d \exp\left( \frac{-nt^2}{8M(\|\Sigma\|_{\text{op}} + 2t/3)} \right)\;,
\]
which proves the claimed concentration bound.

Using that $\|\Sigma\|_{\text{op}} \leq M$ as we proved earlier, we also have that 
\[
\Pbb \left\{ \|\hat{\Sigma} - \Sigma\|_{\text{op}} \geq t\right\} \leq 2d \exp\left( \frac{-nt^2}{2M(M + 2t/3)} \right)\;.
\]
Now assuming that $M = \Omega(1)$ and $M^2 \log d = o(n)$, taking $t = \frac{4 M \sqrt{\log d}}{\sqrt{n}}$ gives
\begin{align*}
\Pbb \left\{ \|\hat{\Sigma} - \Sigma\|_{\text{op}} \geq \frac{4 M \sqrt{\log d}}{\sqrt{n}}\right\} 
&\leq 2d \exp\left( \frac{- 16M^2 \log d}{8M(M + o(1))} \right)\\
&= 2d \exp\left( - (2-o(1)) \log d \right)
= 2/d^{1-o(1)} = o(1)\;,
\end{align*}
hence 
\[
\|\hat{\Sigma} - \Sigma\|_{\text{op}} = O_p\left( \sqrt{\frac{M^2 \log d}{n}} \right)\;.
\]
\end{proof}

\begin{lemma}\label{lem:conc-covariance-vector}
Let $Z_1,\dots,Z_n$ be independent random vectors in $\mathbb{R}^d$ such that $\|Z_i\|^2 \leq M$ a.s. and $Y_1,\ldots,Y_n$ be real-valued random variables such that $Y_i^2 \leq K_Y$ a.s. Then for $t \geq 11 \sqrt{\frac{K_Y M}{n}}$,
\[
\mathbb{P}\{\|\widehat{\text{Cov}}(Y,Z) - \text{Cov}(Y,Z)\| \geq t\} \leq 4(d+1) \exp\left(\frac{-n t^2}{16 (8 K_Y M + \sqrt{K_Y M}t/3)}\right)\;.
\]
\end{lemma}

\begin{proof}
Define the centered variables $\tilde{Y}_i = Y_i - \mathbb{E}[Y_1]$ and  $\tilde{Z}_i = Z_i - \mathbb{E}[Z_1]$, then:
\[
\widehat{\text{Cov}}(Y,Z) = \frac{1}{n}\sum_{i=1}^n \tilde{Y}_i \tilde{Z}_i, \quad \text{Cov}(Y,Z) = \mathbb{E}[\tilde{Y}_1 \tilde{Z}_1]
\]

Define $S_i = \frac{1}{n}(\tilde{Y}_i \tilde{Z}_i - \mathbb{E}[\tilde{Y}_1 \tilde{Z}_1]) \in \mathbb{R}^{d \times 1}$ for $i = 1,\ldots,n$. Then:
\[
\widehat{\text{Cov}}(Y,Z) - \text{Cov}(Y,Z) = \sum_{i=1}^n S_i
\]
We will use for our proof the Intrinsic Matrix Bernstein inequality (Theorem 7.3.1 in \citep{tropp2015introduction}). Let us first verify that $(S_i)_i$ satisfy the assumptions of the theorem.

\textbf{1. Zero mean}
$\mathbb{E}[S_i] = \frac{1}{n}(\mathbb{E}[\tilde{Y}_i \tilde{Z}_i] - \mathbb{E}[\tilde{Y}_1 \tilde{Z}_1]) = 0$.

\textbf{2. Uniform bound}
We need to bound $\|\tilde{Y}_i \tilde{Z}_i\|_2 = |\tilde{Y}_i|\|\tilde{Z}_i\|_2$.

By triangle inequality then Jensen inequality we have:
$|\tilde{Y}_i| = |Y_i - \mathbb{E}[Y_1]| \leq |Y_i| + \mathbb{E}[|Y_1|] \leq 2\sqrt{K_Y}$ a.s., and similarly $\|\tilde{Z}_i\|_2 \leq  2\sqrt{M}$ a.s. Therefore:
\[
\|\tilde{Y}_i \tilde{Z}_i\| \leq 2\sqrt{K_Y} \cdot 2\sqrt{M} = 4\sqrt{K_Y M} \quad \text{a.s.}
\]
And by Jensen's inequality $\|\mathbb{E}[\tilde{Y}_1 \tilde{Z}_1]\| \leq 4\sqrt{K_Y M}$.

Thus:
\[
\|S_i\|_{\text{op}} \leq \frac{1}{n}(4\sqrt{K_Y M} + 4\sqrt{K_Y M}) = \frac{8\sqrt{K_Y M}}{n} := L
\]

\textbf{3. Matrix-valued variances}
We need to prove semidefinite upper bounds on $\mathbb{E}[S_i S_i^T]$ and $\mathbb{E}[S_i^T S_i]$. We have immediately that
\[
\sum_{i=1}^n \mathbb{E}[S_i S_i^T] \preceq \sum_{i=1}^n \mathbb{E}[\|S_i\|_2^2] I_d \preceq \frac{64 K_Y M}{n} I_d
\]
and $S_i^T S_i \in \Rbb$ and satisfies
\[
\sum_{i=1}^n \mathbb{E}[S_i^T S_i] = \sum_{i=1}^n \mathbb{E}[\|S_i\|_2^2] \leq \frac{64 K_Y M}{n}
\]

\textbf{4. Apply Intrinsic Matrix Bernstein}
Therefore, using the intrinsic Matrix Bernstein inequality (Theorem 7.3.1 in \citep{tropp2015introduction}) with $v = \frac{64 K_Y M}{n}$ and $L = \frac{8 \sqrt{K_Y M}}{n}$, we obtain for all $t \geq \frac{8 \sqrt{K_Y M}}{\sqrt{n}} + \frac{8 \sqrt{K_Y M}}{3 n}$:
\[
\mathbb{P}\{\|\widehat{\text{Cov}}(Y,Z) - \text{Cov}(Y,Z)\|_2 \geq t\} \leq 4(d+1) \exp\left(\frac{-t^2/2}{v + Lt/3}\right)
\]
It follows that for all $t \geq 11 \sqrt{\frac{K_Y M}{n}}$:
\[
\mathbb{P}\{\|\widehat{\text{Cov}}(Y,Z) - \text{Cov}(Y,Z)\|_2 \geq t\} \leq 4(d+1) \exp\left(\frac{-n t^2}{16 (8 K_Y M + \sqrt{K_Y M}t/3)}\right)
\]
\end{proof}
\subsection{Adaptive Transformation Selection}

This section provides the proof of Theorem~\ref{thm:basis-selection}, which establishes that empirical variance minimization asymptotically selects the optimal transformation.

\subsubsection{Notation and Setup}
We remind here the notation and assumptions used in the proof.
For each transformation $\phi^j: \mathbb{R} \to \mathbb{R}^{d^j}$ with $j \in [m]$, define:
\begin{itemize}[leftmargin=*]
\item Population quantities: $\Sigma_j = \mathrm{Cov}(\phi^j(f(X)))$, $\sigma_{Y,j} = \mathrm{Cov}(Y, \phi^j(f(X)))$
\item Optimal weight: $\lambda^*_j = \frac{1}{1+r} \Sigma_j^{-1} \sigma_{Y,j}$
\item Asymptotic variance: $V_j = \sigma_Y^2 - \frac{1}{1+r} \sigma_{Y,j}^T \Sigma_j^{-1} \sigma_{Y,j}$
\item Sample estimators: $\hat{\Sigma}_j$, $\hat{\sigma}_{Y,j}$, $\hat{\lambda}_j$, $\hat{V}_j$ as defined in Section~\ref{sec:basis-selection}
\end{itemize}

Let $j^* = \arg\min_{j \in [m]} V_j$ denote the oracle-optimal transformation and $\hat{j} = \arg\min_{j \in [m]} \hat{V}_j$ the empirically selected transformation.

\begin{assumption}
\label{assump:basis-selection}
For all $j \in [m]$:
\begin{enumerate}[label=(\roman*), leftmargin=*]
\item $\|\phi^j(f(X))\|^2 \leq  M $ almost surely, where $ M $ is a uniform bound
\item $Y^2 \leq K_Y$ almost surely
\item $\Sigma_j = \mathrm{Cov}(\phi^j(f(X))) \succ 0$ (positive definite)
\item $d^j \leq d$ for all $j$, where $d$ is a fixed constant
\item $d \log(m) = o(n)$
\end{enumerate}
\end{assumption}

\subsubsection{Preliminary lemmas}

\begin{lemma}[Uniform Covariance Concentration]
\label{lem:uniform-cov-conc}
Under Assumption~\ref{assump:basis-selection}, for any $\delta \in (0,1)$, we have for $N$ sufficiently large that
\[
\Pbb \left\{ \max_{j \in [m]}\|\hat{\Sigma}_j - \Sigma_j\|_{\text{op}} \geq 3\sqrt{\frac{ M^2 \log(2dm/\delta)}{N}} \right\} 
\leq \delta
\]
\end{lemma}

\begin{proof}
Let $j \in [M]$, the random vectors $(\phi^j(f(\tilde{X}_i)))_i$ are independent and have their norm bounded by $M$ almost surely. Therefore, by union-bound and Lemma \ref{lem:bernstein} we have for all $t\geq$ that
\[
\Pbb \left\{ \max_{j \in [m]}\|\hat{\Sigma}_j - \Sigma_j\|_{\text{op}} \geq t\right\} 
\leq \sum_{j=1}^m \Pbb \left\{ \|\hat{\Sigma}_j - \Sigma_j\|_{\text{op}} \geq t\right\} \leq 2d m \exp\left( \frac{-Nt^2}{8M(M + 2t/3)} \right)\;.
\]
Let $\delta > 0$, and $t = 3\sqrt{\frac{ M^2 \log(2dm/\delta)}{N}}$. For $N$ sufficiently large we have $\sqrt{\frac{4 \log(2dm/\delta)}{N}} \leq 1$ and
\begin{align*}
\Pbb \left\{ \max_{j \in [m]}\|\hat{\Sigma}_j - \Sigma_j\|_{\text{op}} \geq  3\sqrt{\frac{ M^2 \log(2dm/\delta)}{N}} \right\} 
&\leq 2d m \exp\left( \frac{-Nt^2}{8M(M + 2t/3)} \right) \\
&= 2d m \exp\left( \frac{-9 \log(2dm/\delta)}{8\left(1 + \sqrt{\frac{4 \log(2dm/\delta)}{N}}\right)} \right) \\
&\leq 2d m \exp\left( - \log(2dm/\delta) \right) = \delta
\end{align*}

\end{proof}

\begin{lemma}[Uniform Cross-Covariance Concentration]
\label{lem:uniform-cross-cov-conc}
Under Assumption~\ref{assump:basis-selection}, for any $\delta \in (0,1)$, we have for $N$ sufficiently large that
\[
\Pbb \left\{ \max_{j \in [m]}\|\hat{\sigma}_{Y,j} - \sigma_{Y,j}\|_2 \geq 3\sqrt{\frac{2 K_Y M \log(2md/\delta)}{n}} \right\} 
\leq \delta \;.
\]
\end{lemma}

\begin{proof}
For each $j \in [m]$, we apply Lemma~\ref{lem:conc-covariance-vector} with $Z_i = \phi^j(f(X_i))$ and $Y_i$. By Assumption~\ref{assump:basis-selection}(i) and (ii), we have $\|\phi^j(f(X_i))\|^2 \leq M$ a.s. and $Y_i^2 \leq K_Y$ a.s.

By Lemma~\ref{lem:conc-covariance-vector}, for $t \geq 11\sqrt{\frac{K_Y M}{n}}$:
\[
\mathbb{P}\{\|\hat{\sigma}_{Y,j} - \sigma_{Y,j}\|_2 \geq t\} \leq 4(d^j+1) \exp\left(\frac{-n t^2}{16 (8 K_Y M + \sqrt{K_Y M}t/3)}\right)
\]

Set $t = C\sqrt{\frac{K_Y M \log(2md/\delta)}{n}}$ with $C \geq 11$. Assuming that $d\geq 2$ or $m \geq 2$ or $\delta \leq 1/2$ we have $\log(2md/\delta) \geq 1$, hence $t \geq 11\sqrt{\frac{K_Y M}{n}}$. For $N$ sufficiently large we have $\frac{\sqrt{K_Y M}}{3}t \leq 2 K_Y M$, we have:
\[
\frac{nt^2}{16(8 K_Y M + \sqrt{K_Y M}t/3)} \geq \frac{C^2 \log(2md/\delta)}{160}\;,
\]
and taking $C^2 \geq 320$ (i.e. $C \geq 18$) ensures:
\begin{align*}
\mathbb{P}\{\|\hat{\sigma}_{Y,j} - \sigma_{Y,j}\|_2 \geq t\} 
&\leq 4(d+1) \exp(-2\log(2md/\delta)) \\
&= \frac{4(d+1)}{(2md/\delta)^{2}}\;.
\end{align*}
By the union bound over all $j \in [m]$:
\begin{align*}
\mathbb{P}\left\{\max_{j \in [m]}\|\hat{\sigma}_{Y,j} - \sigma_{Y,j}\|_2 \geq t\right\} &\leq \frac{4 m(d+1)}{(2md/\delta)^{2}} 
\leq \frac{8 md}{(2md/\delta)^{2}} \\
&= \frac{2 \delta^2}{md}
= \delta\;.
\end{align*}
where the last inequality is true when $md \geq 2\delta$, for example if $d\geq 2$ or $m\geq 2$. This completes the proof.
\end{proof}

\begin{lemma}[Uniform Weight Estimation Error]
\label{lem:uniform-weight-error}
Under Assumption~\ref{assump:basis-selection}, for any $\delta \in (0,1)$ and $n$ sufficiently large:
\[
\Pbb \left\{ \max_{j \in [m]} \|\hat{\lambda}_j - \lambda^*_j\| \geq C\sqrt{\frac{K_Y M^3 \log(md/\delta)}{n}} \right\} \leq \delta
\]
for some universal constant $C > 0$.
\end{lemma}

\begin{proof}
Recall that $\lambda^T_j = \frac{1}{1+r} \Sigma_j^{-1} \sigma_{Y,j}$ and $\hat{\lambda}_j = \frac{1}{1+r} \hat{\Sigma}_j^{-1} \hat{\sigma}_{Y,j}$. We decompose the error as:
\begin{align*}
\hat{\lambda}_j - \lambda^T_j &= \frac{1}{1+r}\left(\hat{\Sigma}_j^{-1} \hat{\sigma}_{Y,j} - \Sigma_j^{-1} \sigma_{Y,j}\right) \\
&= \frac{1}{1+r}\left(\hat{\Sigma}_j^{-1} \hat{\sigma}_{Y,j} - \hat{\Sigma}_j^{-1} \sigma_{Y,j} + \hat{\Sigma}_j^{-1} \sigma_{Y,j} - \Sigma_j^{-1} \sigma_{Y,j}\right) \\
&= \frac{1}{1+r}\left(\hat{\Sigma}_j^{-1} (\hat{\sigma}_{Y,j} - \sigma_{Y,j}) + (\hat{\Sigma}_j^{-1} - \Sigma_j^{-1}) \sigma_{Y,j}\right)
\end{align*}

By the triangle then Cauchy Schwarz inequalities:
\[
\|\hat{\lambda}_j - \lambda^T_j\| \leq \frac{1}{1+r}\left(\|\hat{\Sigma}_j^{-1}\| \cdot \|\hat{\sigma}_{Y,j} - \sigma_{Y,j}\| + \|\hat{\Sigma}_j^{-1} - \Sigma_j^{-1}\| \cdot \|\sigma_{Y,j}\|\right)
\]

\noindent
\textbf{Bound on $\|\sigma_{Y,j}\|$.} Jensen's inequality gives
\[
\|\sigma_{Y,j}\| = \|\mathbb{E}[Y \phi^j(f(X))]\| \leq \mathbb{E}[|Y| \cdot \|\phi^j(f(X))\|] \leq \sqrt{K_Y M}
\]

\noindent
\textbf{Bound on $\|\hat{\Sigma}_j^{-1}\|_{\text{op}}$.} Let $\rho_{\min}(\Sigma_j)$ denote the minimum eigenvalue of $\Sigma_j$. By Assumption~\ref{assump:basis-selection}(iii), $\rho_{\min}(\Sigma_j) > 0$. Let $\kappa = \min_{j \in [m]} \rho_{\min}(\Sigma_j) > 0$.

By Lemma~\ref{lem:uniform-cov-conc}, with probability at least $1 - \delta/2$:
\begin{equation}\label{eq:event-var-conc<kappa}
\max_{j \in [m]}\|\hat{\Sigma}_j - \Sigma_j\|_{\text{op}} \leq \sqrt{\frac{3 M^2 \log(4md/\delta)}{N}} \leq \frac{\kappa}{2} 
\end{equation}

for $N$ sufficiently large. On this event, by Weyl's inequality:
\[
\rho_{\min}(\hat{\Sigma}_j) \geq \rho_{\min}(\Sigma_j) - \|\hat{\Sigma}_j - \Sigma_j\|_{\text{op}} \geq \kappa - \frac{\kappa}{2} = \frac{\kappa}{2}\;.
\]
Therefore:
\[
\|\hat{\Sigma}_j^{-1}\|_{\text{op}} = \frac{1}{\rho_{\min}(\hat{\Sigma}_j)} \leq \frac{2}{\kappa}\;.
\]

\paragraph{Bound on  $\|\hat{\Sigma}_j^{-1} - \Sigma_j^{-1}\|$.}
Using the resolvent identity $\hat{\Sigma}_j^{-1} - \Sigma_j^{-1} = \hat{\Sigma}_j^{-1}(\Sigma_j - \hat{\Sigma}_j)\Sigma_j^{-1}$:
\[
\|\hat{\Sigma}_j^{-1} - \Sigma_j^{-1}\|_{\text{op}} \leq \|\hat{\Sigma}_j^{-1}\|_{\text{op}} \cdot \|\hat{\Sigma}_j - \Sigma_j\|_{\text{op}} \cdot \|\Sigma_j^{-1}\|_{\text{op}}
\]
On the high-probability event \eqref{eq:event-var-conc<kappa}:
\begin{align*}
\max_{j \in [m]}\|\hat{\Sigma}_j^{-1} - \Sigma_j^{-1}\|_{\text{op}} \leq \frac{2}{\kappa} \cdot \sqrt{\frac{3 M^2 \log(4dm/\delta)}{N}} \cdot \frac{1}{\kappa} = \frac{2}{\kappa^2}\sqrt{\frac{3 M^2 \log(4dm/\delta)}{N}}    
\end{align*}

\paragraph{Combining the bounds.}
By Lemma~\ref{lem:uniform-cross-cov-conc}, with probability at least $1 - \delta/2$:
\[
\max_{j \in [m]} \|\hat{\sigma}_{Y,j} - \sigma_{Y,j}\| \leq \sqrt{\frac{C_1 K_Y M \log(4md/\delta)}{n}}
\]

Combining all bounds, and assuming $M \geq 1$, we have with probability at least $1 - 2\delta/3$ that:
\begin{align*}
\max_{j \in [m]} \|\hat{\lambda}_j - \lambda^T_j\| &\leq \frac{1}{1+r}\left(\frac{2}{\kappa} \sqrt{\frac{C_1 K_Y M  \log(4md/\delta)}{n}} + \frac{2\sqrt{K_Y M}}{\kappa^2}\sqrt{\frac{3 M^2 \log(4md/\delta)}{N}}\right) \\
&\leq C\sqrt{\frac{K_Y M^3 \log(md/\delta)}{n}}
\end{align*}
for some universal constant $C$ depending on $\kappa$ and $r$. This concludes the proof.
\end{proof}

\begin{lemma}[Uniform Inverse Covariance Bounds]
\label{lem:uniform-inverse-cov}
Under Assumption~\ref{assump:basis-selection}, let $\kappa = \min_{j \in [m]} \rho_{\min}(\Sigma_j) > 0$. For any $\delta \in (0,1)$ and $N$ sufficiently large, with probability at least $1 - \delta$:
\[
\max_{j \in [m]}\|\hat{\Sigma}_j^{-1}\|_{\text{op}} \leq \frac{2}{\kappa}
\quad \text{and} \quad
\max_{j \in [m]}\|\hat{\Sigma}_j^{-1} - \Sigma_j^{-1}\|_{\text{op}} \leq \frac{6}{\kappa^2}\sqrt{\frac{M^2 \log(4dm/\delta)}{N}}
\]
\end{lemma}

\begin{proof}
By Lemma~\ref{lem:uniform-cov-conc}, with probability at least $1 - \delta$:
\begin{equation}\label{eq:event-var-conc<kappa/2}
\max_{j \in [m]}\|\hat{\Sigma}_j - \Sigma_j\|_{\text{op}} \leq \sqrt{\frac{3 M^2 \log(4md/\delta)}{N}} \leq \frac{\kappa}{2} 
\end{equation}
for $N$ sufficiently large.

\noindent
\textbf{Bound on $\|\hat{\Sigma}_j^{-1}\|_{\text{op}}$.}
On the event \eqref{eq:event-var-conc<kappa/2}, by Weyl's inequality:
\[
\rho_{\min}(\hat{\Sigma}_j) \geq \rho_{\min}(\Sigma_j) - \|\hat{\Sigma}_j - \Sigma_j\|_{\text{op}} \geq \kappa - \frac{\kappa}{2} = \frac{\kappa}{2}
\]
Therefore:
\[
\|\hat{\Sigma}_j^{-1}\|_{\text{op}} = \frac{1}{\rho_{\min}(\hat{\Sigma}_j)} \leq \frac{2}{\kappa}
\]

\noindent
\textbf{Bound on $\|\hat{\Sigma}_j^{-1} - \Sigma_j^{-1}\|_{\text{op}}$.}
Using the resolvent identity $\hat{\Sigma}_j^{-1} - \Sigma_j^{-1} = \hat{\Sigma}_j^{-1}(\Sigma_j - \hat{\Sigma}_j)\Sigma_j^{-1}$:
\[
\|\hat{\Sigma}_j^{-1} - \Sigma_j^{-1}\|_{\text{op}} \leq \|\hat{\Sigma}_j^{-1}\|_{\text{op}} \cdot \|\hat{\Sigma}_j - \Sigma_j\|_{\text{op}} \cdot \|\Sigma_j^{-1}\|_{\text{op}}
\]
On the event \eqref{eq:event-var-conc<kappa/2}:
\begin{align*}
\max_{j \in [m]}\|\hat{\Sigma}_j^{-1} - \Sigma_j^{-1}\|_{\text{op}} &\leq \frac{2}{\kappa} \cdot \sqrt{\frac{3 M^2 \log(4dm/\delta)}{N}} \cdot \frac{1}{\kappa} \\
&= \frac{2\sqrt{3}}{\kappa^2}\sqrt{\frac{M^2 \log(4dm/\delta)}{N}} \\
&\leq \frac{6}{\kappa^2}\sqrt{\frac{M^2 \log(4dm/\delta)}{N}}
\end{align*}
This completes the proof.
\end{proof}

\begin{lemma}[Uniform Residual Variance Concentration]
\label{lem:uniform-residual-var}
Under Assumption~\ref{assump:basis-selection}, for any $\delta \in (0,1)$ and $n$ sufficiently large:
\[
\Pbb \left\{ \max_{j \in [m]} |\hat{V}_j - V_j| \geq C K_Y M^2 \sqrt{\frac{ \log(md/\delta)}{n}} \right\} \leq \delta
\]
for some universal constant $C > 0$, where $V_j = \sigma_Y^2 - \frac{1}{1+r} \sigma_{Y,j}^T \Sigma_j^{-1} \sigma_{Y,j}$ and $\hat{V}_j$ is defined analogously with sample quantities.
\end{lemma}

\begin{proof}
Recall that:
\[
V_j = \sigma_Y^2 - \frac{1}{1+r} \sigma_{Y,j}^T \Sigma_j^{-1} \sigma_{Y,j}
\]
and
\[
\hat{V}_j = \hat{\sigma}_Y^2 - \frac{1}{1+r} \hat{\sigma}_{Y,j}^T \hat{\Sigma}_j^{-1} \hat{\sigma}_{Y,j}
\]

We decompose the error as:
\begin{equation}\label{eq:basis-selection-res-var-conc}
\hat{V}_j - V_j = \left(\hat{\sigma}_Y^2 - \sigma_Y^2\right) - \frac{1}{1+r}\left(\hat{\sigma}_{Y,j}^T \hat{\Sigma}_j^{-1} \hat{\sigma}_{Y,j} - \sigma_{Y,j}^T \Sigma_j^{-1} \sigma_{Y,j}\right)\;.
\end{equation}
In the rest of proof, we will demonstrate high probability bounds on these two terms.

\paragraph{1. Bound on $|\hat{\sigma}_Y^2 - \sigma_Y^2|$.}
We have $\sigma_Y^2 = \mathrm{Var}(Y)$ and $\hat{\sigma}_Y^2$ is the sample variance:
\[
\hat{\sigma}_Y^2 = \frac{1}{n}\sum_{i=1}^n (Y_i - \bar{Y}_n)^2, \quad \sigma_Y^2 = \mathbb{E}[(Y - \mathbb{E}[Y])^2]\;.
\]

Define the centered variables $\tilde{Y}_i = Y_i - \mathbb{E}[Y]$. Then $\mathbb{E}[\tilde{Y}_i] = 0$, $\mathbb{E}[\tilde{Y}_i^2] = \sigma_Y^2$, and
\begin{align*}
\hat{\sigma}_Y^2 
= \frac{1}{n}\sum_{i=1}^n (\tilde{Y}_i - \bar{\tilde{Y}}_n)^2
= \frac{1}{n}\sum_{i=1}^n \tilde{Y}_i^2 - \bar{\tilde{Y}}_n^2
\end{align*}
where $\bar{\tilde{Y}}_n = \frac{1}{n}\sum_{i=1}^n \tilde{Y}_i$. 
We can write
\begin{align*}
|\hat{\sigma}_Y^2 - \sigma_Y^2| &= \left|\frac{1}{n}\sum_{i=1}^n \tilde{Y}_i^2 - \bar{\tilde{Y}}_n^2 - \sigma_Y^2\right| \\
&\leq \left|\frac{1}{n}\sum_{i=1}^n \tilde{Y}_i^2 - \sigma_Y^2\right| + \bar{\tilde{Y}}_n^2.
\end{align*}
We will bound separately these two terms. We have
\[
|\tilde{Y}_i| = |Y_i - \mathbb{E}[Y]| \leq |Y_i| + |\mathbb{E}[Y]| \leq 2\sqrt{K_Y} \quad \text{a.s.}
\]
Hence $\tilde{Y}_i^2 \leq 4K_Y$ a.s., and by Jensen's inequality $\sigma_Y^2 = \mathbb{E}[\tilde{Y}_i^2] \leq 4K_Y$.

Define $W_i = \tilde{Y}_i^2 - \sigma_Y^2$. Then:
\begin{itemize}
\item $\mathbb{E}[W_i] = \mathbb{E}[\tilde{Y}_i^2] - \sigma_Y^2 = 0$
\item Since $0 \leq \tilde{Y}_i^2 \leq 4K_Y$ and $0 \leq \sigma_Y^2 \leq 4K_Y$, we have $|W_i| = |\tilde{Y}_i^2 - \sigma_Y^2| \leq 4K_Y$ a.s.
\end{itemize}

By Hoeffding's inequality:
\[
\mathbb{P}\left\{\left|\frac{1}{n}\sum_{i=1}^n W_i\right| \geq t\right\}  \leq 2\exp\left(-\frac{nt^2}{32K_Y^2}\right)\;.
\]

Setting $t = 6\sqrt{\frac{K_Y^2 \log(4/\delta)}{n}}$, we get:
\[
\mathbb{P}\left\{\left|\frac{1}{n}\sum_{i=1}^n \tilde{Y}_i^2 - \sigma_Y^2\right| \geq 6\sqrt{\frac{K_Y^2 \log(4/\delta)}{n}}\right\} \leq 2\exp\left(-\frac{36\log(4/\delta)}{32}\right) < \delta/4\;.
\]

Similarly, $\bar{\tilde{Y}}_n = \frac{1}{n}\sum_{i=1}^n \tilde{Y}_i$ is the sum of $n$ iid random variables all bounded almost surely in $[-2\sqrt{K_Y}, 2\sqrt{K_Y}]$, and Chernoff bound gives for all $t \geq 0$ that
\[
\Pbb( |\bar{\tilde{Y}}_n| \geq t) \leq 2 \exp\left(-\frac{nt^2}{32 K_Y^2} \right)\;,
\]
hence with probability at least $1-\delta/4$ we have that $\bar{\tilde{Y}}_n \leq  6\sqrt{\frac{K_Y^2 \log(4/\delta)}{n}}$.

By the union bound, we deduce that with probability at least $1 - \delta/2$:
\begin{align*}
|\hat{\sigma}_Y^2 - \sigma_Y^2|
&\leq \left|\frac{1}{n}\sum_{i=1}^n \tilde{Y}_i^2 - \sigma_Y^2\right| + \bar{\tilde{Y}}_n^2 \\
&\leq 6\sqrt{\frac{K_Y^2 \log(4/\delta)}{n}} + \frac{36 K_Y^2 \log(4/\delta)}{n} \\
&\leq C_1\sqrt{\frac{K_Y^2 \log(1/\delta)}{n}}
\end{align*}
for some universal constant $C_1 > 0$ and $n$ sufficiently large.

\paragraph{2. Bound on $|\hat{\sigma}_{Y,j}^T \hat{\Sigma}_j^{-1} \hat{\sigma}_{Y,j} - \sigma_{Y,j}^T \Sigma_j^{-1} \sigma_{Y,j}|$.}
We use the identity:
\begin{align*}
&\hat{\sigma}_{Y,j}^T \hat{\Sigma}_j^{-1} \hat{\sigma}_{Y,j} - \sigma_{Y,j}^T \Sigma_j^{-1} \sigma_{Y,j} \\
&= (\hat{\sigma}_{Y,j} - \sigma_{Y,j})^T \hat{\Sigma}_j^{-1} \hat{\sigma}_{Y,j} + \sigma_{Y,j}^T \hat{\Sigma}_j^{-1} (\hat{\sigma}_{Y,j} - \sigma_{Y,j}) \\
&\quad + \sigma_{Y,j}^T (\hat{\Sigma}_j^{-1} - \Sigma_j^{-1}) \sigma_{Y,j}
\end{align*}

By the triangle inequality:
\begin{align*}
&|\hat{\sigma}_{Y,j}^T \hat{\Sigma}_j^{-1} \hat{\sigma}_{Y,j} - \sigma_{Y,j}^T \Sigma_j^{-1} \sigma_{Y,j}| \\
&\leq \|\hat{\sigma}_{Y,j} - \sigma_{Y,j}\| \cdot \|\hat{\Sigma}_j^{-1}\|_{\text{op}} \cdot \|\hat{\sigma}_{Y,j}\| \\
&\quad + \|\sigma_{Y,j}\| \cdot \|\hat{\Sigma}_j^{-1}\|_{\text{op}} \cdot \|\hat{\sigma}_{Y,j} - \sigma_{Y,j}\| \\
&\quad + \|\sigma_{Y,j}\|^2 \cdot \|\hat{\Sigma}_j^{-1} - \Sigma_j^{-1}\|_{\text{op}}
\end{align*}

In the proof of Lemma~\ref{lem:uniform-inverse-cov}, we showed that with probability at least $1 - \delta/4$, it holds for all $j \in [m]$ that $\|\hat{\Sigma}_j^{-1}\|_{\text{op}} \leq 2/\kappa$  and 
$\|\hat{\Sigma}_j^{-1} - \Sigma_j^{-1}\|_{\text{op}} \leq \frac{6}{\kappa^2}\sqrt{\frac{M^2 \log(4dm/\delta)}{N}}$

Also, by Lemma~\ref{lem:uniform-cross-cov-conc}, with probability at least $1 - \delta/4$ we have
\[
\max_{j \in [m]} \|\hat{\sigma}_{Y,j} - \sigma_{Y,j}\| \leq 3\sqrt{\frac{2 K_Y M \log(4md/\delta)}{n}}
\]
and by Jensen's inequality we have
\[
\|\sigma_{Y,j}\| = \|\mathbb{E}[Y \phi^j(f(X))]\| \leq \mathbb{E}[|Y| \cdot \|\phi^j(f(X))\|] \leq \sqrt{K_Y M}\;.
\]
We deduce that for $n$ sufficiently large, we have under the previous event that
\[
\|\hat{\sigma}_{Y,j}\| \leq \|\sigma_{Y,j}\| + \|\hat{\sigma}_{Y,j} - \sigma_{Y,j}\| \leq 2\sqrt{K_Y M}\;.
\]

Combining these two bounds, we obtain with probability at least $1 - \delta$:
\begin{align*}
&|\hat{\sigma}_{Y,j}^T \hat{\Sigma}_j^{-1} \hat{\sigma}_{Y,j} - \sigma_{Y,j}^T \Sigma_j^{-1} \sigma_{Y,j}| \\
&\leq 3\sqrt{\frac{2 K_Y M \log(4md/\delta)}{n}} \cdot \frac{2}{\kappa} \cdot 2\sqrt{K_Y M} \\
&\quad + \sqrt{K_Y M} \cdot \frac{2}{\kappa} \cdot 3\sqrt{\frac{2 K_Y M \log(4md/\delta)}{n}} \\
&\quad + K_Y M \cdot \frac{6}{\kappa^2}\sqrt{\frac{M^2 \log(4dm/\delta)}{N}} \\
&\leq C_2 K_Y M^2 \sqrt{\frac{ \log(md/\delta)}{n}}
\end{align*}
for some constant $C_2$ depending on $\kappa$ and $r$.

Substituting in \eqref{eq:basis-selection-res-var-conc}, by the union bound over all events, with probability at least $1 - \delta$:
\begin{align*}
\max_{j \in [m]} |\hat{V}_j - V_j| &\leq C_1\sqrt{\frac{K_Y^2 \log(1/\delta)}{n}} + C_2K_Y M^2 \sqrt{\frac{ \log(md/\delta)}{n}} \\
&\leq C K_Y M^2 \sqrt{\frac{ \log(md/\delta)}{n}}
\end{align*}
for some universal constant $C$ depending on $\kappa$ and $r$.
\end{proof}

\begin{lemma}[Consistency of Selection]
\label{lem:selection-consistency}
Under Assumption~\ref{assump:basis-selection}, $\hat{j} \xrightarrow{p} j^*$.
\end{lemma}

\begin{proof}
Recall that $j^* = \arg\min_{j \in [m]} V_j$ and $\hat{j} = \arg\min_{j \in [m]} \hat{V}_j$. We need to show that for any $\epsilon > 0$:
\[
\mathbb{P}\{\hat{j} \neq j^*\} \to 0 \quad \text{as } n \to \infty
\]

Let $\Delta = \min_{j \neq j^*} (V_j - V_{ j^*})$ be the minimum gap between the optimal transformation and all other bases. We assume that $\Delta > 0$, i.e. $j^*$ is the unique minimizer.

The event $\{\hat{j} \neq j^*\}$ implies that there exists some $j \neq j^*$ such that $\hat{V}_j \leq \hat{V}_{ j^*}$. This can be written as:
\begin{align*}
\hat{V}_j \leq \hat{V}_{ j^*} &\iff \hat{V}_j - V_j + V_j \leq \hat{V}_{ j^*} - V_{ j^*} + V_{ j^*} \\
&\iff (\hat{V}_j - V_j) - (\hat{V}_{ j^*} - V_{ j^*}) \leq -(V_j - V_{ j^*})
\end{align*}

Since $V_j - V_{ j^*} \geq \Delta$ for all $j \neq j^*$, we have:
\[
\{\hat{j} \neq j^*\} \subseteq \bigcup_{j \neq j^*} \left\{(\hat{V}_j - V_j) - (\hat{V}_{ j^*} - V_{ j^*}) \leq -\Delta\right\}
\]

By the triangle inequality:
\[
|(\hat{V}_j - V_j) - (\hat{V}_{ j^*} - V_{ j^*})| \leq |\hat{V}_j - V_j| + |\hat{V}_{ j^*} - V_{ j^*}| \leq 2\max_{k \in [m]} |\hat{V}_k - V_k|
\]

Therefore:
\begin{align*}
\{\hat{j} \neq j^*\} &\subseteq \left\{\max_{k \in [m]} |\hat{V}_k - V_k| \geq \frac{\Delta}{2}\right\}
\end{align*}

\textbf{Step 4: Apply uniform convergence.}
By Lemma~\ref{lem:uniform-residual-var}, for any $\delta \in (0,1)$:
\[
\mathbb{P}\left\{\max_{j \in [m]} |\hat{V}_j - V_j| \geq C K_Y M^2 \sqrt{\frac{ \log(md/\delta)}{n}}\right\} \leq \delta
\]

For $n$ sufficiently large such that $C K_Y M^2 \sqrt{\frac{ \log(md/\delta)}{n}} < \frac{\Delta}{2}$, we have:
\[
\mathbb{P}\{\hat{j} \neq j^*\} \leq \mathbb{P}\left\{\max_{k \in [m]} |\hat{V}_k - V_k| \geq \frac{\Delta}{2}\right\} \leq \delta
\]

Therefore, we proved that for any $\delta \in (0,1)$, we have for $n$ sufficiently large that $\mathbb{P}\{\hat{j} \neq j^*\} \leq \delta$, which means that
\[
\mathbb{P}\{\hat{j} \neq j^*\} \to 0 \quad \text{as } n \to \infty\;,
\]
 hence $\hat{j} \xrightarrow{p} j^*$.
\end{proof}

\subsubsection{Proof of Theorem~\ref{thm:basis-selection}}

We now prove the three claims of Theorem~\ref{thm:basis-selection} using the lemmas established above.

\begin{proof}[Proof of Theorem~\ref{thm:basis-selection}]

\textbf{Part (i): Consistent selection.}
This follows immediately from Lemma~\ref{lem:selection-consistency}, which establishes that $\hat{j} \xrightarrow{p} j^*$ under Assumption~\ref{assump:basis-selection}.

\textbf{Part (ii): Asymptotic normality.}
We need to show that:
\[
\sqrt{n}(\hat{\theta}_{\text{GPPI}}^{\hat{j}} - \theta^*) \xrightarrow{d} \mathcal{N}(0, V_{j^*})
\]

By Theorem~\ref{thm:gppi-main}, for each fixed transformation $j$, the GPPI estimator with estimated weights satisfies:
\[
\sqrt{n}(\hat{\theta}_{\text{GPPI}}^j - \theta^*) \xrightarrow{d} \mathcal{N}(0, V_j)
\]

We decompose the adaptive estimator:
\begin{align*}
\sqrt{n}(\hat{\theta}_{\text{GPPI}}^{\hat{j}} - \theta^*) &= \sqrt{n}(\hat{\theta}_{\text{GPPI}}^{j^*} - \theta^*) + \sqrt{n}(\hat{\theta}_{\text{GPPI}}^{\hat{j}} - \hat{\theta}_{\text{GPPI}}^{j^*})
\end{align*}

On the event $\{\hat{j} = j^*\}$, the second term is zero, and we have for any $\epsilon > 0$:
\begin{align*}
\mathbb{P}\left\{\left|\sqrt{n}(\hat{\theta}_{\text{GPPI}}^{\hat{j}} - \hat{\theta}_{\text{GPPI}}^{j^*})\right| > \epsilon\right\} &\leq  \mathbb{P}\{\hat{j} \neq j^*\} \to 0
\end{align*}
by Lemma~\ref{lem:selection-consistency}. Tis means that $\sqrt{n}(\hat{\theta}_{\text{GPPI}}^{\hat{j}} - \hat{\theta}_{\text{GPPI}}^{j^*}) = o_p(1)$.

By Slutsky's theorem, we obtain
\[
\sqrt{n}(\hat{\theta}_{\text{GPPI}}^{\hat{j}} - \theta^*) \xrightarrow{d} \mathcal{N}(0, V_{j^*})
\]

\textbf{Part (iii): Variance bound.}
We need to show that:
\[
V_{\hat{j}} \leq \min_{j \in [m]} V_j + O_p\left(\sqrt{\frac{\log(m)}{n}}\right)
\]

By the definition of $\hat{j}$ and $j^*$, we have that $\hat{j}$ minimizes the empirical variance, i.e. $\hat{V}_{\hat{j}} \leq \hat{V}_{j^*}$, hence
\begin{align*}
\hat{V}_{\hat{j}} 
&= V_{j^*} + (\hat{V}_{j^*} - V_{j^*}) + (\hat{V}_{\hat{j}} - \hat{V}_{j^*}) \\
&\leq V_{j^*} + (\hat{V}_{j^*} - V_{j^*})\\
&\leq \min_{j \in [m]} V_j + \max_{j \in [m]} |\hat{V}_j - V_j|
\end{align*}

By Lemma~\ref{lem:uniform-residual-var}, for any $\delta \in (0,1)$:
\[
\mathbb{P}\left\{\max_{j \in [m]} |\hat{V}_j - V_j| \geq C K_Y M^2 \sqrt{\frac{\log(md/\delta)}{n}}\right\} \leq \delta
\]

Therefore, with probability $1-\delta$
\[
\hat{V}_{\hat{j}} \leq \min_{j \in [m]} V_j + C K_Y M^2 \sqrt{\frac{\log(md/\delta)}{n}}
\]

This completes the proof of all three parts of Theorem~\ref{thm:basis-selection}.
\end{proof}

\begin{corollary}\label{cor:greedy_gppi_coverage}
Under the assumptions of Theorem \ref{thm:basis-selection}, the estimator of GPPI with greedy selection satisfies
$$
\Pr(\theta^* \in I_{\hat j}) \geq 1 - \alpha - \frac{C'\, z_{\alpha/2} K_Y M^2}{V_{j^*}}\sqrt{\frac{\log(mdn)}{n}}.
$$
which proves asymptotic coverage.
\end{corollary}

\begin{proof}
Fix $\delta \in (0,1)$ and let $\varepsilon = C K_Y M^2 \sqrt{\log(md/\delta)/n}$. By Lemma B.12, with probability at least $1-\delta$,
$$\hat V_{\hat j} \geq V_{\hat j} - \varepsilon \geq V_{j^*} - \varepsilon.$$
Denoting by $\mathcal{E}$ this event,  $I_{\hat j} = \hat\theta^{\hat j}_{\mathrm{GPPI}} \pm z_{\alpha/2}\sqrt{\hat V_{\hat j}/n}$ and $I^*$ the analogous CI with $V_{j^*} - \varepsilon$ in place of $\hat V_{\hat j}$, we have $I^* \subseteq I_{\hat j}$. By Theorem 3.2(ii), $Z_n := \sqrt{n}(\hat\theta^{\hat j}_{\mathrm{GPPI}} - \theta^*)/\sqrt{V_{j^*}} \xrightarrow{d} \mathcal{N}(0,1)$, hence
$$\Pr(\theta^* \in I_{\hat j} \mid \mathcal{E}) \geq \Pr(\theta^* \in I^* \mid \mathcal{E}) = \Pr\!\left(|Z_n| \leq z_{\alpha/2}\sqrt{1 - \varepsilon/V_{j^*}}\right) \geq \Pr\!\left(|Z_n| \leq z_{\alpha/2}(1 - \varepsilon/V_{j^*})\right),$$
Denoting by $\Phi$ the standard normal CDF and using Berry–Esseen, assuming finite third moments, together with $1/\sqrt{2\pi}$-Lipschitz continuity of $\Phi$,
$$\Pr(\theta^* \in I_{\hat j} \mid \mathcal{E}) \geq 2\Phi\!\left(z_{\alpha/2}(1 - \varepsilon/V_{j^*})\right) - 1 - O(n^{-1/2}) \geq 1 - \alpha - \frac{2 z_{\alpha/2}}{\sqrt{2\pi}\, V_{j^*}}\,\varepsilon - O(n^{-1/2}).$$
Accounting for the failure event of probability at most $\delta$,
$$\Pr(\theta^* \in I_{\hat j}) \geq 1 - \alpha - \frac{2 z_{\alpha/2}}{\sqrt{2\pi}\, V_{j^*}}\,\varepsilon - O(n^{-1/2}) - \delta.$$
Choosing $\delta = 1/\sqrt{n}$ gives $\varepsilon = O(\sqrt{\log(mdn)/n})$ and
$$\Pr(\theta^* \in I_{\hat j}) \geq 1 - \alpha - \frac{C'\, z_{\alpha/2} K_Y M^2}{V_{j^*}}\sqrt{\frac{\log(mdn)}{n}}.$$

\end{proof}
\subsection{Proof of Theorem~\ref{thm:dimension-scaling}}

We prove that GPPI++ remains asymptotically valid when the transformation dimension $d = d_n$ grows with sample size $n$ at an appropriate rate. In all the proof, we assume without loss of generality that $c_d \leq 1 \leq M_d$ for all $d$. We demonstrate below several auxiliary lemmas before proving the theorem.

\subsubsection{Concentration of Sample Covariance}

\begin{lemma}[Covariance Concentration]
\label{lem:cov-conc}
Under the assumptions of Theorem \ref{thm:dimension-scaling}
\[
\|\hat{\Sigma}_d - \Sigma_d\|_{\text{op}} = O_p\left(\sqrt{\frac{M_d^2 \log d}{N}}\right).
\]
\end{lemma}

\begin{proof}
Without loss of generality, we assume that $\kappa_0 \geq 1$.
For transformation $\phi_d: \mathbb{R} \to \mathbb{R}^d$, we have $\phi_d(f(X_i^U))$ are i.i.d. random vectors. Define the centered random matrix:
\[
Z_i = \phi_d(f(X_i^U))\phi_d(f(X_i^U))^T - \Sigma_d \in \mathbb{R}^{d \times d}.
\]

Then $\hat{\Sigma}_d - \Sigma_d = \frac{1}{N}\sum_{i=1}^N Z_i$ with $\Ebb[Z_i] = 0$.

By the matrix Bernstein inequality (Theorem 1.6.2 in \citet{tropp2015introduction}), for independent centered self-adjoint matrices $\tilde{Z}_1, \ldots, \tilde{Z}_N$ with $\|\tilde{Z}_i\|_{\text{op}} \leq L$ almost surely:
\[
\mathbb{P}\left\{\left\|\sum_{i=1}^N \tilde{Z}_i\right\|_{\text{op}} \geq t\right\} \leq 2d \exp\left(-\frac{t^2/2}{\sigma^2 + Lt/3}\right),
\]
where $\sigma^2 = \left\|\sum_{i=1}^N \Ebb[Z_i^2]\right\|_{\text{op}}$. To use this concentration bound on $(Z_i)_i$, we need to prove bounds $\sigma^2$ and $L$.

\noindent
\noindent
\textbf{Bound on $\sigma^2$:} By Assumption, $\|\phi_d(f(X))\|^2 \leq M_d$. Therefore:
\begin{align*}
\Ebb[Z_i^2] &= \Ebb[(\phi_d(f(X))\phi_d(f(X))^T - \Sigma_d)^2] \\
&\preceq 2\Ebb[(\phi_d(f(X))\phi_d(f(X))^T)^2] + 2\Sigma_d^2 \\
&= 2\Ebb[\|\phi_d(f(X))\|^2 \cdot \phi_d(f(X))\phi_d(f(X))^T] + 2\Sigma_d^2 \\
&\preceq 2M_d \cdot \Ebb[\phi_d(f(X))\phi_d(f(X))^T] + 2\Sigma_d^2 \\
&= 2M_d \Sigma_d + 2\Sigma_d^2 \\
&\preceq 2M_d \|\Sigma_d\|_{\text{op}} I_d + 2\|\Sigma_d\|_{\text{op}}^2 I_d.
\end{align*}
Note that $\|\Sigma_d\|_{\text{op}} = \sup_{\|v\|=1} v^T\Sigma_d v = \sup_{\|v\|=1} \mathbb{E}[(v^T\phi_d(f(X)))^2] \leq \mathbb{E}[\|\phi_d(f(X))\|^2] \leq M_d$ by assumption. Thus:
\begin{align*}
\mathbb{E}[Z_i^2] &\preceq 2M_d^2 I_d + 2M_d^2 I_d = 4M_d^2 I_d.
\end{align*}

Therefore:
\[
\sigma^2 = \left\|\sum_{i=1}^N \Ebb[Z_i^2]\right\|_{\text{op}} = O(NM_d^2).
\]

\noindent
\textbf{Bound on $L$:} We need to bound $\|Z_i\|_{\text{op}}$. By the triangle inequality:
\[
\|Z_i\|_{\text{op}} = \|\phi_d(f(X_i))\phi_d(f(X_i))^T - \Sigma_d\|_{\text{op}} \leq \|\phi_d(f(X_i))\phi_d(f(X_i))^T\|_{\text{op}} + \|\Sigma_d\|_{\text{op}}.
\]

For any vector $v \in \mathbb{R}^d$, the rank-one matrix $vv^T$ has operator norm $\|vv^T\|_{\text{op}} = \|v\|^2$ (since $vv^T$ has a single non-zero eigenvalue equal to $\|v\|^2$). Therefore:
\[
\|\phi_d(f(X_i))\phi_d(f(X_i))^T\|_{\text{op}} = \|\phi_d(f(X_i))\|^2 \leq M_d\quad \text{a.s.}\;.
\]
Since $\|\Sigma_d\|_{\text{op}} \leq M_d$ (as shown above), we have:
\[
\|Z_i\|_{\text{op}} \leq M_d + M_d = 2M_d\;.
\]
So we can take $L = 2M_d$.

\noindent
\textbf{Applying Bernstein:} We have shown that:
\begin{itemize}
\item $\mathbb{E}[Z_i^2] \preceq 4M_d^2 I_d$, so $\sigma^2 = \left\|\sum_{i=1}^N \Ebb[Z_i^2]\right\|_{\text{op}} \leq 4NM_d^2$
\item $\|Z_i\|_{\text{op}} \leq 2M_d$ almost surely
\end{itemize}

By the matrix Bernstein inequality, for any $t > 0$:
\begin{align*}
\mathbb{P}\left\{\left\|\sum_{i=1}^N Z_i\right\|_{\text{op}} \geq t\right\} 
&\leq 2d \exp\left(-\frac{t^2/2}{\sigma^2 + Lt/3}\right) \leq 2d \exp\left(-\frac{t^2/2}{4NM_d^2 + 2M_d t/3}\right)\\
&\leq 2d \exp\left(- \frac{t^2}{8 M_d^2 N + 2 M_d t}\right)
\end{align*}
Set $t = \sqrt{24 N M_d^2\log d}$. Then:
\begin{align*}
\frac{t^2}{8 M_d^2 N + 2 M_d t} 
&= \frac{24 N M_d^2\log d}{8 M_d^2 N + 2 M_d\sqrt{24 N M_d^2\log d}} \\
&= \frac{6 \log(d)}{2 +  \sqrt{6 \log(d) / N}} \\
&\geq 2 \log d
\;,
\end{align*}
where the last inequality holds for $d$ sufficiently large, as $\log d = o(n)$. Thus
\[
\mathbb{P}\left\{\left\|\sum_{i=1}^N Z_i\right\|_{\text{op}} \geq \sqrt{24 N M_d^2\log d}\right\} 
\leq 2d \exp\left(-2\log d\right) = \frac{2}{d} \to 0.
\]
This gives:
\[
\left\|\sum_{i=1}^N Z_i\right\|_{\text{op}} = O_p\left(\sqrt{N M_d^2 \log d}\right).
\]
Therefore:
\[
\|\hat{\Sigma}_d - \Sigma_d\|_{\text{op}} = \left\|\frac{1}{N}\sum_{i=1}^N Z_i\right\|_{\text{op}} = O_p\left(\sqrt{\frac{M_d^2 \log d}{N}}\right).
\]
\end{proof}

\subsubsection{Concentration of Cross-Covariance}

\begin{lemma}[Cross-Covariance Concentration]
\label{lem:cross-cov-conc}
Under the theorem assumptions:
\[
\|\hat{\sigma}_{Y,d} - \sigma_{Y,d}\| = O_p\left(\sqrt{\frac{M_d}{n}}\right).
\]
\end{lemma}

\begin{proof}
Define $W_i = Y_i \phi_d(f(X_i^L)) - \sigma_{Y,d}$. 
Since $(W_i)_i$ are iid and have null expectations, it holds that
\begin{equation}\label{eq:sigma_diff_W}
\Ebb[\|\hat{\sigma}_{Y,d} - \sigma_{Y,d}\|^2] = \frac{1}{n}\Ebb[\|W_1\|^2]\;.
\end{equation}
Furthermore, using the variance decomposition formula then  gives
\begin{align*}
\Ebb[\|W_1\|^2]
&= \Ebb[\|Y\phi_d(f(X)) - \sigma_{Y,d}\|^2] \\
&= \Ebb[\|Y\phi_d(f(X))\|^2] - \|\sigma_{Y,d}\|^2 \\
&\leq \Ebb[\|Y\phi_d(f(X))\|^2] \\
&= \Ebb[Y^2\|\phi_d(f(X))\|^2]\\
&\leq \Ebb[Y^2] M_d\;.
\end{align*}
The last inequality holds because $\|\phi_d(f(X))\|^2 \leq M_d$ almost surely. Substituting in \eqref{eq:sigma_diff_W} and using that $Y$ has bounded second moment yields
\[
\Ebb[\|\hat{\sigma}_{Y,d} - \sigma_{Y,d}\|^2] \leq \frac{\Ebb[Y^2] M_d}{n} = O(M_d/n)\;,
\]
and we conclude from Chebyshev inequality that $\|\hat{\sigma}_{Y,d} - \sigma_{Y,d}\| = O_p(\sqrt{M_d/n})$.
\end{proof}

\subsubsection{Weight Estimation Error}

\begin{lemma}[Weight Error]
\label{lem:weight-error}
Under the assumptions of Theorem \ref{thm:dimension-scaling}
\[
\|\hat{\lambda}_d - \lambda^*_d\| 
\leq O_p\left(\frac{M_d^{3/2} \sqrt{\log(d)}}{c_d^2 \sqrt{n}}\right)\;,
\]
\end{lemma}

\begin{proof}
Using Weyl's inequality for eigenvalues then Lemma \ref{lem:cov-conc} yields
\[
|\rho_{\min}(\hat{\Sigma}_d) - \rho_{\min}(\Sigma_d)| \leq \|\hat{\Sigma}_d - \Sigma_d\|_{\text{op}} = O_p\left(\sqrt{\frac{M_d^2 \log d}{N}}\right)\;.
\]
Moreover, using the assumptions $\rho_{\min}(\Sigma_d) \geq c_d$ and $\sqrt{M_d^2 \log d/N} = o(c_d)$ we deduce that
\begin{align*}
\rho_{\min}(\hat{\Sigma}_d) 
&\geq c_d - |\rho_{\min}(\hat{\Sigma}_d) - \rho_{\min}(\Sigma_d)|\\
&\geq c_d - o_p\left(c_d\right)    \\
&= \Omega_p(c_d)
\end{align*}
where the last inequality holds for $n$ sufficiently large. It follows that
\[
\|\hat{\Sigma}_d^{-1}\|_{\text{op}} = \frac{1}{\rho_{\min}(\hat{\Sigma}_d)} = O_p(1/c_d)\;.
\]

\noindent
Decompose:
\begin{align}
\hat{\lambda}_d - \lambda^*_d &= \frac{1}{1+r}\hat{\Sigma}_d^{-1}\hat{\sigma}_{Y,d} - \frac{1}{1+r}\Sigma_d^{-1}\sigma_{Y,d}
\notag
\\
&= \frac{1}{1+r}\hat{\Sigma}_d^{-1}(\hat{\sigma}_{Y,d} - \sigma_{Y,d}) + \frac{1}{1+r}(\hat{\Sigma}_d^{-1} - \Sigma_d^{-1})\sigma_{Y,d}.
\label{aligneq:lambda_diff}
\end{align}

\noindent
\textbf{First term:} By Lemma~\ref{lem:cross-cov-conc} and $\|\hat{\Sigma}_d^{-1}\|_{\text{op}} = O_p(1/c_d)$:
\begin{equation}\label{eq:lambda_diff-term1}
\left\|\frac{1}{1+r}\hat{\Sigma}_d^{-1}(\hat{\sigma}_{Y,d} - \sigma_{Y,d})\right\| 
\leq O_p(1/c_d) \cdot O_p\left(\sqrt{\frac{M_d}{n}}\right) = O_p\left(\frac{\sqrt{M_d}}{c_d\sqrt{n}}\right).
\end{equation}

\noindent
\textbf{Second term:} Using the identity $\hat{\Sigma}_d^{-1} - \Sigma_d^{-1} = -\hat{\Sigma}_d^{-1}(\hat{\Sigma}_d - \Sigma_d)\Sigma_d^{-1}$, we obtain
\begin{align*}
\left\|\frac{1}{1+r}(\hat{\Sigma}_d^{-1} - \Sigma_d^{-1})\sigma_{Y,d}\right\| &\leq \frac{1}{1+r} \cdot \|\hat{\Sigma}_d^{-1}\|_{\text{op}} \cdot \|\Sigma_d^{-1}\|_{\text{op}} \cdot \|\hat{\Sigma}_d - \Sigma_d\|_{\text{op}} \cdot \|\sigma_{Y,d}\| \\
&\leq O_p(1/c_d) \cdot 1/c_d \cdot O_p\left(\sqrt{\frac{M_d^2 \log d}{N}}\right) \cdot \|\sigma_{Y,d}\|.
\end{align*}
The term $\|\sigma_{Y,d}\|$ can be bounded using Jensen's inequality and assumption that $\phi_d(f(X))$ is bounded a.s.:
\[
\|\sigma_{Y,d}\|^2 
= \|\Ebb[Y \phi_d(f(X))] \|^2
\leq \Ebb[ \|Y \phi_d(f(X))\|^2]
= \Ebb[ Y^2 \|\phi_d(f(X))\|^2]
\leq \Ebb[ Y^2] M_d,
\]
thus $\|\sigma_{Y,d}\| = O(\sqrt{M_d})$. It follows that
\begin{equation}\label{eq:lambda_diff-term2}
\left\|\frac{1}{1+r}(\hat{\Sigma}_d^{-1} - \Sigma_d^{-1})\sigma_{Y,d}\right\| 
= O_p\left(\frac{\sqrt{M_d^3 \log d}}{c_d^2\sqrt{n}}\right)\;.
\end{equation}

Finally, combining \eqref{aligneq:lambda_diff}, \eqref{eq:lambda_diff-term1} and \eqref{eq:lambda_diff-term2} we deduce that
\[
\|\hat{\lambda}_d - \lambda^*_d\| 
\leq O_p\left(\frac{\sqrt{M_d}}{c_d\sqrt{n}}\right) + O_p\left(\frac{\sqrt{M_d^3 \log d}}{c_d^2\sqrt{n}}\right)
= O_p\left(\frac{\sqrt{M_d/n}}{c_d}\Big(1 + \frac{M_d\sqrt{\log d}}{c_d}\;\Big)\right)\;,
\]
and with the assumption that $c_d \leq 1 \leq M_d$ we deduce that
\[
\|\hat{\lambda}_d - \lambda^*_d\| 
\leq O_p\left(\frac{M_d^{3/2} \sqrt{\log(d)}}{c_d^2 \sqrt{n}}\right)\;,
\]
\end{proof}

\subsubsection{Feature Mean Concentration}

\begin{lemma}[Feature Mean Concentration]
\label{lem:feature-mean-conc}
Under the theorem assumptions, let $\mu_d = \Ebb[\phi_d(f(X))]$. Then:
\begin{align*}
\|\bar{\phi}_{d,N}^U - \mu_d\| &= O_p\left(\sqrt{\frac{M_d}{N}}\right), \\
\|\bar{\phi}_{d,n}^L - \mu_d\| &= O_p\left(\sqrt{\frac{M_d}{n}}\right).
\end{align*}
\end{lemma}

\begin{proof}
Since $\bar{\phi}_{d,N}^U = \frac{1}{N}\sum_{i=1}^N \phi_d(f(X_i^U))$ and $\bar{\phi}_{d,n}^L = \frac{1}{n}\sum_{i=1}^n \phi_d(f(X_i^L))$ are sample means of i.i.d. random vectors, the proof follows the same concentration argument as Lemma~\ref{lem:cross-cov-conc}. 

Using the assumption $\|\phi_d(f(X))\| \leq \sqrt{M_d}$ and applying again matrix Bernstein inequality yields the stated bounds.
\end{proof}


\begin{lemma}[Second-Order Bias Vanishes]
\label{lem:second-order}
Under the assumptions of Theorem \ref{thm:dimension-scaling}
\[
\sqrt{n}(\hat{\theta}(\hat{\lambda}_d) - \hat{\theta}(\lambda^*_d)) = o_p(1)\;.
\]
\end{lemma}

\begin{proof}
Define $\Delta_\lambda = \hat{\lambda}_d - \lambda^*_d$ and $\Delta_\phi = \bar{\phi}_{d,N}^U - \bar{\phi}_{d,n}^L$. By definition:
\begin{align*}
\sqrt{n}(\hat{\theta}(\hat{\lambda}_d) - \hat{\theta}(\lambda^*_d)) 
&= \sqrt{n}\Delta_\lambda^T\Delta_\phi\;.
\end{align*}

Let $\mu_d = \Ebb[\phi_d(f(X))]$. Decompose:
\[
\Delta_\phi = (\bar{\phi}_{d,N}^U - \mu_d) + (\mu_d - \bar{\phi}_{d,n}^L).
\]

Using Lemmas~\ref{lem:feature-mean-conc} and \ref{lem:weight-error}, and that $N = \Omega(n)$ we obtain
\begin{align*}
\sqrt{n}|\Delta_\lambda^T\Delta_\phi| 
&\leq \sqrt{n} \|\Delta_\lambda\| \cdot \|\Delta_\phi\| \\
&\leq \sqrt{n} \|\Delta_\lambda\| \cdot \left(\|\bar{\phi}_{d,N}^U - \mu_d\| + \|\bar{\phi}_{d,n}^L - \mu_d\|\right) \\
&= O_p\left(\sqrt{n} \cdot 
\frac{M_d^{3/2} \sqrt{\log(d)}}{c_d^2 \sqrt{n}}
\cdot \sqrt{\frac{M_d}{n}}\right) \\
&= O_p\left(\frac{M_d^2 \sqrt{\log d}}{c_d^2 \sqrt{n}}\right)
= o_p(1)\;.
\end{align*}
\end{proof}


\begin{lemma}[Variance Estimator Consistency]
\label{lem:variance-consistency}
Let $\hat{V} = \hat{\sigma}_Y^2 - \frac{1}{1+r}\hat{\sigma}_{Y,d}^T\hat{\Sigma}_d^{-1}\hat{\sigma}_{Y,d}$ be the plug-in variance estimator. Under the assumptions of Theorem \ref{thm:dimension-scaling}, it holds that
\[
\hat{V} \xrightarrow{p} V_\infty.
\]
\end{lemma}

\begin{proof}
Recall that
\[
\hat{V} = \hat{\sigma}_Y^2 - \frac{1}{1+r}\hat{\sigma}_{Y,d}^T\hat{\Sigma}_d^{-1}\hat{\sigma}_{Y,d}\;,
\]
hence we have that
\begin{align}
\hat{V} - V_\infty &= \hat{V} - V^*_d + V^*_d - V_\infty \notag \\
&= \left(\hat{\sigma}_Y^2 - \sigma_Y^2\right) - \frac{1}{1+r}\left(\hat{\sigma}_{Y,d}^T\hat{\Sigma}_d^{-1}\hat{\sigma}_{Y,d} - \sigma_{Y,d}^T\Sigma_d^{-1}\sigma_{Y,d}\right) + (V^*_d - V_\infty)\;. \label{eq:variance_decomp_correct}
\end{align}
We have by assumption that $V^*_d = \sigma_Y^2 - \frac{1}{1+r}\sigma_{Y,d}^T\Sigma_d^{-1}\sigma_{Y,d} \to V_\infty$, and since $\hat{\sigma}_Y^2 = \frac{1}{n}\sum_{i=1}^n Y_i^2$ and $\mathbb{E}[Y^2] < \infty$, by the strong law of large numbers:
\[
\hat{\sigma}_Y^2 - \sigma_Y^2 \xrightarrow{p} 0\;.
\]
Therefore, it remains to prove that $\left(\hat{\sigma}_{Y,d}^T\hat{\Sigma}_d^{-1}\hat{\sigma}_{Y,d} - \sigma_{Y,d}^T\Sigma_d^{-1}\sigma_{Y,d}\right) \xrightarrow{p} 0$. Let us decompose this as two terms:
\begin{align}
&\left|\hat{\sigma}_{Y,d}^T\hat{\Sigma}_d^{-1}\hat{\sigma}_{Y,d} - \sigma_{Y,d}^T\Sigma_d^{-1}\sigma_{Y,d}\right| \notag \\
&\leq \left|\hat{\sigma}_{Y,d}^T\hat{\Sigma}_d^{-1}\hat{\sigma}_{Y,d} - \sigma_{Y,d}^T\hat{\Sigma}_d^{-1}\sigma_{Y,d}\right| + \left|\sigma_{Y,d}^T\hat{\Sigma}_d^{-1}\sigma_{Y,d} - \sigma_{Y,d}^T\Sigma_d^{-1}\sigma_{Y,d}\right|. \label{eq:quad_triangle}
\end{align}

\noindent
\paragraph{Subterm 1.}
The first term can be upper bounded as
\begin{align}
\left|\hat{\sigma}_{Y,d}^T\hat{\Sigma}_d^{-1}\hat{\sigma}_{Y,d} - \sigma_{Y,d}^T\hat{\Sigma}_d^{-1}\sigma_{Y,d}\right| 
&= \left|(\hat{\sigma}_{Y,d} - \sigma_{Y,d})^T\hat{\Sigma}_d^{-1}(\hat{\sigma}_{Y,d} + \sigma_{Y,d})\right| \notag \\
&\leq \|\hat{\sigma}_{Y,d} - \sigma_{Y,d}\| \|\hat{\Sigma}_d^{-1}\|_{\text{op}} \|\hat{\sigma}_{Y,d} + \sigma_{Y,d}\|. \label{eq:subterm1}
\end{align}
By Lemma \ref{lem:cross-cov-conc}: $\|\hat{\sigma}_{Y,d} - \sigma_{Y,d}\| = O_p(\sqrt{M_d/n})$. Also, from the proof of Lemma \ref{lem:weight-error}: $\|\hat{\Sigma}_d^{-1}\|_{\text{op}} = O_p(1/c_d)$. And the norm of the sum satisfies \[
\|\hat{\sigma}_{Y,d} + \sigma_{Y,d}\| \leq \|\hat{\sigma}_{Y,d}\| + \|\sigma_{Y,d}\| = O_p(\sqrt{M_d}) + O(\sqrt{M_d}) = O_p(\sqrt{M_d})\;.
\]
Therefore:
\begin{equation}]\label{eq:quad-triangle-term1}
\left|\hat{\sigma}_{Y,d}^T\hat{\Sigma}_d^{-1}\hat{\sigma}_{Y,d} - \sigma_{Y,d}^T\hat{\Sigma}_d^{-1}\sigma_{Y,d}\right|  = O_p\left(\sqrt{\frac{M_d}{n}} \cdot \frac{1}{c_d} \cdot \sqrt{M_d}\right) = O_p\left(\frac{M_d}{c_d\sqrt{n}}\right).    
\end{equation}

\noindent
\textbf{Subterm 2:} For the second term in \eqref{eq:quad_triangle}, we have
\begin{align}
\left|\sigma_{Y,d}^T\hat{\Sigma}_d^{-1}\sigma_{Y,d} - \sigma_{Y,d}^T\Sigma_d^{-1}\sigma_{Y,d}\right| 
&= \left|\sigma_{Y,d}^T(\hat{\Sigma}_d^{-1} - \Sigma_d^{-1})\sigma_{Y,d}\right| \notag \\
&\leq \|\sigma_{Y,d}\|^2 \|\hat{\Sigma}_d^{-1} - \Sigma_d^{-1}\|_{\text{op}}. \label{eq:subterm2}
\end{align}
Using the identity $\hat{\Sigma}_d^{-1} - \Sigma_d^{-1} = -\hat{\Sigma}_d^{-1}(\hat{\Sigma}_d - \Sigma_d)\Sigma_d^{-1}$ then Lemma \ref{lem:cov-conc} yields
\begin{align*}
\|\hat{\Sigma}_d^{-1} - \Sigma_d^{-1}\|_{\text{op}} 
&\leq \|\hat{\Sigma}_d^{-1}\|_{\text{op}} \|\hat{\Sigma}_d - \Sigma_d\|_{\text{op}} \|\Sigma_d^{-1}\|_{\text{op}} \notag \\
&= O_p(1/c_d) \cdot O_p\left(\sqrt{\frac{M_d^2 \log d}{N}}\right) \cdot O(1/c_d) \notag \\
&= O_p\left(\frac{M_d\sqrt{\log d}}{c_d^2\sqrt{N}}\right),
\end{align*}
Since $\|\sigma_{Y,d}\|^2 = O(M_d)$ (as shown in Lemma \ref{lem:weight-error}), we have
\begin{equation}\label{eq:quad-triangle-term2}
\left|\sigma_{Y,d}^T\hat{\Sigma}_d^{-1}\sigma_{Y,d} - \sigma_{Y,d}^T\Sigma_d^{-1}\sigma_{Y,d}\right| 
= O_p\left(\frac{M_d^2\sqrt{\log d}}{c_d^2\sqrt{N}}\right).    
\end{equation}

Substituting \eqref{eq:quad-triangle-term1} and \eqref{eq:quad-triangle-term2} into \eqref{eq:quad_triangle}, we obtain:
\begin{align*}
\left|\hat{\sigma}_{Y,d}^T\hat{\Sigma}_d^{-1}\hat{\sigma}_{Y,d} - \sigma_{Y,d}^T\Sigma_d^{-1}\sigma_{Y,d}\right| 
&= O_p\left(\frac{M_d}{c_d\sqrt{n}}\right) + O_p\left(\frac{M_d^2\sqrt{\log d}}{c_d^2\sqrt{N}}\right) \\
&=  O_p\left(\frac{M_d^2\sqrt{\log d}}{c_d^2\sqrt{N}}\right)
\end{align*}
where we used that $c_d \leq 1 \leq M_d$ hence $\frac{M_d}{c_d\sqrt{n}} \leq \frac{M_d}{c_d\sqrt{n}} \cdot M_d/c_d \cdot \sqrt{\log d} = \frac{M_d^2\sqrt{\log d}}{c_d^2\sqrt{N}}$. Moreover, the assumption of Theorem \ref{thm:dimension-scaling} guarantees that this term vanishes, hence
\[
\frac{1}{1+r}\left|\hat{\sigma}_{Y,d}^T\hat{\Sigma}_d^{-1}\hat{\sigma}_{Y,d} - \sigma_{Y,d}^T\Sigma_d^{-1}\sigma_{Y,d}\right| = o_p(1)\;,
\]
which concludes the proof.
\end{proof}

\subsubsection{Main Proof}

\begin{proof}[Proof of Theorem~\ref{thm:dimension-scaling}]
The proof of the theorem follows immediately from the previous lemmas.

\noindent
\paragraph{Asymptotic normality}
By assumption (v), $V^*_d \to V_\infty$. As in the proof of Theorem~\ref{thm:gppi-main}, for the oracle estimator with known optimal weights:
\[
\sqrt{n}(\hat{\theta}(\lambda^*_d) - \theta^*) \xrightarrow{d} \mathcal{N}(0, V_\infty)\;.
\]
On the other hand, we can decompose
\[
\sqrt{n}(\hat{\theta}_{\text{GPPI}} - \theta^*) = \sqrt{n}(\hat{\theta}(\lambda^*_d) - \theta^*) + \sqrt{n}(\hat{\theta}(\hat{\lambda}_d) - \hat{\theta}(\lambda^*_d)).
\]

By Lemma~\ref{lem:second-order}, the second term is $o_p(1)$. Therefore, we deduce By Slutsky's theorem that
\[
\sqrt{n}(\hat{\theta}_{\text{GPPI}} - \theta^*) \xrightarrow{d} \mathcal{N}(0, V_\infty).
\]

\noindent
\textbf{Variance consistency:} By Lemma~\ref{lem:variance-consistency}, $\hat{V} \xrightarrow{p} V_\infty$.

\end{proof}
\section{Experimental Details}
\label{app:experimental-details}

This appendix provides comprehensive details on the experimental setup, including dataset statistics, persona construction procedures, prompt templates, and implementation specifics.

\subsection{Persona Construction}

\paragraph{Jester Jokes personas.}
For each user, we construct a persona from their ratings on the training jokes. The persona includes:
\begin{itemize}[leftmargin=*]
\item \textbf{Rating statistics:} Total number of ratings, mean rating, standard deviation, min, max, and median.
\item \textbf{Rating distribution:} Counts and percentages in four buckets: very negative ($[-10, -5)$), negative ($[-5, 0)$), positive ($[0, 5)$), very positive ($[5, 10]$).
\item \textbf{Favorite jokes:} Up to 15 jokes rated above 7, sorted by rating, with full joke text.
\item \textbf{Least favorite jokes:} Up to 10 jokes rated below $-5$, sorted by rating, with full joke text.
\end{itemize}

\paragraph{MovieLens personas.}
For each user, we construct a persona from their complete rating history on training movies. The persona includes:
\begin{itemize}[leftmargin=*]
\item \textbf{Rating statistics:} Total ratings, mean, standard deviation, and rating distribution across the 10 possible values.
\item \textbf{Temporal evolution:} If the user's rating history spans more than 6 months, we divide it into 3 equal time periods and compute the mean rating in each period to capture temporal trends.
\item \textbf{Genre preferences:} For each of the 20 genres, we compute the average rating and number of movies rated. We report the top 5 genres by average rating.
\item \textbf{Favorite movies:} Up to 10 movies rated 5.0, with title and genres.
\item \textbf{Recent history:} The last 100 movies rated (chronologically), with date, title, genres, and rating.
\end{itemize}

\subsection{Prompt Templates}

We provide below a representative prompt template for MovieLens. All prompts for he other datasets follow a similar structure, including the available data from each dataset.

\begin{tcolorbox}[colback=gray!5, colframe=gray!75, boxrule=0.5pt, arc=2pt, left=3pt, right=3pt, top=3pt, bottom=3pt, breakable]
\begin{verbatim}
You are impersonating a movie viewer. Based on YOUR complete 
viewing history below, predict how YOU would rate the following 
movie on a scale of 0.5-5.0 (in 0.5 increments).

YOUR VIEWING STATISTICS:
- Total movies rated: {total_ratings}
- Overall average rating: {mean_rating:.2f}
- Evolution of average rating with time: {temporal_evolution}
- Top genres: {top_genres_with_avg_ratings}
- Rating distribution: {rating_distribution}
- Active period: {first_date} to {last_date}

SAMPLE OF YOUR FAVORITE MOVIES (rated 5.0):
{favorite_movies_with_genres}

YOUR RECENT RATING HISTORY (last 100 movies, chronological):
{recent_history_with_dates_titles_genres_ratings}

MOVIE TO RATE:
Title: {movie_title}
Genres: {movie_genres}

INSTRUCTIONS:
1. Consider YOUR genre preferences and how this movie's genres 
   align with YOUR tastes.
2. Reflect on YOUR rating patterns over time and YOUR typical 
   rating level.
3. Think about similar movies YOU have rated and how this 
   compares.

Provide ONLY a numerical rating in {0.5, 1.0, 1.5, ..., 5.0}. 
No explanations.
\end{verbatim}
\end{tcolorbox}

\subsection{Model Configuration}\label{app:config}

\paragraph{LLM models.}
We use three Mistral models:
\begin{itemize}[leftmargin=*]
\item \textbf{Mistral-7B-Instruct-v0.2:} 7 billion parameters
\item \textbf{Mistral-24B-Instruct:} 24 billion parameters
\item \textbf{Mistral-123B-Instruct:} 123 billion parameters
\end{itemize}
These models were selected for their open availability and performance on simple tasks such as the ones we test. All models use temperature 0.0 and the same prompt format with no model-specific adjustments. For inference, we used Amazon bedrock. The cost of all the inferences is less than 2000 USD. The rest of the experiments consisted in sampling and doing parameter estimation / hypothesis testing. They did not incur any further costs.

\subsection{Population-level predictions: Synthetic experiment.}
To better illustrate strenghts and weaknesses of the population level prediction framework, we do a synthetic experiments visualizing the different modes.
We generate $n=100$ samples from $\mathcal{N}(\mu_0 + \delta, 1)$ for effect sizes $\delta \in \{0.05, 0.1, 0.2, 0.3, 0.4\}$, testing $H_0: \mu = \mu_0$ at $\alpha = 0.05$ with prediction accuracy $p \in [0.5, 1]$ and asymmetry $\lambda \in [0, 0.99]$. Figure~\ref{fig:power-vs-accuracy} shows moderate $\lambda \in [0.25, 0.75]$ achieve substantial gains with accurate predictions while maintaining robustness when predictions fail, validating Theorem~\ref{thm:main-directional-pred}.

\begin{figure}[h!]
    \centering
    \includegraphics[width=\linewidth]{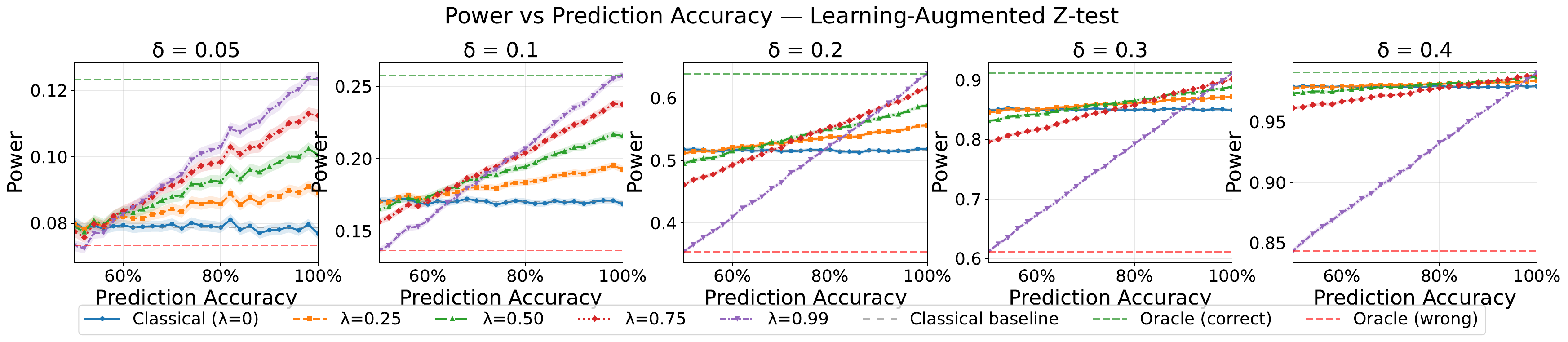}
    \caption{\textbf{Synthetic directional predictions.} Power vs prediction accuracy for various $\delta$ and $\lambda$. Moderate $\lambda$ balances gains with robustness.}
    \label{fig:power-vs-accuracy}
\end{figure}

\subsection{Item-Level GPPI Performance Analysis}
\label{app:heatmap-analysis}

We provide in this section a comprehensive item-level analysis examining how GPPI performs relative to PPI++ on the task of mean estimation for each test individual movie, joke, book and song from the datasets we used. In all the experiments, we construct confidence intervals with type I error $\alpha=5\%$.

For each item and each model (Mistral-7B, 24B, 123B), we sample $n$ labeled samples and $N$ unlabeled samples, where $N=5n$ and $n$ is chosen maximally, i.e. every user is used in the experiment, either as labeled or unlabled sample.  All the values in these experiments are computed over 1000 trials. Note that the predictions with the LLM models are only computed once, the source of randomness in the experiment is in splitting the users into labeled and unlabeled. For each dataset and for each item we compute
\begin{itemize}[leftmargin=*]
\item \textbf{PPI++ ESS gain}: The baseline improvement over the classical estimator using linear correction.
\item \textbf{GPPI gap}: The difference between the ESS gain of GPPI with greedy transform selection (using the same transformations mentioned in the experiments section) and PPI++'s ESS gain. A positive gap indicates GPPI improves over PPI++; negative indicates degradation.
\end{itemize}
We visualize these metrics as heatmaps where rows represent models and columns represent items. Items are labeled with their ID and available sample size $n$. We also visualize heatmaps for the coverage of each method. For each dataset we show two figures next to each other. In the right figure, the top panel shows PPI++ ESS gain (\%) for each of the 15 movies (columns) and three models (rows). The lower three panels: GPPI with greedy selection gap (GPPI ESS gain minus PPI++ ESS gain) without penalization, with AIC penalization, then with BIC penalization. The right figure has one additional top panel for the classical estimator. Each panel  panels corresponds to a mean estimator (classical/PPI++/GPPI greedy selection with different penalizations), and shows the coverage obtained with that estimator for different items and models, i.e. percentage of times when the true mean is in the CI constructed by that algorithm.

\begin{figure}[htbp]
    \centering
    \begin{minipage}{0.41\textwidth}
        \centering
        \includegraphics[width=\linewidth]{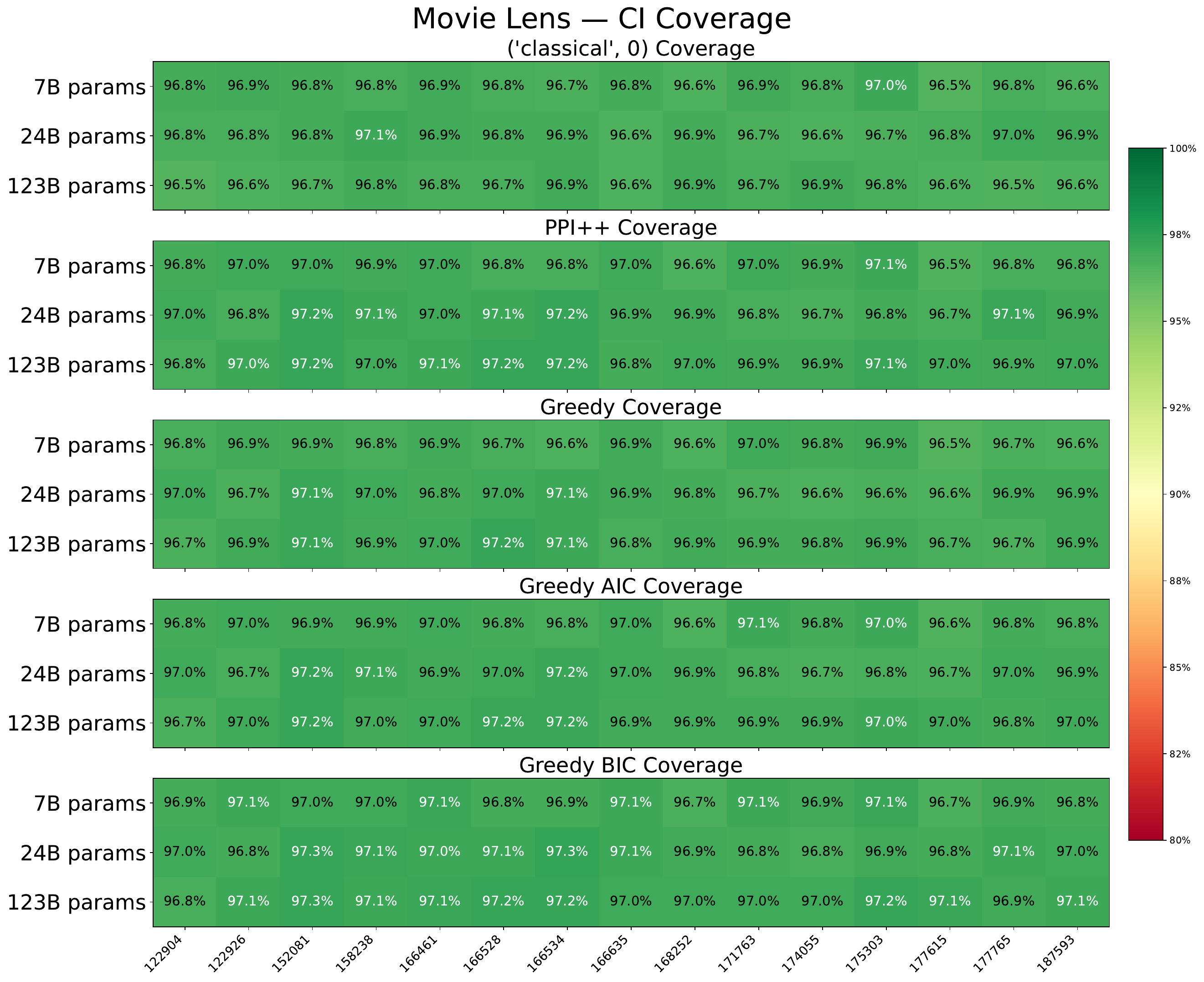}
    \end{minipage}
    \hspace{0.1cm}
    \begin{minipage}{0.55\textwidth}
        \centering
        \includegraphics[width=0.9\linewidth]{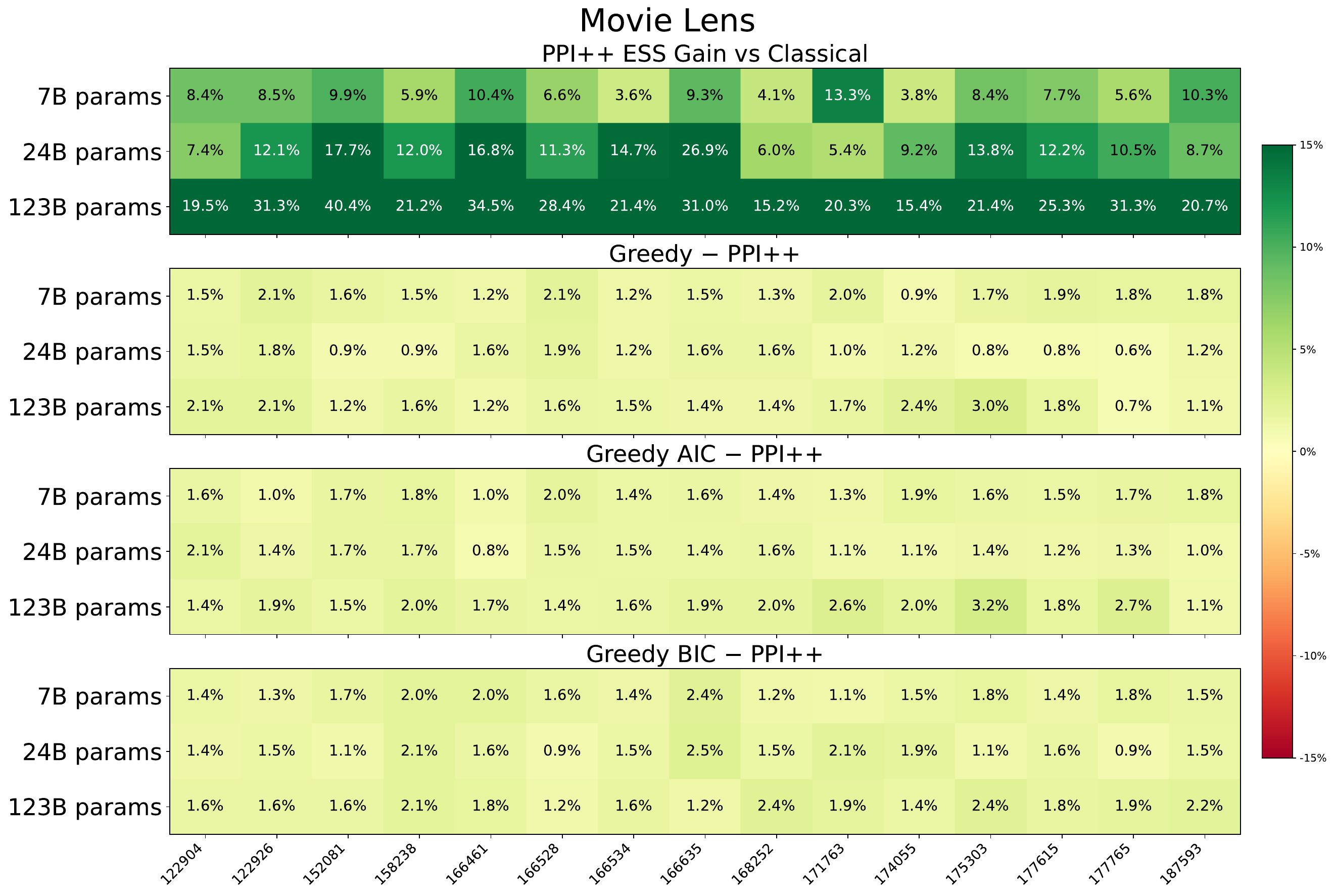}\label{fig:heatmap-movielens}
    \end{minipage}
    \caption{Item-level performance for MovieLens.}
\end{figure}

\begin{figure}[htbp]
    \centering
    \begin{minipage}{0.41\textwidth}
        \centering
        \includegraphics[width=\linewidth]{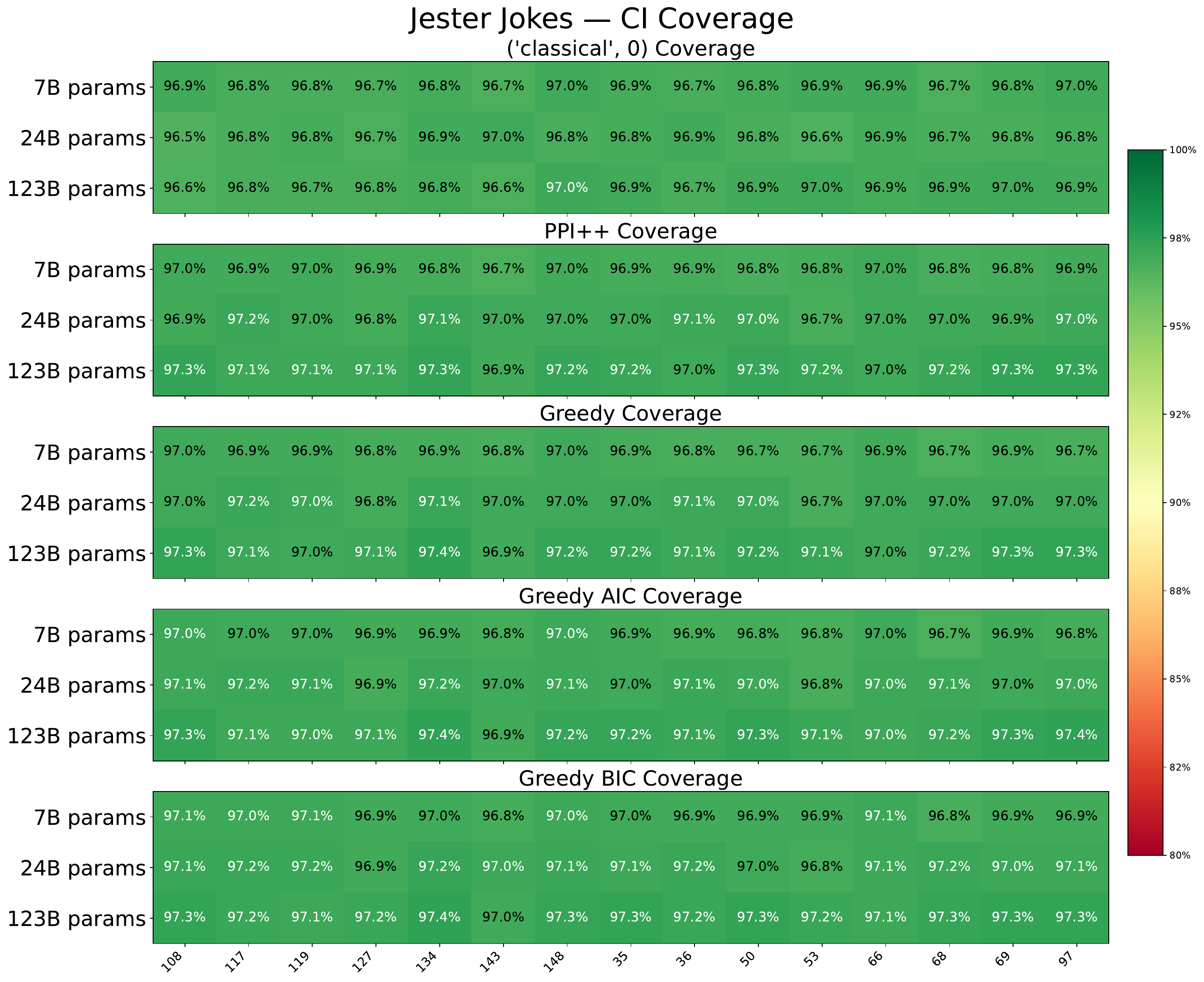}
    \end{minipage}
    \hspace{0.1cm}
    \begin{minipage}{0.55\textwidth}
        \centering
        \includegraphics[width=0.9\linewidth]{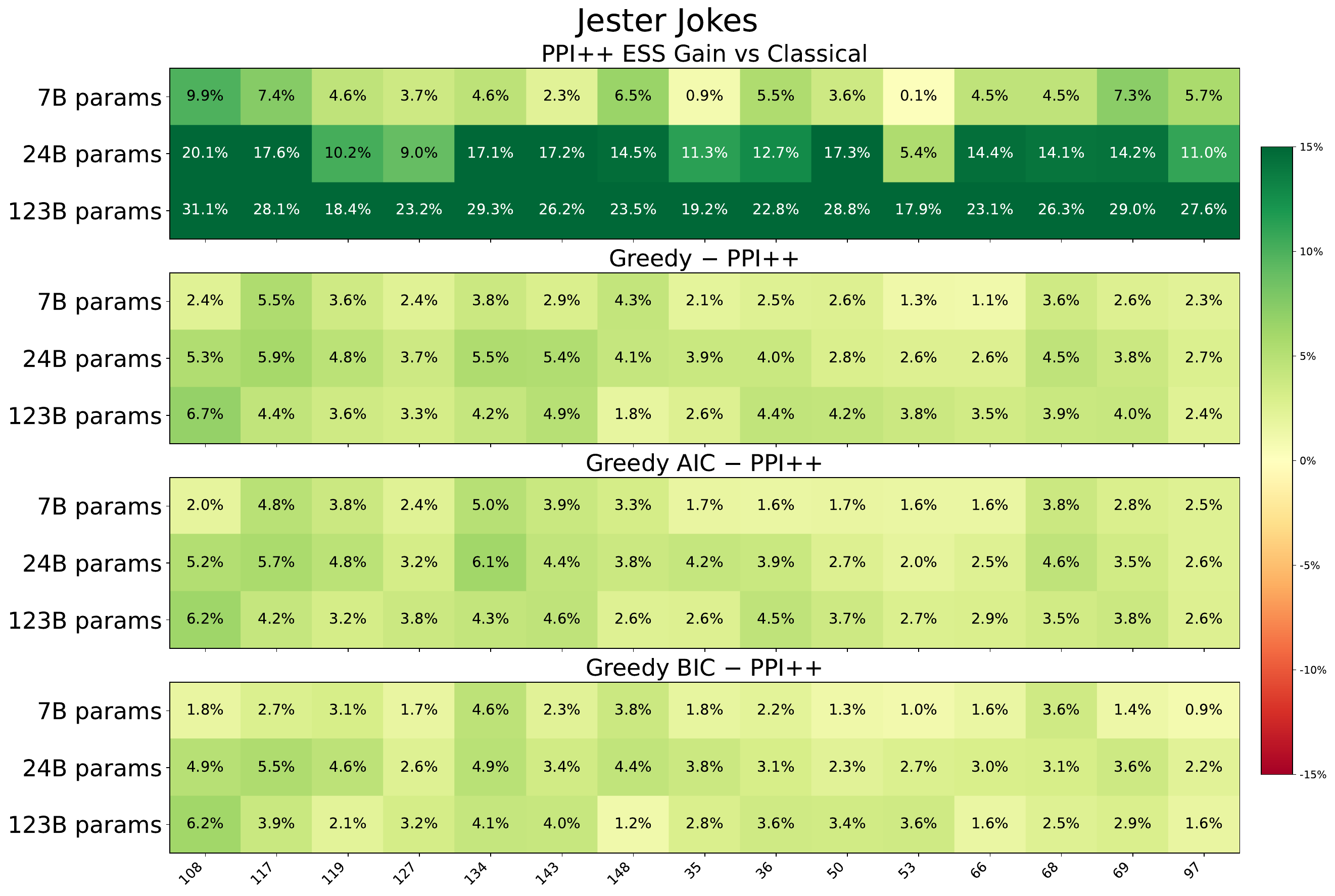}\label{fig:heatmap-jester}
    \end{minipage}
    \caption{Item-level performance for Jester jokes.}
\end{figure}

\begin{figure}[htbp]
    \centering
    \begin{minipage}{0.41\textwidth}
        \centering
        \includegraphics[width=\linewidth]{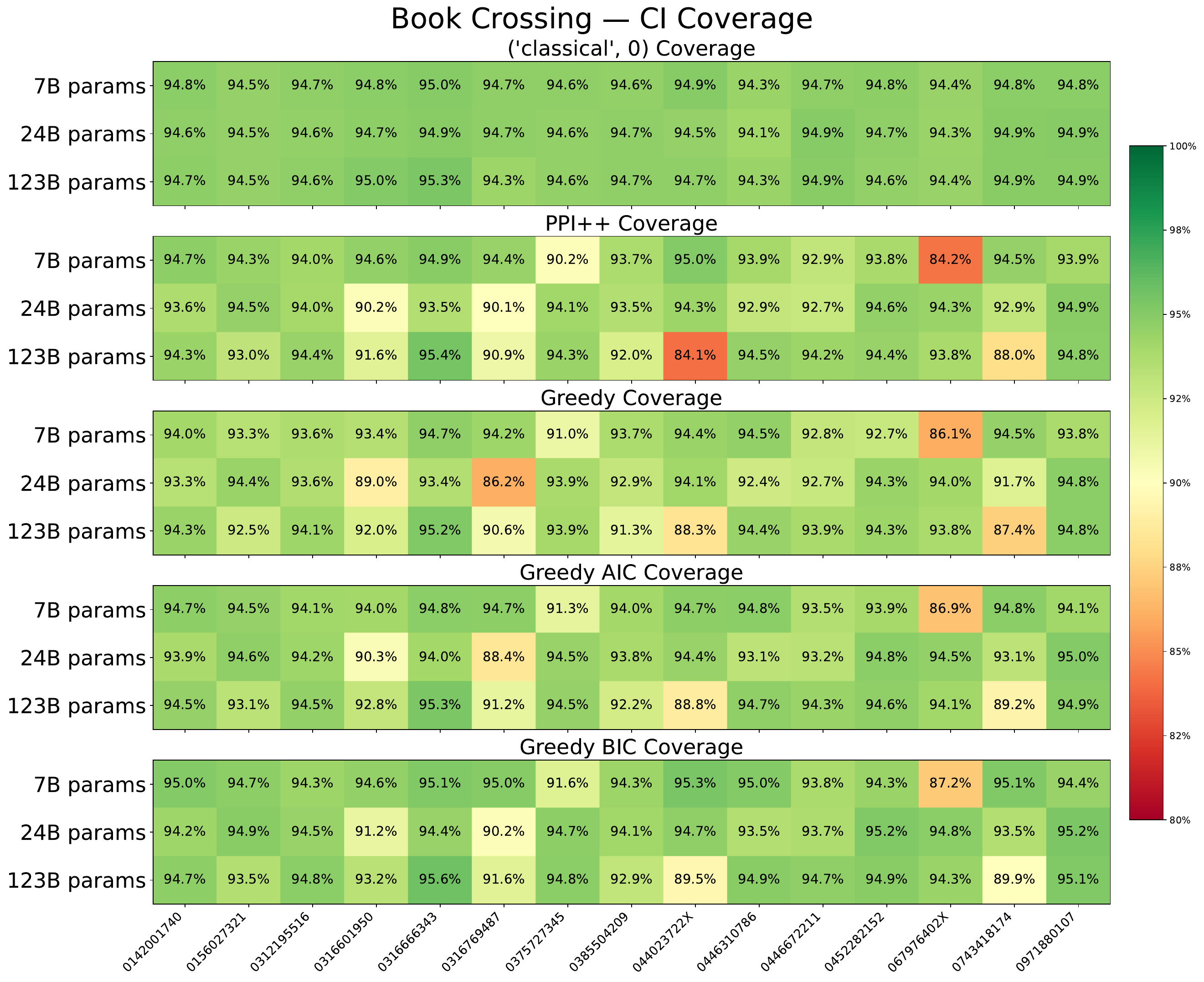}
    \end{minipage}
    \hspace{0.1cm}
    \begin{minipage}{0.55\textwidth}
        \centering
        \includegraphics[width=0.9\linewidth]{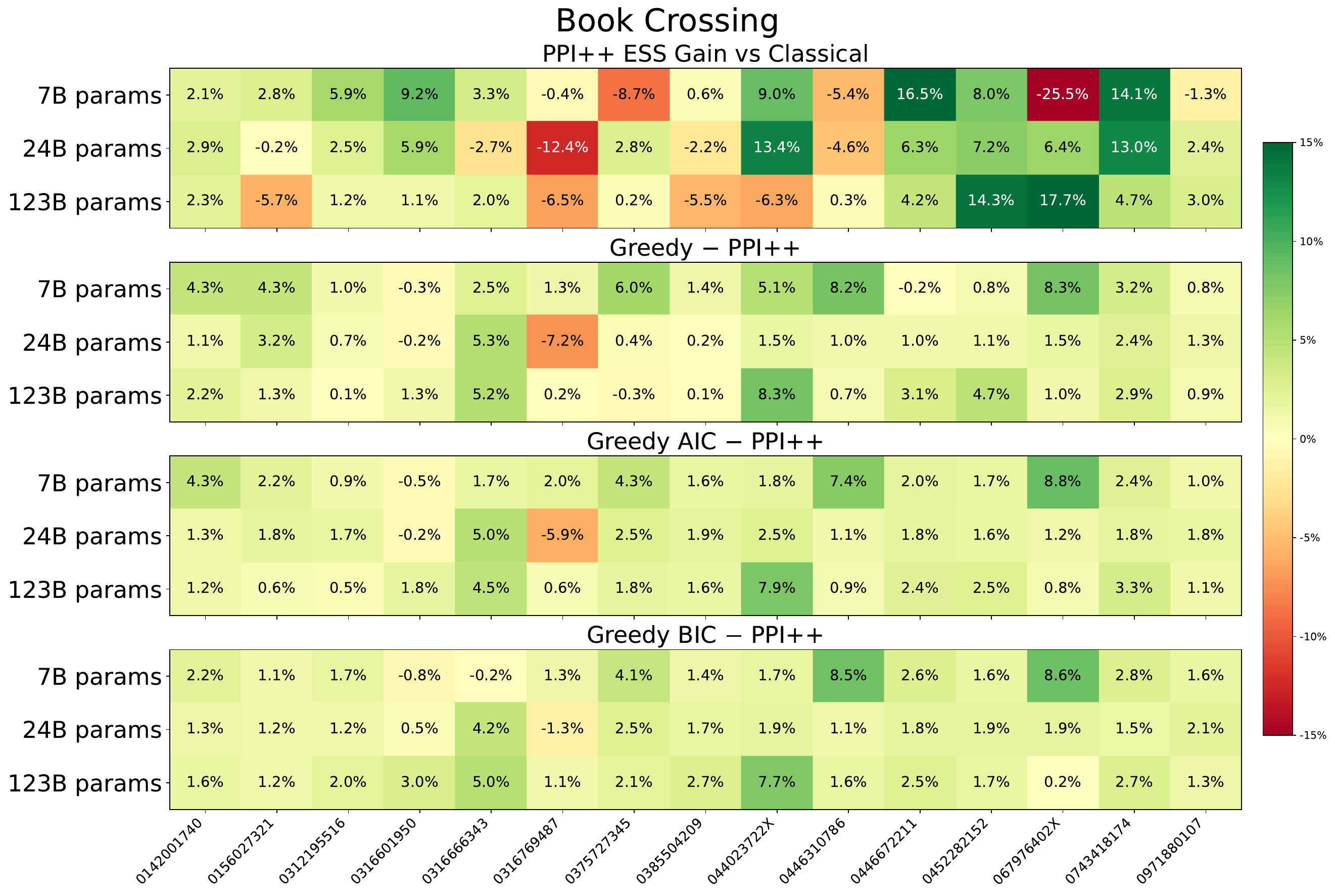}\label{fig:heatmap-book}
    \end{minipage}
    \caption{Item-level performance for Book Crossing.}
\end{figure}

\begin{figure}[htbp]
    \centering
    \begin{minipage}{0.41\textwidth}
        \centering
        \includegraphics[width=\linewidth]{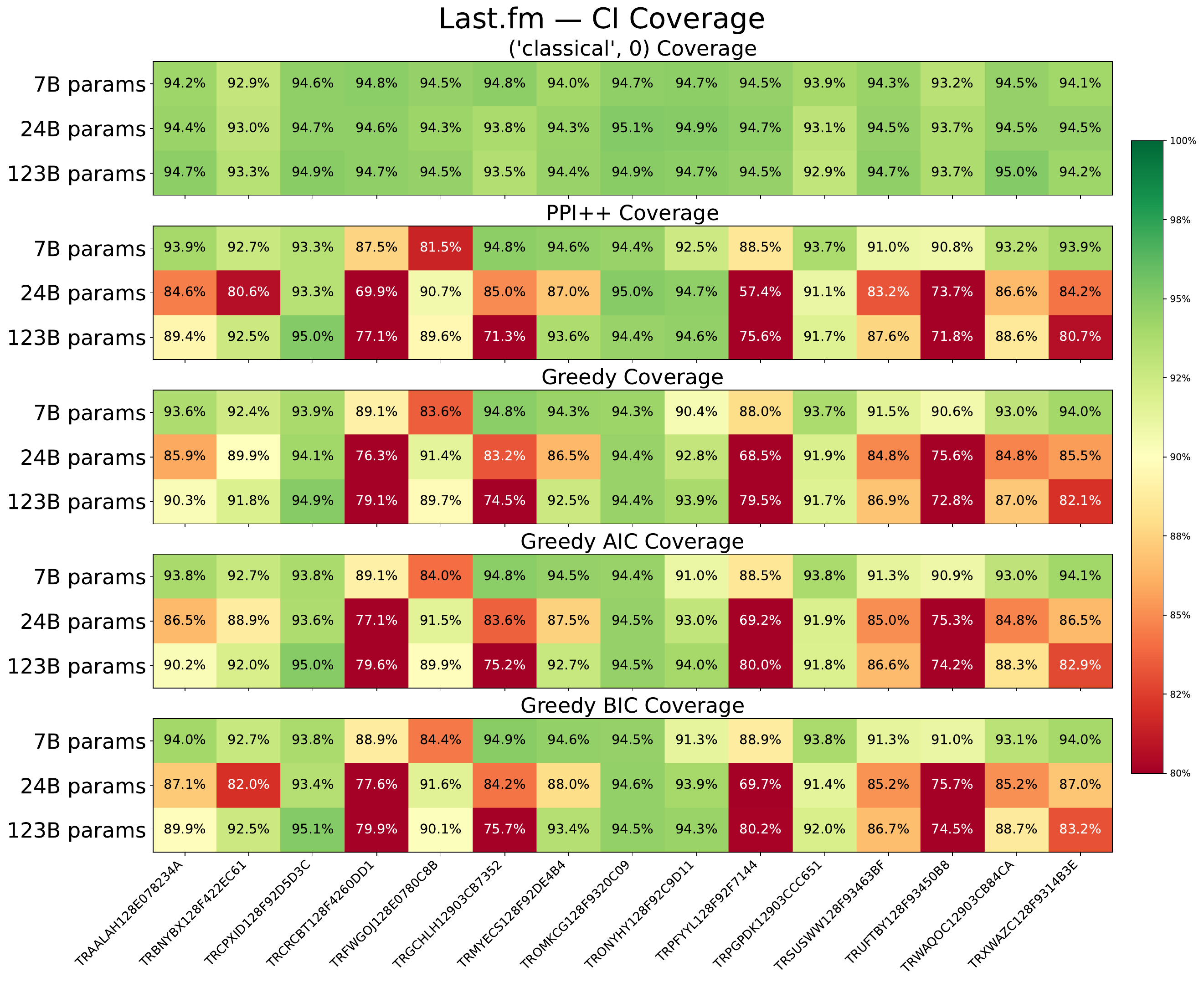}
    \end{minipage}
    \hspace{0.1cm}
    \begin{minipage}{0.55\textwidth}
        \centering
        \includegraphics[width=0.9\linewidth]{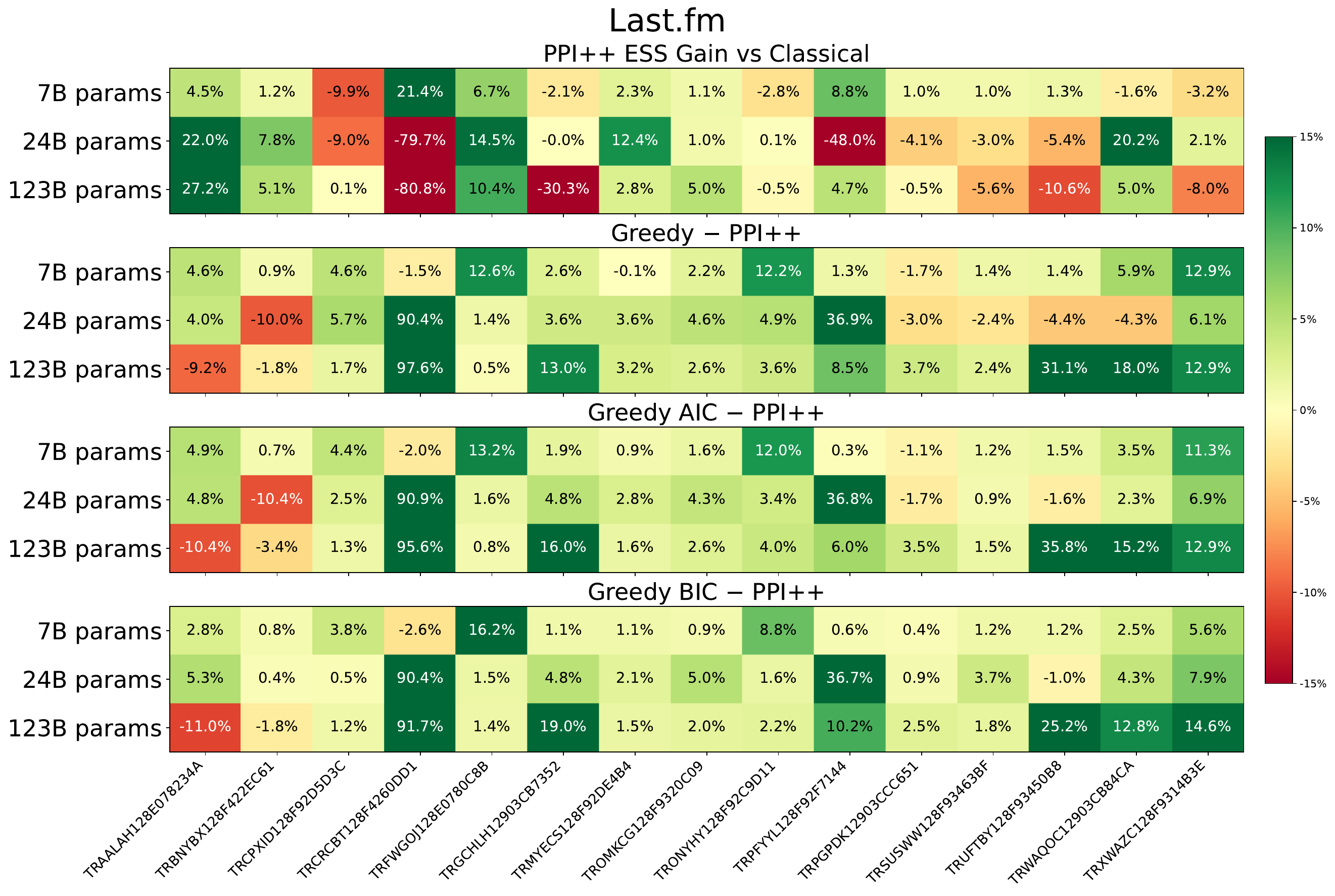}\label{fig:heatmap-music}
    \end{minipage}
    \caption{Item-level performance for Song Rating.}
\end{figure}

Across the four datasets, GPPI with greedy transformation selection outperforms PPI++ in 98\% of the tested items. Improvements are mostly modest on the Movielens dataset, which suggests that linear correction of the predictions is already sifficient. On the Jester and Book crossing datasets the improvements are again moderate for some items but more important for others. 

On the LFM dataset for music engagement, GPPI gives huge improvement over PPI++. The particularity of this dataset is that the  engagement is measured by the number of times a user listened to a song, which is an integer that can be very high. The distribution of play counts for each song is a highly skewed distribution with some rare very high values. PPI++ fails to correctly debias the predictions on such dataset, and its performance, both in terms of coverage and ESS gains, drops significantly. Interestingly, GPPI with greedy selection has globally a better coverage than PPI++, although still does not meet the nominal coverage for difficult items, and it has a significantly better ESS gain. 

While we would expect that on difficult instances PPI++ would be more stable than GPPI with greedy selection, this example shows that, on the opposite, GPPI our approach can be more stable and robust than PPI++.

\paragraph{ESS gain vs Sample size}
Below, we give visualizations of how the ESS gains varies with the sample size for Book-Crossing and Last.fm datasets. The plots for Movielens and Jester datasets are in the main body of the paper.

\begin{figure}[h]
    \centering
    \includegraphics[width=\linewidth]{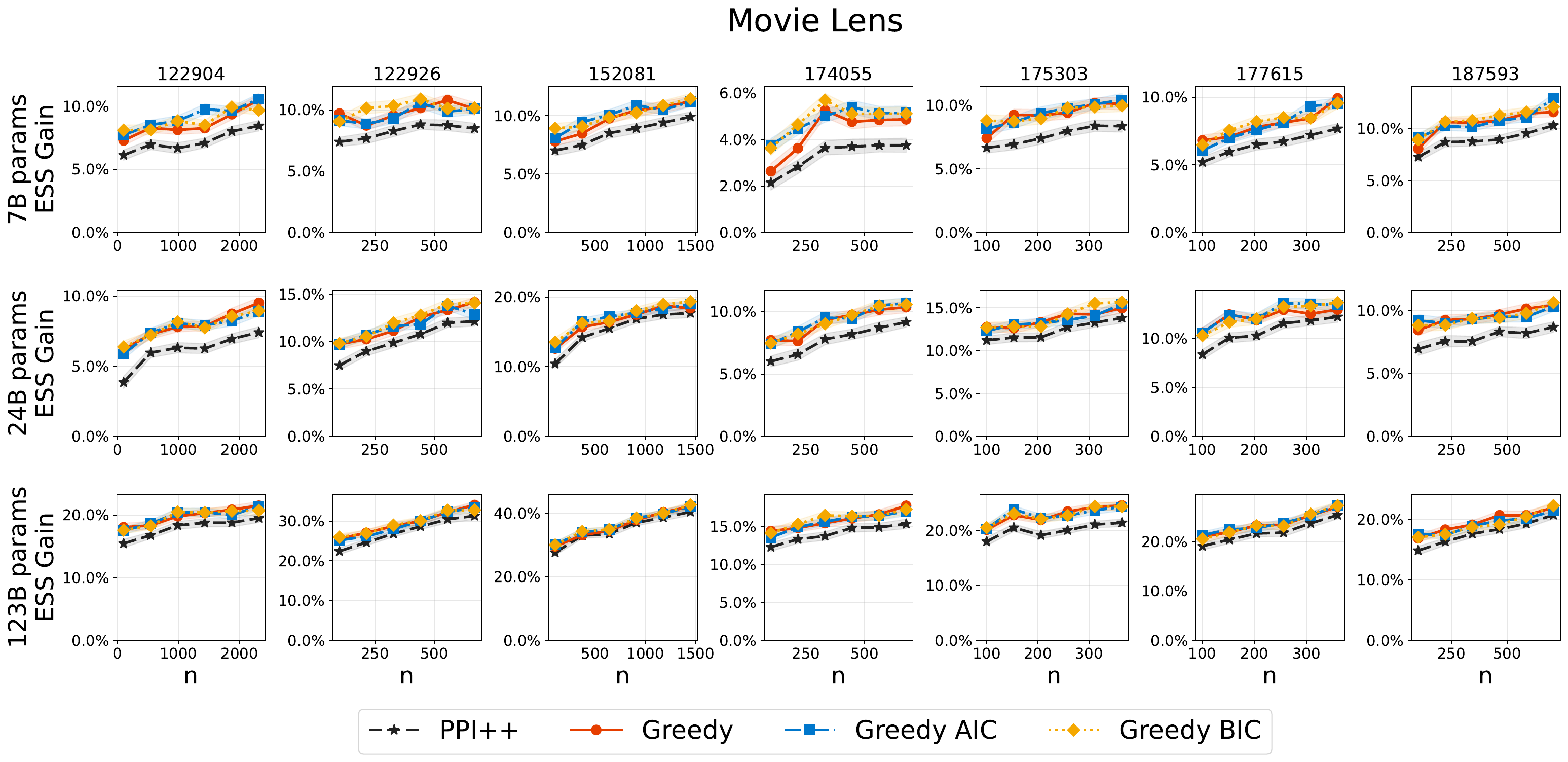}
    \caption{ESS gain comparison on MovieLens dataset across different movies and Mistral models}
    \label{fig:movie-gppi}
\end{figure}

\begin{figure}[h]
    \centering
    \includegraphics[width=\linewidth]{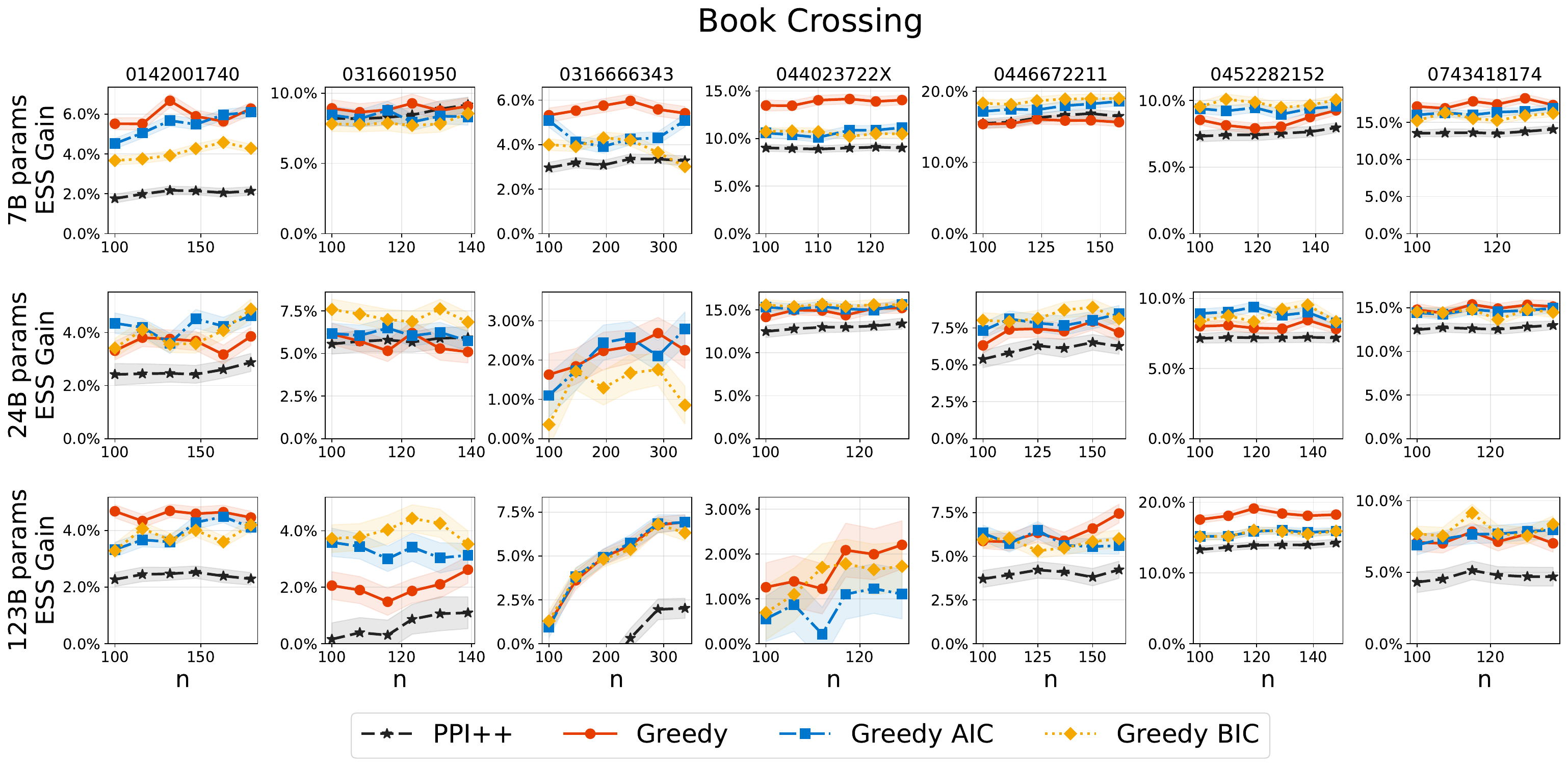}
    \caption{ESS gain comparison on Book-crossing dataset across different jokes and Mistral models}
    
\end{figure}

\begin{figure}[h]
    \centering
    \includegraphics[width=\linewidth]{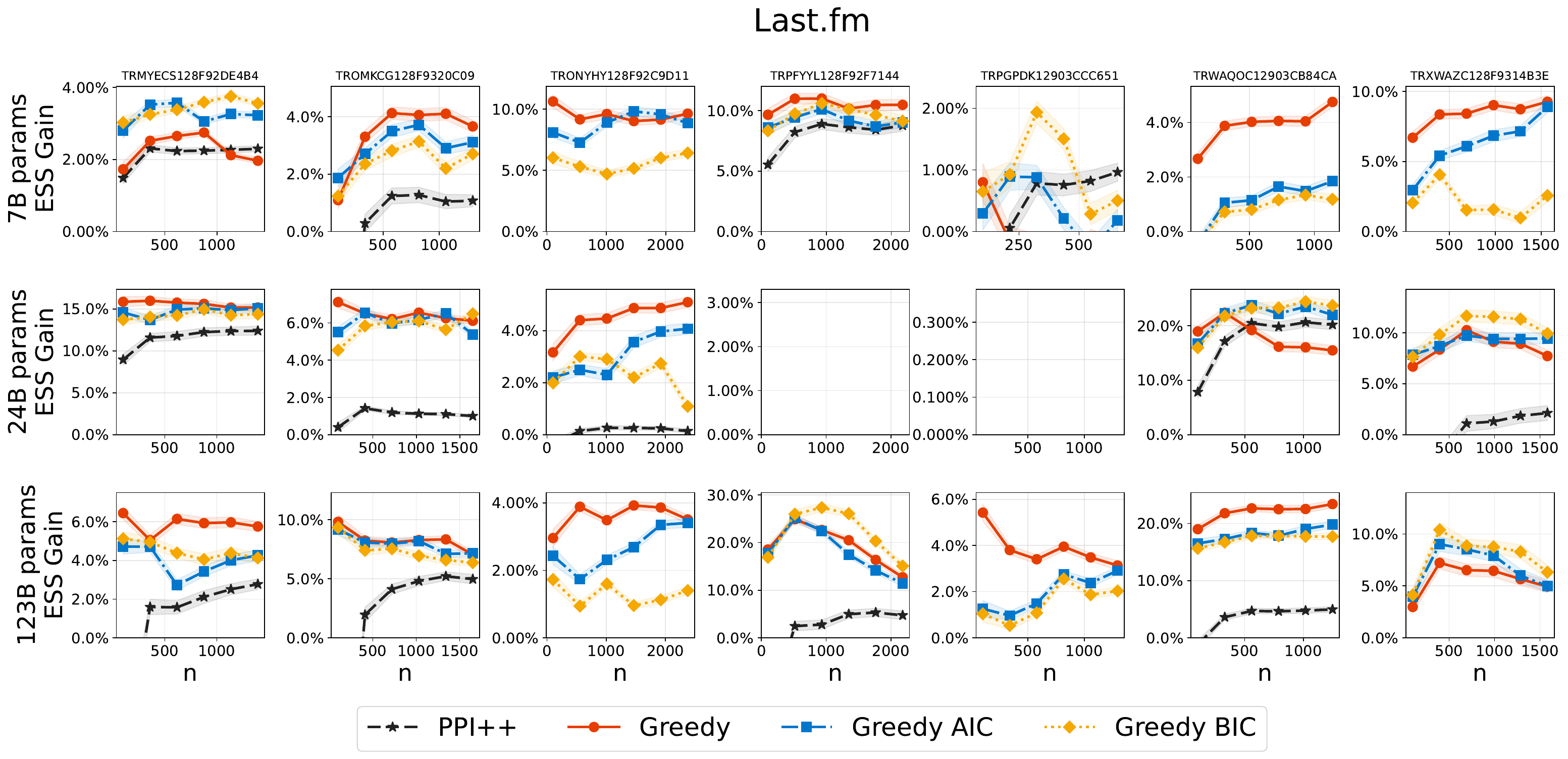}
    \caption{ESS gain comparison on Last.fm across different jokes and Mistral models}
    
\end{figure}

\subsection{A/B testing using GPPI with greedy selection}
We first present the plots of the experiment \ref{sec:exp-population} with Last.fm and Book-Crossing datasets:

\begin{figure}
    \centering
    \includegraphics[width=0.8\linewidth]{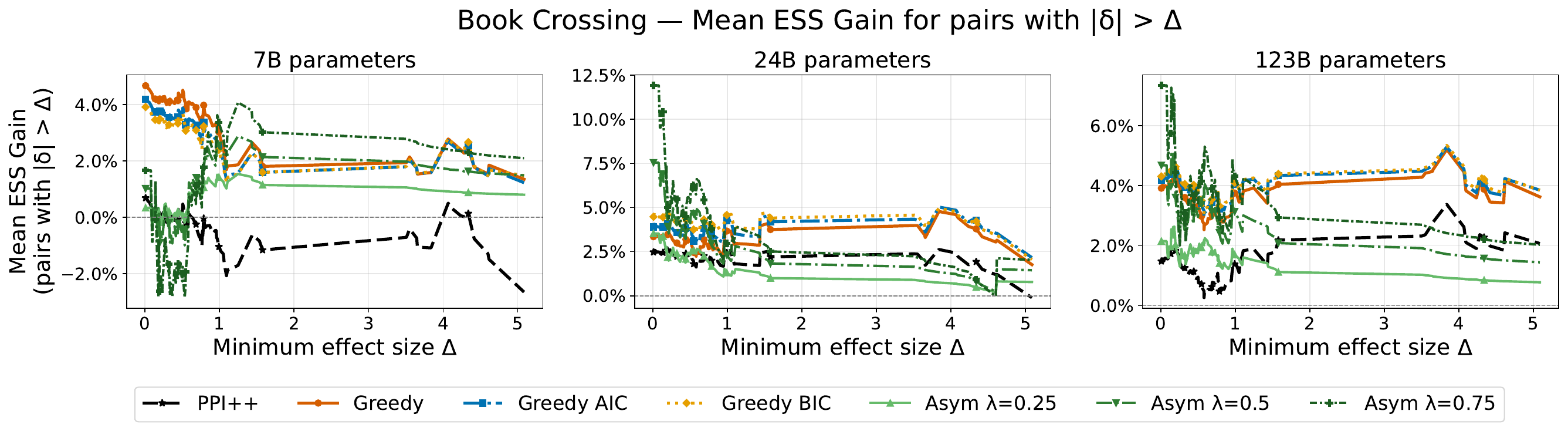}
    \caption{Comparison of the learning-augmented A/B testing approaches for Book-Crossing}
    
\end{figure}

\begin{figure}
    \centering
    \includegraphics[width=0.8\linewidth]{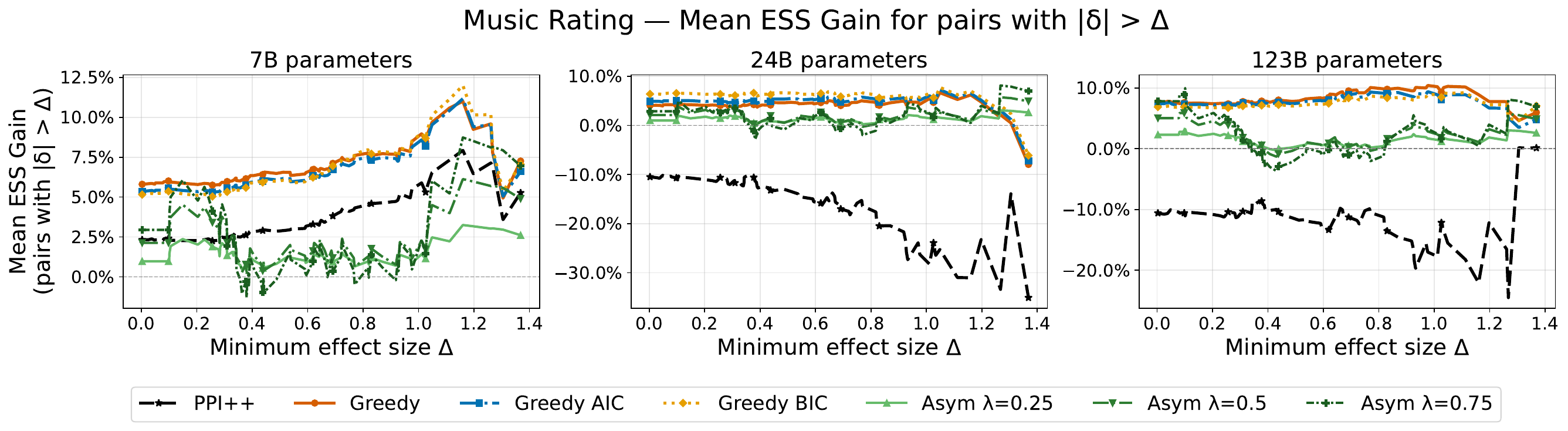}
    \caption{Comparison of the learning-augmented A/B testing approaches for Last.fm}
    
\end{figure}

Consistently with the results for Movielens and Jester datasets, we see that GPPI with greedy transformation selection has better performance than the other approaches. Observe that in some plots (i.e. specific dataset and model choice), PPI++ has negative ESS-gain, which means that the estimation of $\lambda$ is noisy on the labeled data which leads to degrading the quality of the predictions. On the other hand, GPPI with greedy transform selection always give positve ESS gain, and always better than the vanilla PPI++.

To better illustrate the improvement in ESS gains achieved by GPPI with greedy transformation selection compared to PPI++, we conduct A/B tests on all the pairs of items for every dataset, in the same setting described in the experiments section, then we plot the ESS gain of GPPI with greedy transform versus that of PPI++.  We also plot the identity line as a reference. Points above this line are those where GPPI with greedy selection is better than PPI++, and below below it are those where PPI++ has a better performance.

These figures below all illustrate again the power of our method, showing that it consistently and sometimes significantly outperforms PPI++.

\begin{figure}[h]
    \centering
    \includegraphics[width=0.9\linewidth]{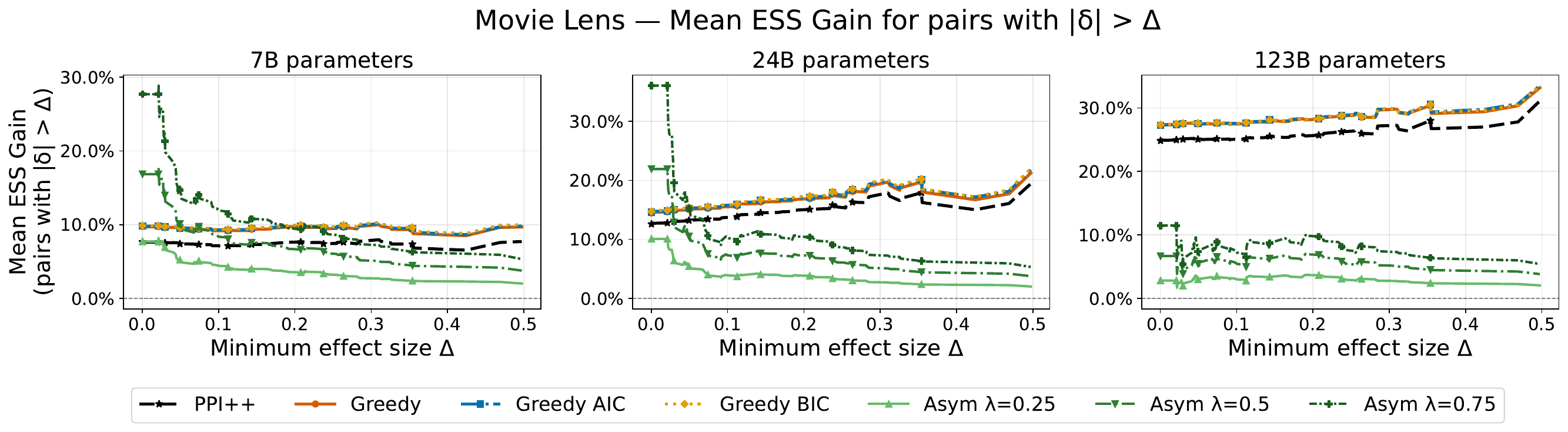}
    \caption{Comparison of the learning-augmented A/B testing approaches for MovieLens}
    \label{fig:ab_movielens}
\end{figure}

\begin{figure}[h]
    \centering
    \includegraphics[width=0.9\linewidth]{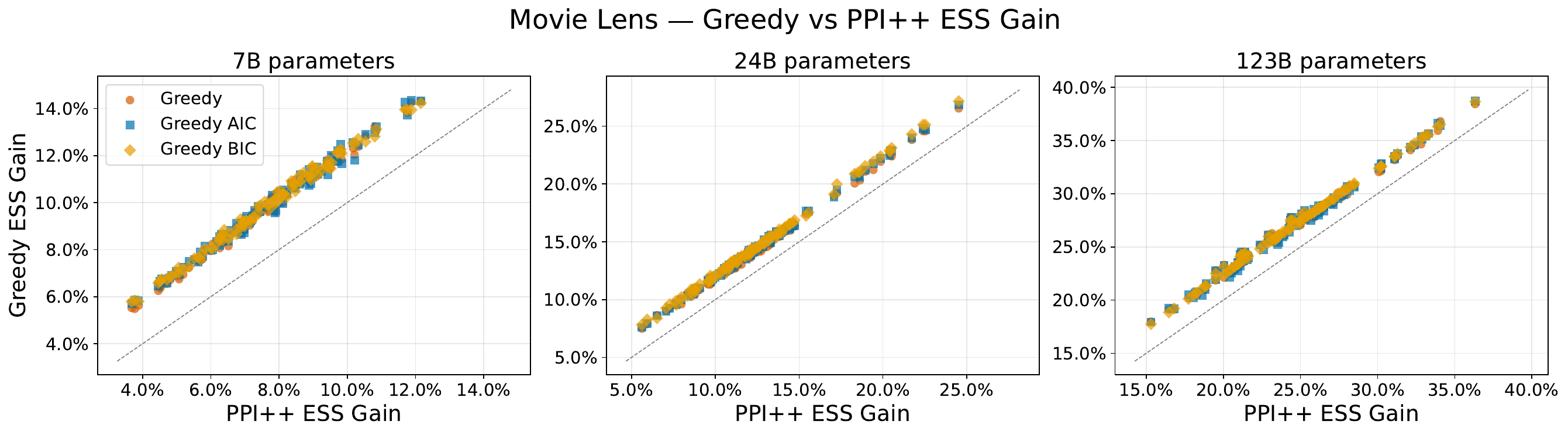}
    \caption{Movielens: comparison of GPPI with greedy selection versus PPI++}
    
\end{figure}

\begin{figure}[h]
    \centering
    \includegraphics[width=0.9\linewidth]{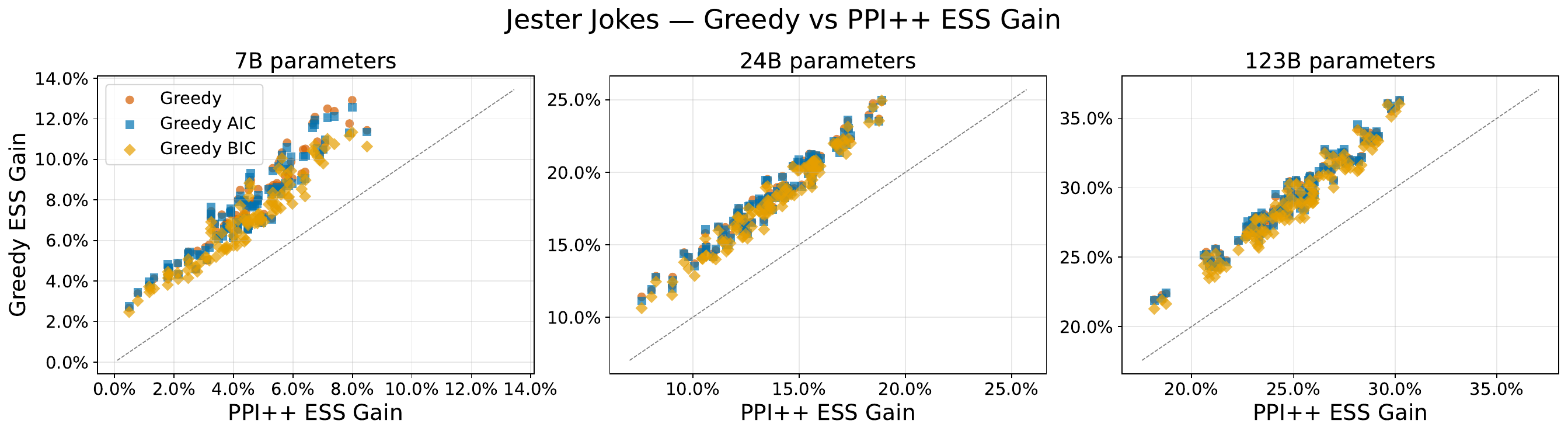}
    \caption{Jester: comparison of GPPI with greedy selection versus PPI++}
    
\end{figure}

\begin{figure}[h]
    \centering
    \includegraphics[width=0.9\linewidth]{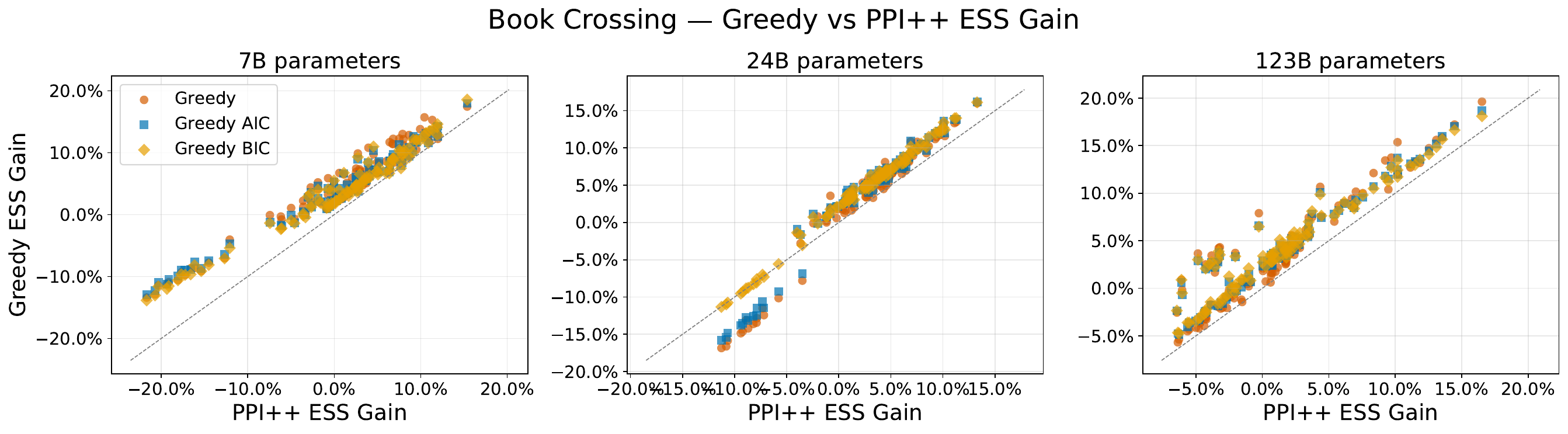}
    \caption{Book crossing: comparison of GPPI with greedy selection versus PPI++}
    
\end{figure}

\begin{figure}[h]
    \centering
    \includegraphics[width=0.9\linewidth]{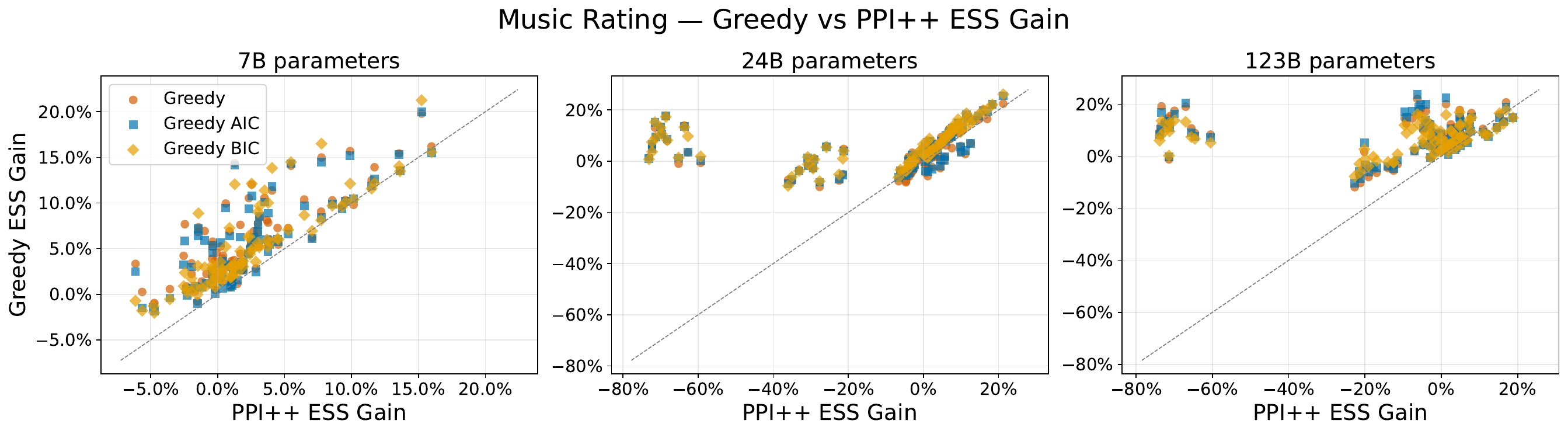}
    \caption{LFM music: comparison of GPPI with greedy selection versus PPI++}
    
\end{figure}

\subsection{Replication of \citep{angelopoulos2023ppi} experiments}

For completeness, we also replicate experiments from~\citep{angelopoulos2023ppi} on deforestation and galaxy classification datasets, where they illustrate how PPI++ reduces the confidence interval width compared to the classical estimator. Since this is a classification task, we use transformations different than before, of the form $\phi(p) = [\mathbf{1}(p > t_1), \ldots, \mathbf{1}(p > t_{d})]$, where $t_1, \ldots, t_d$ are evenly spaced in $(0,1)$. Greedy selection operates on a family of 8 transformations: identity and threshold transformations with $d \in \{5,10,15,20,25,30,35\}$. 

Figure~\ref{fig:ppi++-exp} demonstrates that our method achieves narrower confidence intervals. The coverage is slightly weaker than that of PP++, but still meets the nominal value of 90\% for this experiment for most values of $n$.

\begin{figure}[htbp]
    \centering
    \begin{minipage}{0.48\textwidth}
        \centering
        \includegraphics[width=\textwidth]{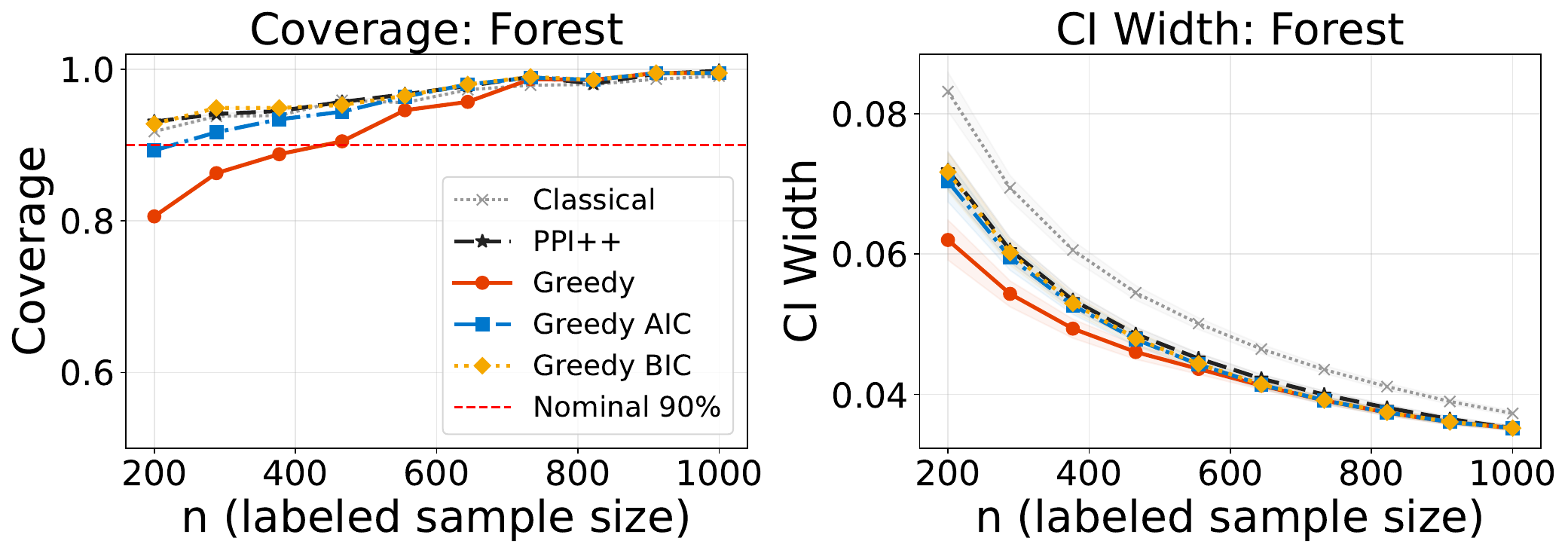}
    \end{minipage}
    \hfill
    \begin{minipage}{0.48\textwidth}
        \centering
        \includegraphics[width=\textwidth]{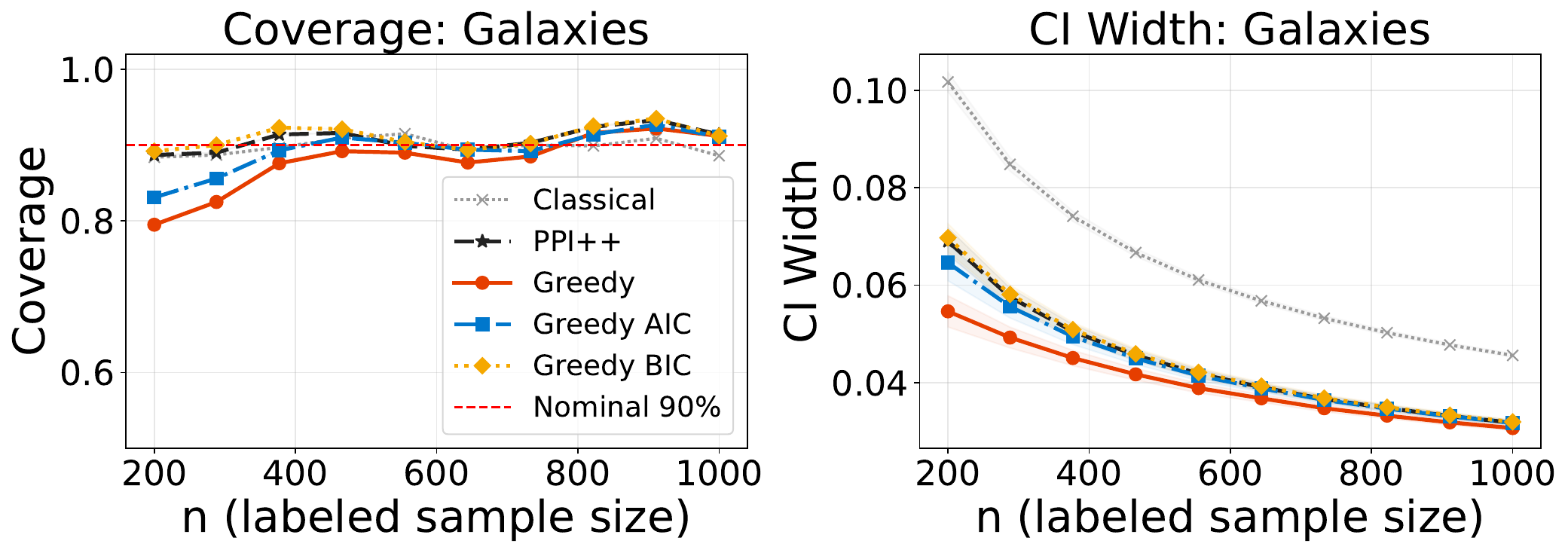}
    \end{minipage}
    \caption{Mean estimation results on deforestation (left) and galaxies (right) datasets showing coverage and confidence interval width.}\label{fig:ppi++-exp}
\end{figure}

\end{document}